%% file: main.tex
\documentclass{article}

\PassOptionsToPackage{numbers, compress}{natbib}

\usepackage[main, final]{neurips_2026}
\usepackage{subcaption} 

\usepackage[utf8]{inputenc} 
\usepackage[T1]{fontenc}    
\usepackage{hyperref}       
\usepackage{url}            
\usepackage{booktabs}       
\usepackage{amsfonts}       
\usepackage{nicefrac}       
\usepackage{microtype}      
\usepackage{xcolor}         

\usepackage{multirow}
\usepackage{graphicx}
\usepackage{wrapfig}
\usepackage{amssymb}
\usepackage{amsmath}
\usepackage{colortbl}
\usepackage{comment}

\usepackage{mathtools}
\usepackage{amsthm}
\usepackage{algorithm}
\usepackage{algpseudocode}

\input{math_commands.tex}

\RequirePackage{hyperref}

\usepackage[capitalize,noabbrev]{cleveref}
\crefname{section}{Sec.}{Secs.}
\crefname{subsection}{Sec.}{Secs.}
\crefname{subsubsection}{Sec.}{Secs.}
\theoremstyle{plain}
\newtheorem{theorem}{Theorem}[section]
\newtheorem{proposition}[theorem]{Proposition}

\theoremstyle{definition}
\newtheorem{definition}[theorem]{Definition}

\theoremstyle{remark}
\newtheorem{remark}[theorem]{Remark}

\usepackage[export]{adjustbox} 

\usepackage[textsize=tiny]{todonotes}

\definecolor{darkblue}{rgb}{0, 0, 0.5}
\hypersetup{colorlinks=true, citecolor=darkblue, linkcolor=darkblue, urlcolor=darkblue}
\definecolor{darkred}{rgb}{0.68,0.05,0.0}
\definecolor{lightblue}{rgb}{0, 0, 1.0}

\newcommand{\firstauthor}[1]{\textsuperscript{\textbf{\textcolor{darkred}{#1}}}}

\newcommand\blfootnote[1]{%
  \begingroup
  \renewcommand\thefootnote{}\footnote{\textcolor{darkred}{*}#1}%
  \addtocounter{footnote}{-1}%
  \endgroup
}

\usepackage{arydshln}

\title{Retrieval-Centric Deep Learning in\\ Growing Nonparametric Neural Networks}

\author{Maximilian Schlegel\firstauthor{*}$^{1}$ ~ Rajai Nasser\firstauthor{*}$^{1}$ ~ Seijin Kobayashi\firstauthor{*}$^{1}$ ~ \textbf{Yanick Schimpf}\firstauthor{*}$^{1}$ \\
\textbf{Oliver Sieberling}\firstauthor{*}$^{2}$ ~ \textbf{Robert Obryk}\firstauthor{*}$^{1}$ \\ \textbf{Kazuki Irie}\firstauthor{*}$^{3}$ ~ \textbf{João Sacramento}\firstauthor{*}$^{1}$ ~ \textbf{Johannes von Oswald}\firstauthor{*}$^{1}$ \\
$^{1}$Google, Paradigms of Intelligence Team, Zurich, Switzerland\\
$^{2}$MIT CSAIL, Cambridge, MA, USA \\
$^{3}$Department of Computer Science and Wu Tsai Institute\\Yale University, New Haven, CT, USA
}

\begin{document}
\blfootnote{All authors contributed equally. Correspondence: \texttt{\{schlegelm,rajainasser,jvoswald\}@google.com, osieberl@mit.edu,  kazuki.irie@yale.edu}. Public code: \href{https://github.com/paradigms-of-intelligence/rcdl}{https://github.com/paradigms-of-intelligence/rcdl}}

\maketitle

\begin{abstract}
We investigate a general-purpose layer for deep learning that, instead of compressing arbitrary-size training data into fixed-size weight matrices, stores a new pair of key-value representations for every data point during training, and retrieves and recombines these representations through an attention mechanism at inference time---resulting in a growing neural net (NN).
While \citet{irie2022dual} have put forward this perspective from the classic duality expressing any linear layer in a deep NN trained by gradient descent as linear attention (LA) over the training data points, replacing LA by more powerful attention functions, as they suggest, turns out to be non-trivial: we show that naively applying learning rules from the LA case to advanced kernels does not lead to principled optimization.
Here we fill this gap and develop functional gradient-based learning rules for kernelized attention layers, based on radial basis function (RBF) and softmax-like kernels---establishing the principled ``retrieval-centric deep learning'' (RCDL) paradigm.
Empirically, we demonstrate the promising performance and learning-efficiency of RCDL on image classification and synthetic teacher-student learning tasks.
Moreover, we show that replacing LA in the dual form of NNs by advanced LA variants, namely MesaNet/DeltaNet, yields a formal connection to recently proposed optimizers for conventional fixed-size NNs, offering a novel perspective on deep learning optimization.\looseness=-1


\end{abstract}

\begin{table}[h]
\vspace{-0.25cm}
\setlength{\tabcolsep}{4pt}
\centering
\caption{Comparison of parametric vs.~retrieval-centric deep learning paradigms.}
\label{tab:concept_table}
\vspace{0.5em} 

\footnotesize 

\begin{tabular}{l m{5.5cm} m{4.8cm}}

\toprule
Model & \textbf{\textit{Parametric} neural networks} & \textbf{Growing \textit{nonparametric} networks} \\ 
\midrule

Building block & Linear layer & Attention layer \\

\midrule

Learnable state & Fixed-size weight matrix & Growing database of keys \& values \\
(for $t$ train examples) & $\mW \in \mathbb{R}^{d_{\text{out}} \times d_{\text{in}}}$ &  $\mathcal{D} = \{\mK \in \mathbb{R}^{d_{\text{in}} \times t}, \mV \in \mathbb{R}^{d_{\text{out}} \times t}\}$ \\

\midrule

Inference & \multirow{2}{*}{$\vy = \mW \vx$}  & \multirow{2}{*}{$\vy = s_{\mathcal{D}}(\vx)= \text{Attention}(\mK, \mV, \vx)$} \\
 ($\vx \in \mathbb{R}^{d_{\text{in}}}$, $\vy \in \mathbb{R}^{d_{\text{out}}}$) & & \\

\midrule

Illustration  & 
\includegraphics[width=3.5cm, valign=c]{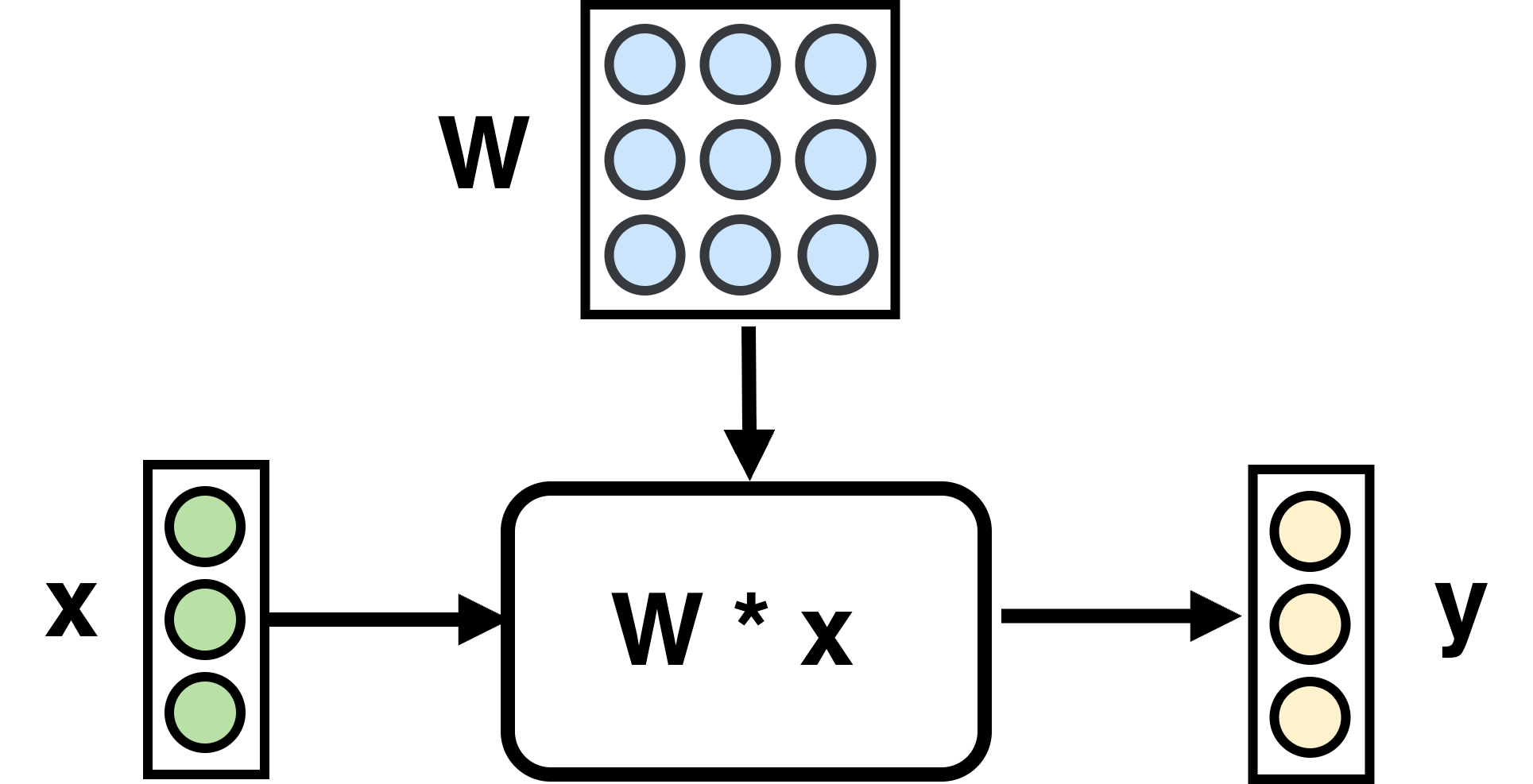} & \includegraphics[width=4.8cm, valign=c]{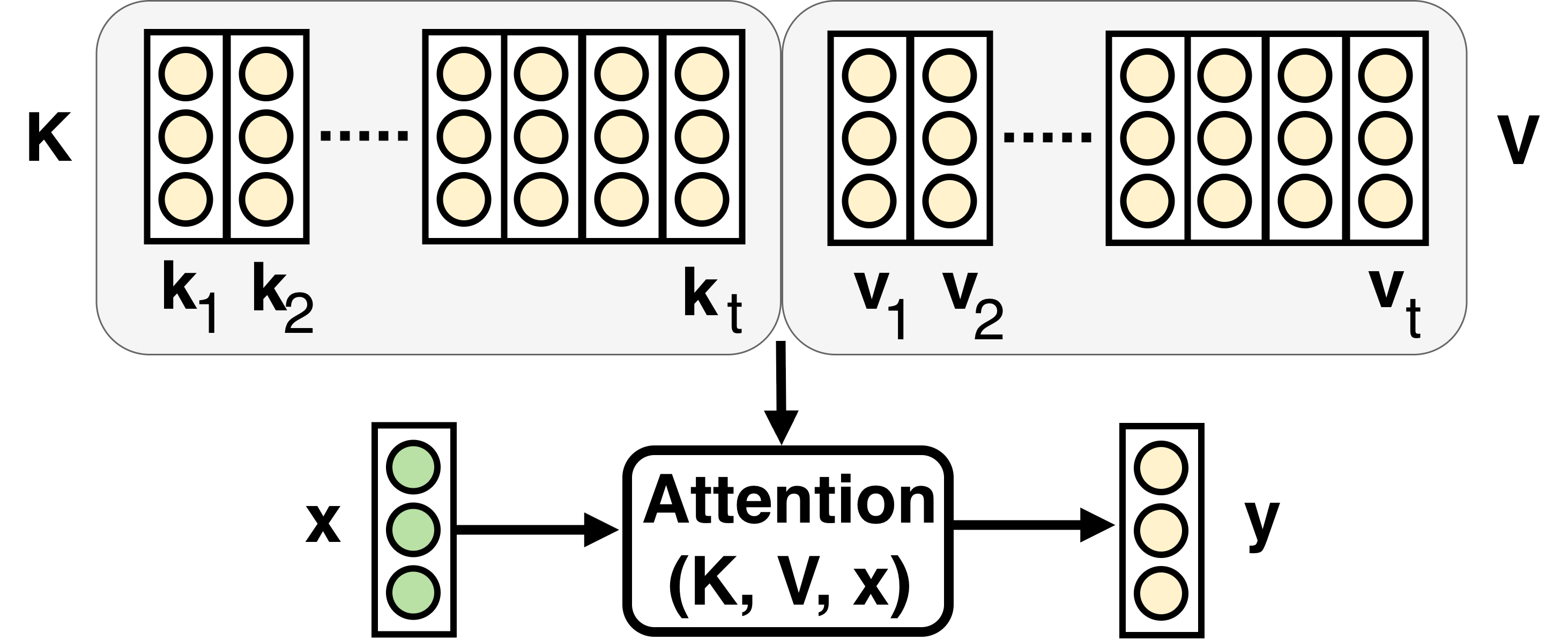} \\ 

\midrule

Forming deep models & Compose linear layers and \newline nonlinearities, e.g., $\sigma(\mW_2 \sigma(\mW_1 \vx))$ & Compose attention layers and nonlinearities, e.g., $\sigma(s_{\mathcal{D}_2}(\sigma(s_{\mathcal{D}_1}(\vx))))$ \\

\midrule

Learning objective & \multicolumn{2}{c}{Minimize a 
global loss $\mathcal{L}$ (e.g., via gradient descent)} \\

\midrule


Learning update & Update the fixed weight matrix & Append a new key-value pair \\
(e.g., gradient-based) & (e.g., $\Delta \mW \propto -\nabla_{\mW} \mathcal{L} = -\nabla_{\vy_t}\mathcal{L} \, \vx_t^\top$) & (e.g., $\vk_t = \vx_t$, $\vv_t \propto -\nabla_{\vy_t}\mathcal{L}$) \\

\midrule

Key question & How to update the weight matrix? & What key-value vectors to append? \\

\bottomrule
\end{tabular}
\vspace{-0.4cm}
\end{table}

\section{Introduction}
\label{sec:intro}

Traditional deep learning uses gradient descent to compress arbitrary-size training datasets and/or experiences into fixed-size weight matrices of a deep neural network (NN).
A classic but lesser-known mathematical property thereof is that any linear layer in such a gradient descent-trained deep NN admits a nonparametric dual form: Its exact input/output relation can also be achieved by an equivalent ``key-value memory system'' that stores all the training datapoints presented to the layer during training and the initial weights, and produces outputs through \textit{retrieval} using ``unnormalized dot-product attention'' over the entire training experience stored in the form of ``key/value pairs'' \citep{aizerman1964theoretical, irie2022dual} (reviewed in Sec.~\ref{sec:background}).\looseness=-1

From a modern sequence modeling research perspective, however, such an unnormalized dot-product attention represents a form of ``linear attention'' \citep{katharopoulos2020transformers, shen2018efficient}, which is well-known to be sub-optimal to process sequences, e.g., compared with the softmax attention \citep{bahdanau2014neural} commonly used in the transformer \citep{trafo}.\looseness=-1

This observation has raised a natural question: can we achieve more powerful deep learning systems by replacing the linear attention (LA) underlying deep NNs by a more powerful attention function \citep{irie2022dual,irie2026dual}?
While this question has been left unanswered, the time is ripe to 
revisit it, especially given how recent research in sequence modeling has advanced LA and provided insights to interpret and improve LA through various angles, including its connection to fast weight programming and learning rules \citep{schmidhuber1992learning, schlag2021linear, irie2021going, irie2022neural, IrieSCS22,yang2024parallelizing,yang2024gated,siems2025deltaproduct}, the underlying local optimization perspectives \citep{von2022transformers,von2023uncovering, von2025mesanet, sun2024learning,behrouz2025atlas,behrouz2024titans,behrouz2025nested,wang2025test}, as well as traditional sequence processing tools such as gating \citep{peng2021random,YangWSPK24}.

Here we pursue two approaches that aim to go beyond current deep learning by extending the LA mechanism underlying any modern deep NNs' dual form.
First, we advance the retrieval-centric view of the dual form by developing a theory for the ``growing neural network'' obtained by introducing a (non-linearizable) kernel to the dual form's LA, and devise its learning algorithms (which prescribe the key and value representations to be stored for each data point) that optimize a given global objective function. We study the cases of positive definite kernels (e.g., the classic radial basis function (RBF) \citep{scholkopf2002learning}) and the ``weighted softmax'' kernel we propose. In fact, while \citet{irie2022dual} have suggested to directly ``insert'' the softmax function in the corresponding linear attention to enhance retrieval from NNs' training experience, devising a principled learning algorithm for such a system turns out to be highly non-trivial, as such an insertion invalidates the optimality of the key and value representations computed in the linear case.
This is especially true for the softmax, for which  we introduce a closely related ``weighted softmax kernel'' to formally derive a learning rule for the resulting retrieval-centric growing NN. We conduct experiments using classic image classification datasets and teacher-student learning tasks, and demonstrate promising results in performance and sample efficiency compared to the standard deep NNs with fixed-size weight matrices.\looseness=-1

The second extension we pursue is to leverage recent advances in linear attention, namely the weight updates of DeltaNet \citep{schlag2021linear,yang2024gated}, and MesaNet \citep{von2025mesanet}.
In contrast to the first direction above, this approach gives rise to optimizers that improve  conventional weight-based learning.
We derive a formal connection to a recently proposed optimizer and second-order optimization, and empirically demonstrate the efficiency of the resulting optimizers, with competitive performance to Muon \cite{jordan2024muon}.

Finally, as a new framework that represents a departure from standard deep learning (see comparison in Table \ref{tab:concept_table}), retrieval-centric deep learning (RCDL) still faces many challenges. Perhaps the most central one is that model size grows at every training step when a nonlinear kernel is chosen. We finish our paper by presenting first results on a hybrid learning algorithm which manages capacity allocation dynamically, by combining standard (parametric) gradient updates with model growing (non-parametric) steps, triggered only when the inputs to a given layer are sufficiently distinct from past ones. We show empirically that this hybrid algorithm is able to match the performance of the non-parametric learner at a small fraction of its size.

\section{Background: Dual Form of a Linear Layer}
\label{sec:background}

Here we briefly review the dual form of a linear layer trained by gradient descent as a linear attention system \citep{aizerman1964theoretical, irie2022dual}, which is our starting point. Let $t$, $T$, $d_\text{in}$, $d_\text{out}$ denote positive integers.
Consider a linear layer inside an arbitrary complex/deep neural network, that transforms a layer input $\vx \in \mathbb{R}^{d_\text{in}}$ into a layer output $\vy \in \mathbb{R}^{d_\text{out}}$ as
$\vy = \mW \vx$ where $\mW \in \mathbb{R}^{d_\text{out} \times d_\text{in}}$ is a weight matrix. 

Given an arbitrary differentiable loss function $\mathcal{L}$ and a sequence $(\vx_1, \ldots,\vx_T)$ of $T$ training inputs presented to the layer, with  $\vx_t \in \mathbb{R}^{d_\text{in}}$ for $1 \leq t \leq T$, we denote the corresponding (backpropagated) error signals as $(\vv_1, \ldots, \vv_T)$ with $\vv_t \in \mathbb{R}^{d_{\text{out}}}$ defined by gradient descent, that is $\vv_t = -\eta_t \nabla_{\vy_t} \mathcal{L}$ where $\eta_t \in \mathbb{R}_{>0}$ is the learning rate and $\nabla_{\vy_t} \mathcal{L}$ is the gradient of the loss w.r.t.~the layer output $\vy_t$ corresponding to input $\vx_t$.

At training step $t$, gradient descent updates the weight matrix (WM), initialized with $\mW_0$, as
$\mW_t = \mW_{t-1} -\eta_t (\nabla_{\vy_t} \mathcal{L}) \otimes \vx_t =  \mW_{t-1} + \vv_t \otimes \vx_t$,
where $\otimes$ denotes the outer product, that is $\vv_t \otimes \vx_t = \vv_t \vx_t^{\top}$.
We can unfold this equation to express the WM at step $T$ as: $\mW_T = \mW_0 + \sum_{t=1}^{T} \vv_t \vx_t^{\top}$.

Thus, the layer output for an arbitrary input $\vq \in \mathbb{R}^{d_\text{in}}$ can be expressed as:
\begin{equation}
\vy = \mW_T \vq = 
\mW_0 \vq + \sum_{t=1}^{T} \vv_t \vx_t^{\top} \vq = \mW_0 \vq + \sum_{t=1}^{T} \vv_t \phi(\vx_t, \vq)
\label{eq:gd_linear_attention}
\end{equation}
where the last term is a weighted sum of vectors $\vv_t \in \mathbb{R}^{d_{\text{out}}}$ with weights $\vx_t^{\top} \vq \in \mathbb{R}$, which can be interpreted as an (unnormalized) ``attention mechanism'' by identifying $\vv_t$ as values, $\vx_t$ as keys, and $\vq$ as the query; and $\phi(\vx_t, \vq) = \vx_t^{\top} \vq$ as attention weights \citep{luong2015effective,trafo,bahdanau2014neural}.
To be more specific, this is (a simple form of) linear attention \citep{katharopoulos2020transformers, shen2018efficient}.

This implies that the input/output relation of the linear layer can also be realized by a ``growing nonparametric neural network'': Instead of storing the fixed-size weight matrix $\mW$ as the system's training memory, the system maintains a growing database by storing a new pair $(\vk_t$, $\vv_t)$ of key-value vectors for each training data point---the layer input $\vk_t = \vx_t \in \mathbb{R}^{d_\text{in}}$ as the ``key'' and the corresponding error signal $\vv_t = -\eta_t \nabla_{\vy_t} \mathcal{L} \in \mathbb{R}^{d_\text{out}}$ as the ``value'' vector (in addition to the initial weight matrix $\mW_0$). Inference is performed through the linear attention mechanism above (Eq.~\ref{eq:gd_linear_attention}).

Given the ubiquity of linear layers in deep neural networks, this dual view suggests that a fundamental mechanism linking training and inference in virtually any deep learning systems is based on (a cascade of) this simple linear attention mechanism.
However, recent research on linear attention in sequence modeling indicates that such ``simple linear attention'' mechanisms are highly suboptimal. This points to the potential for fundamentally improving deep learning by enhancing the underlying linear attention mechanism.

\section{Growing Nets with Non-Linear Attention}
\label{sec:attention}

With the goal of developing a general-purpose deep learning layer that can improve upon the conventional linear layer,  we replace $\phi$ in Eq.~\ref{eq:gd_linear_attention} by a more powerful kernel that is not reducible to a finite-dimensional linear map.
Consequently, the resulting layer can no longer be converted into a linear layer with a fixed-size weight matrix. Instead, we obtain a ``growing layer'' in which a new pair of key/value representations is added to the system for each training datapoint.
This raises a fundamental question: how should these key and value representations be chosen to implement a \textit{principled} learning algorithm that minimizes a given loss function?
Here we formally derive such algorithms in two classes of powerful attention functions.

\subsection{Positive Definite Kernel Attention}
\label{sec:rbf}
We first investigate the case where $\phi$ in Eq.~\ref{eq:gd_linear_attention} is a positive definite kernel (PDK).
By the Moore-Aronszajn theorem \citep{aronszajn1950theory}, any PDK admits a representation as an inner product of feature embeddings into a Hilbert space $\mathcal{H}_\phi$. That is, there exists a map $\psi_\phi:\mathbb{R}^{d_{\text{in}}}\to\mathcal{H}_{\phi}$ such that $\phi(\vx,\vx')=\langle \psi_{\phi}(\vx),\psi_{\phi}(\vx')\rangle_{\mathcal{H}_{\phi}}$.\looseness=-1

This allows us to view the nonparametric layer as a linear layer operating on the implicit feature space $\mathcal{H}_\phi$ with an infinite-dimensional weight matrix $\mW \in \mathbb{R}^{d_{\text{out}}\times \infty}$.
In \cref{app:rbf}, we rigorously show that performing standard gradient descent on this implicit weight matrix is mathematically equivalent to appending a new key-value pair to the database at every step.
Analogously to the analysis of the linear kernel case in \cref{sec:background}, the gradient update of the implicit weight matrix is $\Delta \mW \propto \left(-\eta_t \nabla_{\vy_t} \mathcal{L}\right) \psi_\phi(\vx_t)^\top$. This translates exactly to adding a new term $\left(-\eta_t \nabla_{\vy_t} \mathcal{L}\right) \phi(\vx_t, \cdot)$ to the layer's output function.

Consequently, to build a growing layer based on a PDK, we retain the exact key-value definitions from the linear case: setting the key $\vk_t = \vx_t$ and the value $\vv_t = -\eta_t \nabla_{\vy_t} \mathcal{L}$ implements a principled loss-minimizing learning rule in the RKHS.
In our experiments, we use the radial basis function (RBF) kernel---a classic PDK \citep{scholkopf2002learning}.

\subsection{Softmax and Weighted Softmax Attention}
\label{sec:softmax}

While PDKs offer a robust theoretical foundation, modern deep learning relies heavily on softmax attention, which empirically outperforms unnormalized kernel attention in many sequence modeling tasks \citep{trafo}. Here we study the case where we replace $\phi$ in Eq.~\ref{eq:gd_linear_attention} by the softmax function (more formally, $\phi(\vx_t, \vq)$ becomes $\phi(\vx_t, \vq; \{\vx_\tau\}_{1\leq\tau\leq T})$ as it involves normalization).
As it turns out, theoretically dealing with the standard softmax attention is highly non-trivial (as we discuss below; and in Appendix \ref{app:remark_standard_softmax}). Instead, we introduce a ``weighted softmax kernel'' and derive a learning algorithm for the corresponding growing NN.\looseness=-1

Let $\mathcal{D} \subset \mathbb{R}^{d_{\text{in}}} \times \mathbb{R}^{d_{\text{out}}}$ be a finite set of key-value pairs or ``database''. This database induces the standard softmax attention function $s_{\mathcal{D}}:
\mathbb{R}^{d_{\text{in}}} \to \mathbb{R}^{d_{\text{out}}}$ defined as:
$
    s_{\mathcal{D}}(\vq) = \sum_{(\vk,\vv)\in\mathcal{D}} e^{\langle \vk, \vq \rangle} \vv \big/ \sum_{(\vk,\vv)\in\mathcal{D}} e^{\langle \vk, \vq \rangle} \,.$

Consider the objective of making a small ``corrective'' adjustment to $\mathcal{D}$ to move $s_{\mathcal{D}}$ in a direction that minimizes a global loss. Ideally, we would achieve this by adding a new key-value pair to the database, analogous to the ``growing layers'' derived in Sec.~\ref{sec:rbf}.

However, this is not possible in an infinitesimal manner for this softmax attention. Introducing any new key \textit{discontinuously} alters the denominator, regardless of how small the associated value vector is.\looseness=-1

To enable the ``smooth addition'' of new key-value pairs, we introduce a scalar weight for each entry. This leads to the ``weighted softmax attention'' formulation, where the database $\mathcal{D}$ consists of a finite set of triples $(w, \vk, \vv) \in \mathbb{R}^+ \times \mathbb{R}^{d_{\text{in}}} \times \mathbb{R}^{d_{\text{out}}}$, and the retrieval function becomes:
\vspace*{-1mm}
\begin{equation}
    \label{eq:weighted_softmax}
    s_{\mathcal{D}}(\vq) = \frac{\sum_{(w,\vk,\vv)\in\mathcal{D}} w e^{\langle \vk, \vq \rangle} \vv}{\sum_{(w,\vk,\vv)\in\mathcal{D}} w e^{\langle \vk, \vq \rangle}} \,.
\end{equation}
In the remainder of this section, we develop a theoretical framework that allows us to derive a ``growing database'' formulation which corresponds to performing a gradient descent in a space containing all weighted softmax functions.

Now consider an infinite database containing a weight for every possible key-value pair. We formalize this by defining a weight function $w: \mathbb{R}^{d_{\text{in}}} \times \mathbb{R}^{d_{\text{out}}} \to \mathbb{R}^+$, which induces the continuous softmax function $s_{w}: \mathbb{R}^{d_{\text{in}}} \to \mathbb{R}^{d_{\text{out}}}$ as:
\vspace*{-1mm}
\begin{equation}
    \label{eq:softmax_with_weight_function}
    s_{w}(\vq) = \frac{\int_{\mathbb{R}^{d_{\text{in}}} \times \mathbb{R}^{d_{\text{out}}}} w(\vk,\vv) e^{\langle \vk, \vq \rangle} \vv \, d\vk \, d\vv}{\int_{\mathbb{R}^{d_{\text{in}}} \times \mathbb{R}^{d_{\text{out}}}} w(\vk,\vv) e^{\langle \vk, \vq \rangle} \, d\vk \, d\vv} \,.
\end{equation}
Appendix \ref{app:remark_functional_space_of_weighted_softmax_functions} provides further remarks on the functional space of weighted softmax functions.

\textbf{The case of Gaussian mixtures.}
We analyze the case where $w:\mathbb{R}^{d_{\text{in}}} \times \mathbb{R}^{d_{\text{out}}}\to\mathbb{R}^+$ represents the density of a Gaussian mixture. Let
$\mathcal{D}=\{(w_i,\vk_i,\vv_i):1\leq i\leq n\}\subset\mathbb{R}^+ \times \mathbb{R}^{d_{\text{in}}} \times \mathbb{R}^{d_{\text{out}}}$
be a finite set of $n$ weighted key-value pairs, and let $\sigma_k,\sigma_v>0$. We define the Gaussian mixture density $w_{\mathcal{D}}$ as follows:\looseness=-1
$$w_{\mathcal{D}}(\vk,\vv)=\sum_{i=1}^n\frac{w_i}{(2\pi)^{\frac{d_{\text{in}}+d_{\text{out}}}{2}}\sigma_k^{d_{\text{in}}}\sigma_v^{d_{\text{out}}}} e^{-\frac{\|\vk-\vk_i\|^2}{2\sigma_k^2}-\frac{\|\vv-\vv_i\|^2}{2\sigma_v^2}}\,.$$
This formulation corresponds to mixing $n$ Gaussian densities centered at $(\vk_i,\vv_i)_{1\leq i\leq n}$ with respective mixing weights $(w_i)_{1\leq i\leq n}$. The covariance matrix $\Sigma$ for each of the $n$ Gaussian components is a block diagonal matrix given by:
\begin{equation}
\label{eq:block_covariance}
\Sigma = \begin{bmatrix}
\Sigma_k & 0 \\
0 & \Sigma_v
\end{bmatrix}= \begin{bmatrix}
\sigma_k^2 \mI_{d_{\text{in}}} & 0 \\
0 & \sigma_v^2 \mI_{d_{\text{out}}}
\end{bmatrix} \,.
\end{equation}

Now, let us compute the softmax function $s_{w_{\mathcal{D}}}$ induced by this mixture density. Substituting $w_{\mathcal{D}}$ into Eq.~\ref{eq:softmax_with_weight_function}, we can show that (see the full derivation in Appendix \ref{app:skipped_derivation}):
\begin{align*}
    s_{w_{\mathcal{D}}}(\vq) &= \frac{\int_{\mathbb{R}^{d_{\text{in}}} \times \mathbb{R}^{d_{\text{out}}}} w_{\mathcal{D}}(\vk,\vv) e^{\langle \vk, \vq \rangle} \vv \, d\vk \, d\vv}{\int_{\mathbb{R}^{d_{\text{in}}} \times \mathbb{R}^{d_{\text{out}}}} w_{\mathcal{D}}(\vk,\vv) e^{\langle \vk, \vq \rangle} \, d\vk \, d\vv} 
    \\
    &= \frac{\sum_{i=1}^n w_i e^{\langle \vk_i, \vq \rangle + \frac{1}{2}\sigma_k^2 \|\vq\|^2} \vv_i}{\sum_{i=1}^n w_i e^{\langle \vk_i, \vq \rangle + \frac{1}{2}\sigma_k^2 \|\vq\|^2}}
    = \frac{\sum_{i=1}^n w_i e^{\langle \vk_i, \vq \rangle} \vv_i}{\sum_{i=1}^n w_i e^{\langle \vk_i, \vq \rangle}} = s_{\mathcal{D}}(\vq) \,.
\end{align*}
This demonstrates that the continuous softmax function induced by the Gaussian mixture density $w_{\mathcal{D}}$ is exactly equivalent to the discrete weighted softmax function $s_{\mathcal{D}}$ defined by the database $\mathcal{D}$. The shared variance-dependent term $e^{\frac{1}{2}\sigma_k^2 \|\vq\|^2}$ cancels out in the numerator and denominator, rendering the prediction independent of the specific variances $\sigma_k$ and $\sigma_v$ used in the mixture model. It is worth noting that the same cancellation holds for any general block-diagonal covariance matrix $\Sigma$ (where $\vk$ and $\vv$ are independent), regardless of the specific correlations within the key or value dimensions.

\textbf{The functional space of Gaussian-modulated softmax functions.}
Now observing that Gaussian mixtures with a fixed covariance matrix $\Sigma$ are absolutely continuous with respect to the centered Gaussian measure $\mu=\mathcal{N}(\mathbf{0},\Sigma)$, we can analyze the system using this background measure. We define the Gaussian-modulated softmax function for a relative weight function $w$ as:
\begin{equation}
    \label{eq:softmax_with_weight_function_Gaussian}
    s_{w}(\vq) = \frac{\int_{\mathbb{R}^{d_{\text{in}}} \times \mathbb{R}^{d_{\text{out}}}} w(\vk,\vv) e^{\langle \vk, \vq \rangle} \vv \, d\mu(\vk, \vv)}{\int_{\mathbb{R}^{d_{\text{in}}} \times \mathbb{R}^{d_{\text{out}}}} w(\vk,\vv) e^{\langle \vk, \vq \rangle} \, d\mu(\vk, \vv)} \,.
\end{equation}
This formulation relates to the original Lebesgue-density formulation via a Radon-Nikodym derivative. If the physical density of the database is $p(\vk, \vv)$, the corresponding functional weight is $w(\vk, \vv) = \frac{dp}{d\mu}(\vk, \vv)$.\looseness=-1

This perspective reveals a crucial structural property: a Gaussian mixture in the physical space corresponds to a linear combination of exponential functions in the functional space. Specifically, the ratio of a shifted Gaussian $\mathcal{N}(\mathbf{x}|\mathbf{u}, \Sigma)$ to the base measure $\mathcal{N}(\mathbf{x}|\mathbf{0}, \Sigma)$ is proportional to $e^{\langle \mathbf{x}, \Sigma^{-1}\mathbf{u} \rangle}$. Consequently, adding a key-value pair to the database is mathematically equivalent to adding an exponential basis function to the weight $w(\vk, \vv)$. This justifies searching for updates in the span of exponential functions during gradient descent.

Let $\mathcal{H}_{\sigma_k,\sigma_v}$ be the space of functions $w:\mathbb{R}^{d_{\mathrm{in}}}\times\mathbb{R}^{d_{\mathrm{out}}}\to\mathbb{R}$ such that
\begin{align*}
    \mathbb{E}_{(\vk,\vv)\sim\mathcal{N}(\mathbf{0},\Sigma)}[w(\vk,\vv)^2&e^{\langle \vk,\vq\rangle+\langle \vv,\vv'\rangle}]<\infty\,,
    \quad\quad\forall \vq\in\mathbb{R}^{d_{\mathrm{in}}}\,,\forall \vv'\in\mathbb{R}^{d_{\mathrm{out}}}\,.
\end{align*}

We endow $\mathcal{H}_{\sigma_k,\sigma_v}$ with the $L_2(\mu)$ inner product:
$\langle w,w'\rangle_{\mathcal{H}_{\sigma_k,\sigma_v}}=\mathbb{E}_{(\vk,\vv)\sim\mathcal{N}(\mathbf{0},\Sigma)}[w(\vk,\vv)w'(\vk,\vv)]\,.$

\textbf{Functional gradient descent update rule.}
We derive the update rule for the weight function $w \in \mathcal{H}_{\sigma_k,\sigma_v}$ by performing functional gradient descent on a loss function $\mathcal{L}(s_w(\vq), \vy)$. For a single observation $(\vq_t, \vy_t)$, we compute the Fréchet derivative of the loss functional $J(w) = \mathcal{L}(s_w(\vq_t), \vy_t)$ and identify the functional gradient $\nabla_w J$ via the Riesz representation theorem in $\mathcal{H}_{\sigma_k,\sigma_v}$. As detailed in \cref{app:functional_gradient}, the gradient $\nabla_w J$ takes a specific form involving the exponential kernel $e^{\langle \vk, \vq_t \rangle}$ and the error signal:

\vspace*{-1mm}
\begin{equation}
    \nabla_w J(\vk, \vv) = \frac{e^{\langle \vk, \vq_t \rangle}}{D(w)} \langle \nabla_{\hat{\vy}}\mathcal{L}, \vv - s_w(\vq_t) \rangle\,\, \text{with} \,\, D(w)=\int_{\mathbb{R}^{d_{\text{in}}} \times \mathbb{R}^{d_{\text{out}}}} w(\vk,\vv) e^{\langle \vk, \vq_t \rangle} \, d\mu(\vk, \vv)\,.
\end{equation}

$D(w)$ is the denominator of the softmax functional induced by $w$.

The functional update $w_{t+1} \leftarrow w_t - \eta \nabla_w J$ corresponds to adding a new term to the weight function, which is constructively equivalent to appending a new key-value tuple to the database (cf. \cref{app:constructive_update_analysis}). Specifically, this establishes the update rule that adds a tuple $(w_{\text{new}}, \vk_{\text{new}}, \tilde{\vv}_{\text{new}})$ where the new key is a scaled version of the query, $\vk_{\text{new}} = \Sigma_k \vq_t= \sigma_k^2 \vq_t$, and the weight and value are determined by the prediction error and the local density. To ensure numerical stability, we store the weighted value vector $\tilde{\vv} = w \vv$, leading to the robust update rules:
\begin{equation}
     w_{\text{new}} \propto \eta \langle \vs_t, \vg_t \rangle, \quad \tilde{\vv}_{\text{new}} \propto -\eta \Sigma_v \vg_t\propto -\eta \sigma_v^2 \vg_t \,,
     \label{eq:update_rule_functional_gd_wsm}
\end{equation}
where $\vg_t = \nabla_{\hat{\vy}} \mathcal{L}$ and $\vs_t = s_{w_t}(\vq_t)$. Refer to \cref{app:functional_gradient} for more details.
In the case of mini-batch stochastic gradient descent with a batch of size $B$, the linearity of the gradient operator implies that the total functional gradient is the sum of individual gradients centered at distinct queries. Consequently, a single step corresponds to adding $B$ distinct tuples to the database simultaneously, reflecting the non-parametric nature of the approach where model complexity grows with data. 

\begin{remark}
\label{rem:weighted_softmax_remark}
Given that $\vk_{\text{new}}=\sigma_k^2 \vq_t$, one can interpret $\sigma_k^2$ as the inverse temperature of the resulting softmax function. In practice, we can store $\vk_{\text{new}}=\vq_t$ and instead use a temperature hyperparameter $T$ in the Softmax, as is typically done\footnote{$\sigma_k^2$ relates to the temperature because increasing $\sigma_k^2$ scales the positions of new keys quadratically, causing them to spread apart faster than the linear expansion of the standard deviation of individual Gaussian components. Consequently, higher variance $\sigma_k^2$ acts as a separation parameter, reducing component overlap and sharpening the attention distribution.}, i.e.,
the final weighted-softmax function becomes:
\begin{equation}
    \label{eq:weighted_softmax_with_weighted_value}
    s_{\mathcal{D}}(\vq) = \frac{\sum_{(w,\vk,\tilde{\vv})\in\mathcal{D}} e^{\frac{1}{T\sqrt{d_{\text{in}}}} \langle \vk, \vq \rangle} \tilde{\vv}}{\sum_{(w,\vk,\tilde{\vv})\in\mathcal{D}} e^{\frac{1}{T\sqrt{d_{\text{in}}}}\langle \vk, \vq \rangle}w}  \,.
\end{equation}
In Appendix \ref{app:optimization_dynamics}, we provide further remarks regarding extensions of the learning algorithm above (corresponding to vanilla gradient descent) with weight decay and momentum.
\end{remark}
\looseness=-1

\subsubsection{Handling negative weights: clipping and gauge invariance} 
\label{sec:gauge}
One challenge with the update rule of \eqref{eq:update_rule_functional_gd_wsm} is that when $\langle \vs_t, \vg_t \rangle < 0$ we get a \emph{negative} weight $w_{\text{new}} < 0$. This may cause denominator cancellation and numerical instability (cf. \cref{app:negative_weights}). A simple heuristic to enforce non-negativity is via clipping, $w_{\text{new}} \leftarrow \max(0, w_{\text{new}})$.

Another method for addressing the negative weights issue without clipping is via a translational symmetry of the softmax functional, which we call \emph{gauge-invariance} (cf. \cref{app:gauge-invariance}): If we shift all the value vectors in the database by $+\vb$, the output of the weighted softmax also shifts by $+\vb$. Hence, if we then subtract $\vb$ from the output, we recover the original weighted softmax. Therefore, by equipping the softmax functional with an auxiliary global bias $\vb$, the model gains a translational symmetry between the values and the bias (cf. Proposition~\ref{prop:gauge_invariance}). This means that the same function has a different representation for each offset (gauge). We can dynamically choose an offset frame---a ``virtual gauge'' $\vb_t$ along the gradient direction---that guarantees that the effective projection $\langle \vg_t, \vs_t + \vb_t \rangle$ is strictly positive, ensuring $w_{\text{new}} > 0$. Refer to \cref{app:gauge-invariance} for a detailed description of this method.

\subsection{Sublinear Growth with Hybrid Non-Parametric and Parametric Learning}
\label{sec:sublinear_growth}
Appending a new key-value tuple at every training step incurs a linear memory growth $\mathcal{O}(T)$ with respect to the number of training steps $T$. To achieve scalable, sublinear growth, we introduce a hybrid learning rule that dynamically balances non-parametric expansion with parametric refinement: We treat the existing key-value pairs in the database as learnable parameters optimized via standard optimizers (e.g., Adam or AdamW), while conditionally growing the database using functional gradient descent only when a newly encountered key $\vx$ is sufficiently distinct from the existing keys $\vk_j$ in the database, based on a threshold parameter $\tau$, i.e., $\min_j \lVert \vx - \vk_j \rVert_2 > \tau$. By restricting memory growth to genuinely novel data while leveraging parametric optimization for representation tuning, this hybrid mechanism dramatically curbs memory footprint while preserving the performance advantages of the fully non-parametric model (see \cref{fig:sub-linear-growth}).

\section{Advanced Linear Attention Perspectives}
\label{sec:mesa}

Sec.~\ref{sec:attention} focused on improving the LA underlying the dual form of NNs (Eq.~\ref{eq:gd_linear_attention}) by introducing nonlinearity.
Here we plug advanced LA models into it, and show how they lead to improved NN optimizers.\looseness=-1

We focus on the mesa-layer \citep{von2023uncovering,von2025mesanet}, which replaces the linear attention of Eq.~\ref{eq:gd_linear_attention} by the best possible linear key-value map, in the sense that it achieves the lowest (regularized) squared error over the keys and values seen up to time $t$:
\vspace*{-2mm}
\begin{align}
    \mW_t &= \argmin_\mW \frac{1}{2} \sum_{\tau=1}^t \|\vv_\tau - \mW \vx_\tau \|^2 + \frac{\lambda}{2} \|\mW\|^2_F\,,
    \label{eq:mesa_layer}
\end{align}
where $\lambda > 0$ is a static (scalar) hyperparameter which controls the strength of regularization, given by the squared Frobenius norm $\|\mW\|^2_F$ of the weights. Note that the mesa-layer considers the cumulative (in time) loss, and fully minimizes it. \citet{von2025mesanet} empirically found that the mesa-layer outperforms both standard LA and the DeltaNet on various sequence modeling problems.

Interestingly, as we show next, the dual view of a linear layer trained by gradient descent yields an approximate second-order optimizer once we replace linear attention (Eq.~\ref{eq:gd_linear_attention}) by the mesa-layer (Eq.~\ref{eq:mesa_layer}). The only remaining step is to determine what the values $\vv_t$ should be. In this section, we set them to be the gradient-corrected layer outputs:
\begin{equation}
\label{eq:mesa_value}
    \vv_t = \mW_{t-1} \vx_t - \eta_t \nabla_{\vy_t} \mathcal{L}.
\end{equation}
An important property of the objective of the mesa-layer (Eq.~\ref{eq:mesa_layer}) is that it can be solved analytically. We use this property to obtain our optimization algorithm, which we develop for the mini-batch setting---the current standard in deep learning. Before proceeding, we introduce additional notation to handle mini-batches. We denote by $T$ the number of batches seen so far, each containing $B$ data points. In this setting, $\mX_T$ is the $d_\text{in} \times B$ matrix that holds the current batch inputs in its columns, and $\mV_T$ is the $d_\text{out} \times B$ matrix that holds the corresponding gradient-corrected outputs in its columns, following Eq.~\ref{eq:mesa_value}. Additionally, we define $\mX_{<T}$ and $\mV_{<T}$ as the history of all inputs and gradient-corrected outputs (resp.) observed thus far, excluding the current batch. Solving Eq.~\ref{eq:mesa_layer} analytically, we obtain a closed-form expression for the weights at iteration $T$ i.e., $
    \mW_T = (\mV_{<T}  \mX_{<T}^\top + \mV_T \mX_T^\top) \cdot (\mX_{<T}  \mX_{<T}^\top + \mX_{T} \mX_{T}^\top + \lambda \mI)^{-1}.$
It turns out that the resulting layer implements exactly a recently proposed optimizer \citep[FOOF;][]{benzing2022foof}, which has been shown to be highly efficient in deep learning. FOOF was derived as a simplification  of KFAC \citep{martens2015kfac}, itself an approximate second-order optimization method. Intuitively, the optimization speed-up that FOOF (here, the mesa-layer) enjoys stems from the additional inverse matrix which is not present in linear attention. This additional factor applies a whitening transform to the inputs, which yields a full (second-order) Newton update for a shallow linear model, and an approximate second-order update for a deep neural network. We refer to \citet{benzing2022foof} for an analysis of this optimizer, and to Appendix~\ref{app:mesa_foof} for additional details.\looseness=-1

As an additional reference, we further consider the DeltaNet \citep{schlag2021linear} case. As derived from the delta rule \citep{widrow1960adaptive}, this case corresponds to performing an online gradient descent step w.r.t.~$\mW$ on the layer-local, time-dependent loss penalizing the discrepancy between the prediction $\mW \vx_t$ and a target value $\vv_t=-\eta\nabla_{\vy_t} \mathcal{L}$:
\begin{equation}
    \ell_t(\mW) := \frac{1}{2} \| \mW\vx_t - \vv_t \|_2^2 = \frac{1}{2} \|\mW\vx_t + \eta \nabla_{\vy_t} \mathcal{L} \|_2^2
    \label{eq:delta-loss}
\end{equation}
with learning rate $\lambda$ (corresponding to an online version of the mesa-layer case). This yields the update $\mW_t = \mW_{t-1} + \lambda (- \eta \nabla_{\vy_t} \mathcal{L} - \mW_{t-1}\vx_t)\vx_t^\top$.  Appendix~\ref{app:mesa_foof} provides additional details.\looseness=-1

\section{Experiments}
\label{sec:exp}
Here we conduct experiments to empirically validate  theories arising from the two extended views of linear layers above: the newly derived growing neural networks and their learning rules (Sec.~\ref{sec:attention}) and advanced optimizers (Sec.~\ref{sec:mesa}).

\textbf{Basic Settings.} The baseline architecture is a residual MLP similar to the transformer's MLP block \citep{trafo}; the model has multiple layers with a block structure, where every block of the model updates the residual stream $\ve \in \mathbb{R}^{d_e}$ following
\vspace*{-1mm}
\begin{equation}
\label{eq:mlp-residual-update}
    \ve \gets \ve + \mW_2 \sigma (\mW_1 \text{RMSNorm}(\ve)),
\end{equation}
with nonlinearity $\sigma(\cdot)$ and RMSNorm prenormalization \citep{ZhangS19a}. The first block is applied to the net input $\vx \in \mathbb{R}^{d_x}$.
The output layer is the standard softmax layer for classification or a linear layer for regression, using the cross-entropy or mean-squared error loss, respectively.

To find a strong baseline, we perform intensive search over the choice of activation $\sigma(\cdot)$, namely: ReLU \citep{fukushimaRelu,NairH10}, GELU \citep{hendrycks2016gaussian} and softmax \citep{bridle1990softmax}; as well as over the space of optimizers, including plain SGD, Adam \citep{kingma2014adam}, AdamW \citep{LoshchilovH19}, RMSProp \citep{hinton2012neural}, and Muon \citep{jordan2024muon}. These serve as the baselines for the MesaNet/DeltaNet-derived optimizers introduced in Sec.~\ref{sec:mesa}. We perform a fine-grained scan over hyperparameters, including model depth and nonlinearity type. All our results are averaged over 5 seeds, and shaded areas reflect stdev.~unless noted otherwise (see Appendix~\ref{app:data} for more details).
Our nonlinear attention-based growing nets (Sec.~\ref{sec:attention}) follow the same residual architecture, but perform:\looseness=-1
\begin{equation}
    \ve \gets \ve + g_\mathcal{D}(\text{RMSNorm}(\ve)),
\end{equation}
where $g_\mathcal{D}(\cdot)$ can either be the weighted-softmax layer (Eq.~\ref{eq:weighted_softmax}) or the RBF layer (Eq.~\ref{eq:gd_linear_attention}, with $\phi$ an RBF).\looseness=-1

\begin{wrapfigure}[19]{r}{0.4\textwidth}
\vspace{-0.15cm}
\begin{center}
\centerline{\includegraphics[width=0.455\textwidth]
{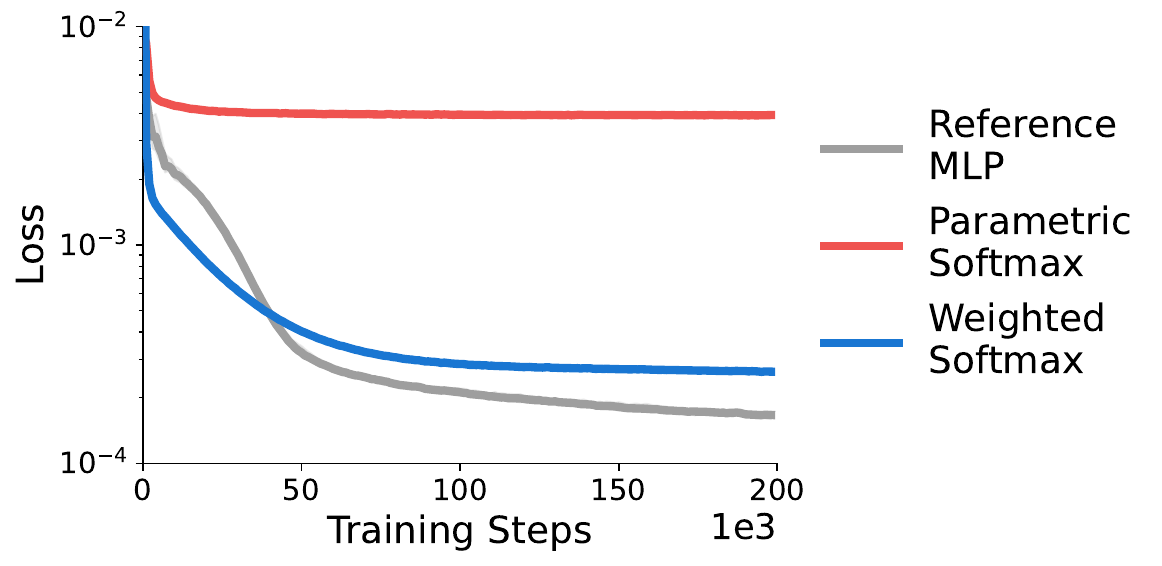}}

\caption{Teacher-student regression loss (mean $\pm$ s.e.m.) comparing: (blue, Weighted Softmax) our growing weighted softmax model; (red, Parametric Softmax) a very wide  MLP baseline with width matching the final number of training samples using the softmax activation; and (gray, Reference MLP) an overparameterized student with matched architecture providing a lower bound.}
\label{fig:teacher-student}
\end{center}
\end{wrapfigure}

The final output layer is similar but without the residual connection.
Importantly, note that the formal learning rules we derived for these growing nets correspond to training them using plain SGD (Sec.~\ref{sec:attention}); derivations for more sophisticated optimizers are left for future work.

\textbf{Teacher-Student Experiments.}
We begin with experiments on a synthetic regression problem, where scalar target outputs $y^* \in \mathbb{R}$ are generated by a held-out residual MLP teacher $y^*(\vx)$ from 100-dimensional inputs $\vx \in \mathbb{R}^{100}$ \citep{gardner1989three}. The teacher NN has 1 hidden layer with ReLU non-linearity, residual-stream $d_e=100$, and random, fixed weights sampled from a LeCun normal distribution (fan-in). The teacher then projects the residual stream activation to an output dimension of $1$. The student $y(\vx)$ is trained through online SGD on the expected squared error loss $\mathbb{E}_{\vx \sim p(\vx)}[(y^*(\vx) - y(\vx))^2]$, by drawing new data samples at every training iteration. The inputs $\vx$ are generated by first sampling from a standard normal distribution, and then projecting to the unit sphere.
Appendix~\ref{app:experimental} provides further setup details.

In this task, an overparameterized student---with the same architecture as the teacher, but larger width---is expected to perform very strongly \citep{simsek2021geometry}. In Fig.~\ref{fig:teacher-student} we see that this is indeed the case; the MLP student (denoted by ``Reference MLP'' in the plot) outperforms all other models, and serves as the gold standard for this task. 

We then compare our Growing Weighted Softmax model (GWS; Sec.~\ref{sec:softmax}) with a fixed-size ``parametric softmax'' baseline that uses softmax activation (Eq.~\ref{eq:mlp-residual-update}) and dimensions matching the full database size—closely mirroring the GWS model by treating $\mW_1$ and $\mW_2$ as learnable keys and values \citep{gevaetal2021transformer}.

We compare a single-layer GWS with a single-block parametric softmax baseline (with two trainable weight matrices; the best model trained with Adam). We find that GWS significantly outperforms the parametric baseline. Despite the architectural mismatch with the teacher, and despite only having one trainable layer (comprising values and respective weights), our growing weighted softmax model can reach significantly lower loss values.\looseness=-1

\begin{figure}[t]
    \centering

    \setlength{\tabcolsep}{4pt} 
    
    \begin{tabular}{@{} c c : c : c @{}}
        
        \includegraphics[width=0.24\textwidth]{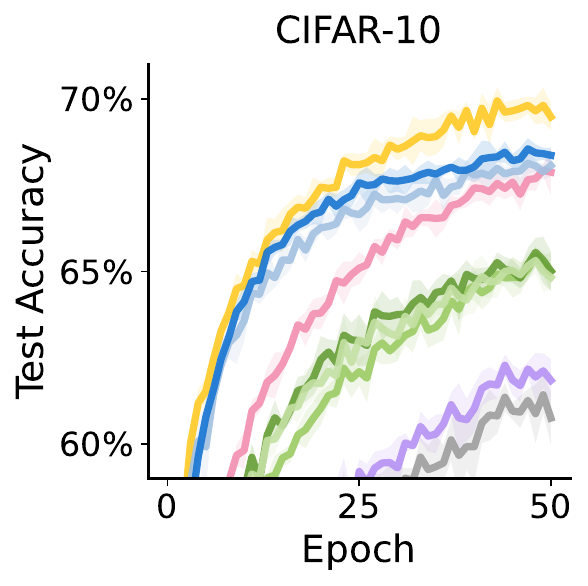} &
        \includegraphics[width=0.24\textwidth]{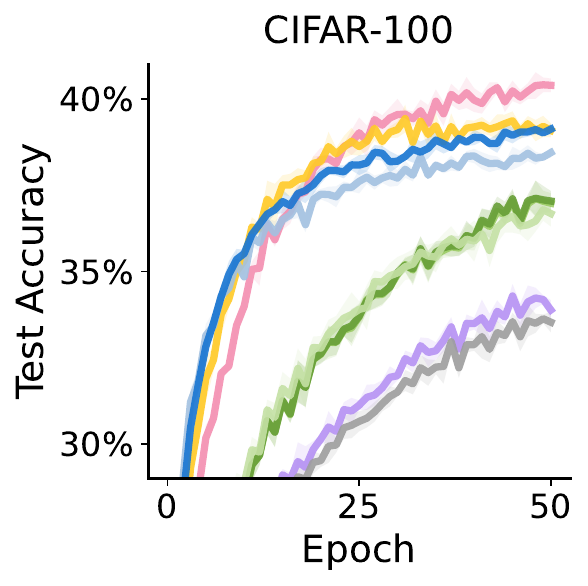} &
        \includegraphics[width=0.24\textwidth]{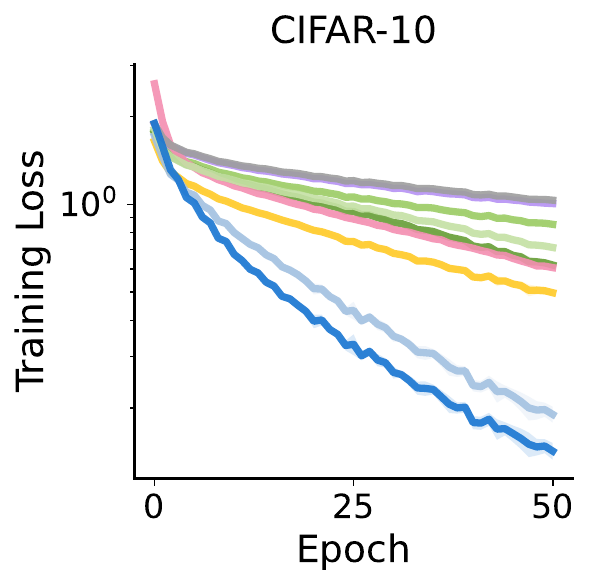} &
        \includegraphics[width=0.24\textwidth]{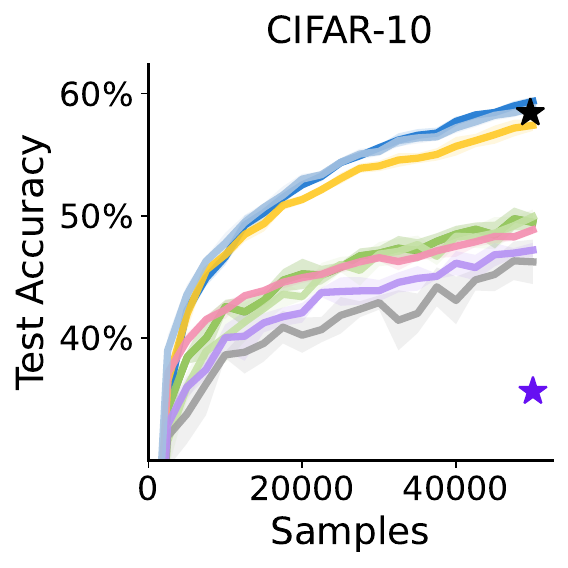} \\
        
    \end{tabular}
    
    \includegraphics[width=\textwidth]{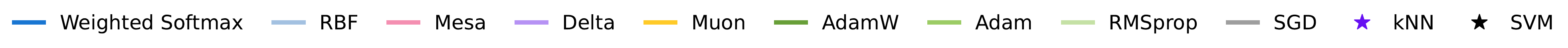}
    
    
    \caption{(Left Plots) Test set accuracy on CIFAR-10/100: Growing Weighted Softmax \& RBF, while based on plain SGD, significantly outperform the SGD-trained MLP and Adam/RMSProp and are competitive with MLPs trained with Muon. The Mesa/FOOF optimizer achieves high accuracy on CIFAR-10 and performs best on CIFAR-100. (Middle Right) Training loss for Growing Weighted Softmax \& RBF  decreases much faster than that of the other models. (Right) Test set accuracy on CIFAR-10 after a \textit{single pass} on the training data. Growing Weighted Softmax \& RBF are the most efficient learning algorithms. To place our results, we include the test accuracy of two classical (not single pass) methods: k-nearest-neighbor (kNN) and SVM (starred; details in Appendix~\ref{app:experimental:nonparam}).} 
    \label{fig:cifar-test-acc}
    \vspace{-7mm}
\end{figure}

\textbf{CIFAR-10/100 Experiments.}
Next, we investigate the performance of the growing
RCDL models on the classic CIFAR-10 and CIFAR-100 image classification tasks \citep{krizhevsky}, using standard data augmentation and non-convolutional architectures.
Fig.~\ref{fig:cifar-test-acc} reports the performance of several parametric baselines and the growing nonparametric models based on RBF and weighted softmax.

Fig.~\ref{fig:cifar-test-acc}/Left shows test accuracy on CIFAR-10/100.
We observe that both the growing weighted softmax and RBF models (Sec.~\ref{sec:attention}), whose learning rules are based on plain SGD, significantly outperform the baselines trained by plain SGD, as well as Adam, AdamW and RMSProp; and
the GWS model slightly outperforms the RBF variant.

Using more advanced optimizers, the parametric MLP eventually surpasses the SGD-based growing models: the MLP trained by the MesaNet-derived (FOOF) optimizer outperforms all methods on CIFAR-100 and performs on par with the growing variants on CIFAR-10 (see Appendix \ref{app:experimental:DeltaNet} for further discussion on the DeltaNet-based one). Similarly, Muon proves to be a strong baseline on both datasets achieving top accuracy on CIFAR-10. These results highlight the importance of developing learning rules for growing models that correspond to these advanced optimizers. 

 Fig.~\ref{fig:cifar-test-acc}/Middle reports the evolution of training loss on CIFAR-10. Interestingly, the training loss decreases drastically faster for the two nonparametric models compared to all the other methods, while having good generalization performance as shown above; this may be a promising property for tasks that require both good generalization and strong memorization.

\begin{figure}
\vspace{-0.7cm}
\begin{minipage}[t]{0.48\textwidth}
\begin{center}
\centerline{\includegraphics[width=1.\columnwidth]{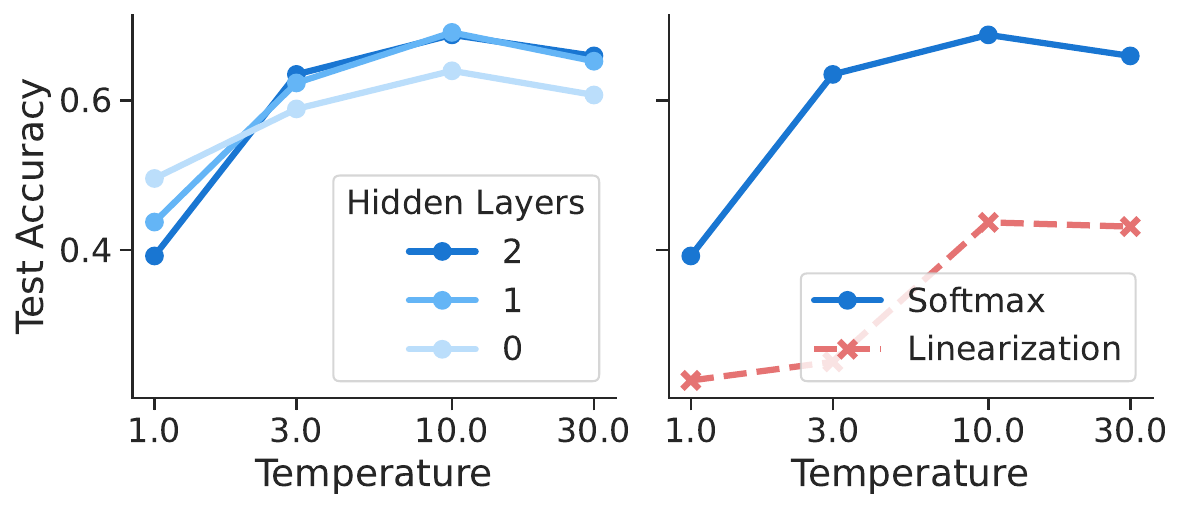}}

\caption{(Left) The growing weighted softmax model trained on CIFAR-10 improves with depth, and displays a non-monotonic temperature behavior w.r.t.~test set acc. For the optimal value of $T=10$ and $d_{\text{in}}=3072$, $\frac{1}{T\sqrt{d_{\text{in}}}} \langle \vk, \vq \rangle\in[-5.54,5.54]$ implying that the softmax layer is operating in the non-linear regime. (Right) This is further confirmed by linearizing post-training the ``exp'' nonlinearity of the layer, which results in a significant performance degradation.\looseness=-1
}

\label{fig:deep-representations-and-temperature}

\end{center}
\end{minipage}
\hfill
\begin{minipage}[t]{0.48\textwidth}
\begin{center}
\centerline{\includegraphics[width=1.\columnwidth]{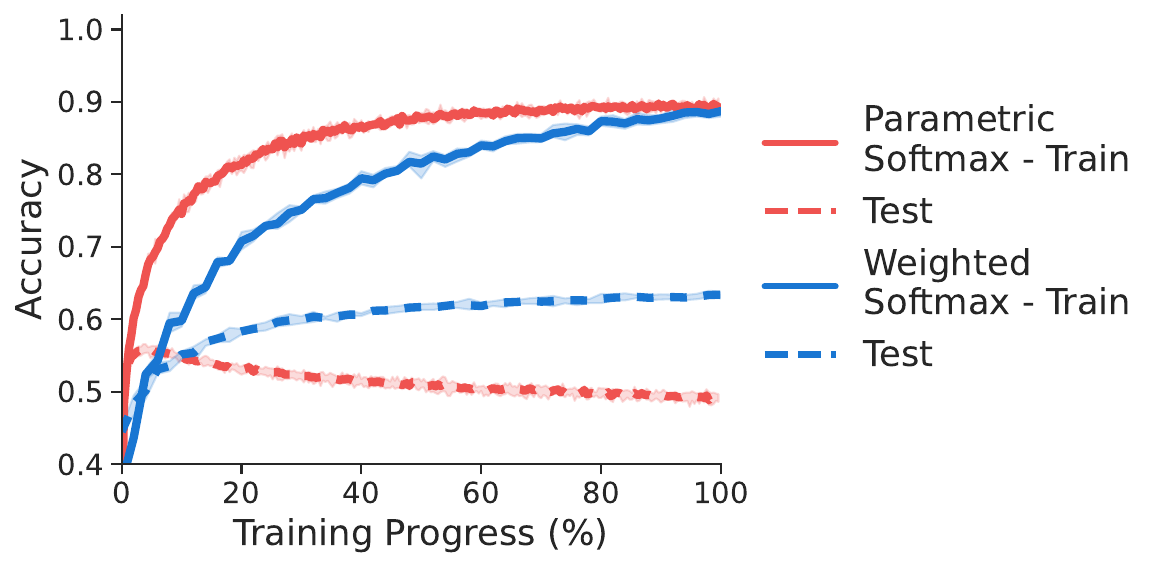}}
\caption{Comparison between parametric softmax (very wide MLP with softmax nonlin.) and our growing non-parametric weighted softmax model. Despite sharing the same architecture, the models behave very differently, both in terms of optimization and generalization: The parametric model displays a widening train-test gap, that is not present in its nonparametric weighted counterpart. We train the former for 2000, the latter for 50 epochs.}
\label{fig:bbb}
\end{center}
\end{minipage}
\vspace{-1cm}
\end{figure}

Finally, Fig.~\ref{fig:cifar-test-acc}/Right shows the learning efficiency of all methods when trained only on a single pass over non-augmented CIFAR-10 dataset. In this online regime, the growing methods outperform all parametric models including Muon and Mesa. We refer to \cref{app:weighted_softmax_exp} for further experimental results and \cref{app:gauge_comparison} for results using the gauge-invariance version of the GWS layer (described in \cref{sec:gauge}).  

\begin{wrapfigure}[15]{r}
{0.4\textwidth}
\vspace{-1em}
\begin{center}
\centerline{\includegraphics[width=0.35\textwidth]{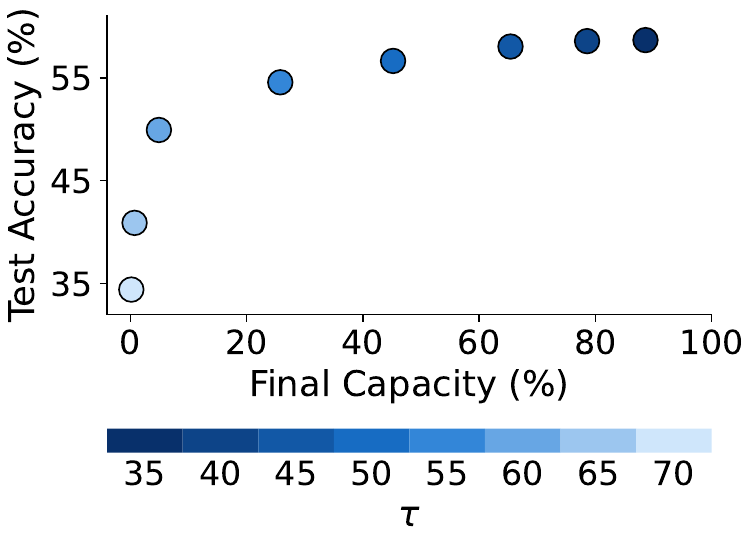}}
\vspace{-0.1cm} 
\caption{CIFAR-10 performance vs final capacity for various growth thresholds $\tau$. The model retains high test acc. while exhibiting graceful performance degradation when compressing further.}
\label{fig:sub-linear-growth}
\end{center}
\end{wrapfigure}

\textbf{Analysis.} Here we analyze how the GWS  model's test performance varies as a function of the number of layers and its  temperature (Sec.~\ref{sec:softmax}), and verify that the model effectively leverages the nonlinear kernel.\looseness=-1

Fig.~\ref{fig:deep-representations-and-temperature} shows that adding hidden layers improves the performance on CIFAR-10 (though performance saturates with 2 hidden layers), and that there is a non-monotonic dependence on temperature.
We also confirm that the model is in the nonlinear regime by ablating the trained model by linearizing the exponential function (i.e., using the approximation $\exp(x) \approx 1 + x$) in the weighted softmax, post-training, which leads to a severe drop in performance.



We further compare the training and generalization behaviors of the GWS model with that of the parametric softmax baseline. Figure \ref{fig:bbb} shows their training and test accuracy on CIFAR-10 as a function of the training progress.
We observe that their behaviors are very different:
despite being able to train a huge set of parameters at every iteration, the parametric baseline is slower to fit the training data, despite being 
highly prone to overfitting; we could not find a hyperparameter configuration that avoided a widening train-test performance gap.
The GWS model has no such issues and reaches much higher test accuracy overall.\looseness=-1

\textbf{Mixing retrieval-centric and parametric layers.}
It is also straightforward to mix RCDL layers and parametric layers within the same model. To illustrate this, we conduct a sanity check experiment on vision transformers (ViTs) \cite{dosovitskiy2021an} where we replace the feed-forward block in every ViT layer with an RCDL layer as a drop-in replacement, while keeping all other layers (e.g., self-attention) parametric and training them by gradient descent. Trained for a single epoch on CIFAR-10, the purely parametric ViT achieves $48.4\%$ while our mixed model achieves $54.4\%$ test set accuracy, a promising result opening up many avenues for future research. For more details, see Appendix~\ref{app:experimental:vit}.\looseness=-1

\textbf{Sub-linear growing.}
Finally, to mitigate the prohibitive memory and compute growth of vanilla RCDL layers, we provide preliminary results of the hybrid layer described in \cref{sec:sublinear_growth} which combines parametric learning with (conditional) non-parametric growth, again when training CIFAR-10 for one epoch. As shown in \cref{fig:sub-linear-growth}, the memory growth can be made sublinear while preserving strong performance. Further experimental details can be found in Appendix~\ref{app:experimental:hybrid}.

\vspace{-2mm}
\section{Discussion, Related Work, and Conclusion}
\label{sec:discussion}
\vspace{-2mm}

The RCDL concept has connections to many existing methods in machine learning, such as kernel machines \citep{aizerman1964theoretical,BoserGV92, burges1998tutorial,scholkopf2002learning} (further related work can be found in Appendix \ref{app:related_work}).
In particular, the principled derivation of our update rules---particularly our novel rules for the weighted softmax kernel---closely parallels functional gradient boosting \citep{mason1999boost, friedman2000greedy}. While this perspective remains relatively underexplored in deep learning, prior work \citep{grubb2010boosted} derived an algorithm for constructing networks by iteratively adding differentiable modules trained to match the functional gradient of the loss; moreover, the same work presented a functional gradient-based rule for reproducing kernel Hilbert spaces, which essentially corresponds to our result of Sec.~\ref{sec:rbf}. We complement this early study with modern deep learning experiments, and with our novel (weighted) softmax attention.

Furthermore, RCDL's growing network is fundamentally a cascade of kernel machines \citep{aizerman1964theoretical,BoserGV92, burges1998tutorial,scholkopf2002learning}, with features in every layer trained in an end-to-end fashion through gradient descent---a core principle of deep learning. This has an interesting difference to \citet{domingos2020every}'s view on how an entire deep NN trained by gradient descent can be \textit{approximated} by a kernel machine.
In contrast, the development of RCDL builds on the \textit{exact} layer-wise equivalence/duality between a linear layer trained by gradient descent and a kernel machine \citep{aizerman1964theoretical,irie2022dual}. This might open new avenues to study optimization in NNs.  \looseness=-1

RCDL may also open up a new path for continual learning. Unlike conventional weight-based learning which involves possibly irreversible compression during training, RCDL preserves past activations; we may potentially augment RCDL with some algorithms that revive old memories (e.g., within the forward pass of the net). Loosely, this property also aligns with the neurobiological findings on ``silent memory'' \citep{ryan2015engram,roy2017silent,gershman2025key} where forgetting occurs due to retrieval failures (not because the actual memory content is lost) and some (optogenetic) procedure can revive the apparently lost memory. \looseness=-1

\textbf{Limitations.}
RCDL has several open challenges. One of the biggest concerns is the memory and compute requirements that linearly grow with the number of training datapoints, which is generally prohibitive.
While some of the key-value pruning methods from the transformer literature (e.g., sliding window or token merging \citep{BolyaFDZFH23}) may provide some insights, deriving a more principled solution for RCDL remains a future research topic. At the same time, the internet (and the corresponding data storage) itself continues to grow, and memory continues to be cheaper as hardware advances---ever growing RCDL may not be completely inconceivable in a distant future. Another notable limitation is the class of base optimization algorithms currently supported by our theory; while we covered some of the basic concepts such as weight decay and momentum to augment vanilla gradient descent, we do not know yet how to derive more advanced optimizers such as Adam \citep{kingma2014adam} and Muon \citep{jordan2024muon} for RCDL. \looseness=-1

\textbf{Conclusion.}
Retrieval-centric deep learning (RCDL) aims to improve upon conventional deep learning by replacing the ``linear attention'' (LA) underlying the classic linear layer trained via gradient descent---a fundamental building block of deep learning---with more powerful attention mechanisms.
While replacing LA with MesaNet/DeltaNet-inspired update rules yields efficient optimizers for conventional weight-based neural nets, replacing it with kernelized and softmax attention leads to unconventional nonparametric networks that learn by growing a key-value memory database instead of a fixed-size weight matrix. Using techniques from functional analysis and functional gradient methods, we derived gradient-based learning rules to train deep networks composed of the latter, and we showed empirically that nonparametric deep networks perform strongly on standard benchmarks. Our findings open a number of avenues for future work, such as investigating how to compress the growing nonparametric layers or integrate such layers into language models.

\bibliography{references}
\bibliographystyle{unsrtnat}
\clearpage


\appendix

\section{Derivation of the Growing Layer Algorithm for Positive Definite Kernels}
\label{app:rbf}

In this appendix, we provide the formal derivation of the learning algorithms for the growing kernel networks introduced in Section \ref{sec:rbf}. We establish the correspondence between the ``growing database'' update rules and stochastic gradient descent (SGD) performed on an implicit, infinite-dimensional weight matrix defined in a Reproducing Kernel Hilbert Space (RKHS).

\subsection{The Growing Kernel Layer}

A growing kernel layer with input dimension $d_{\text{in}}$ and output dimension $d_{\text{out}}$ is defined by a tuple $(\phi, \mathcal{D})$, where:
\begin{itemize}
    \item $\phi: \mathbb{R}^{d_{\text{in}}} \times \mathbb{R}^{d_{\text{in}}} \to \mathbb{R}$ is a positive definite kernel.
    \item $\mathcal{D} = \{(\vk_i, \vv_i)\}_{i=1}^M$ is a finite database of key-value pairs, where $\vk_i \in \mathbb{R}^{d_{\text{in}}}$ and $\vv_i \in \mathbb{R}^{d_{\text{out}}}$.
\end{itemize}

The layer implements a function $f_{\phi, \mathcal{D}}: \mathbb{R}^{d_{\text{in}}} \to \mathbb{R}^{d_{\text{out}}}$ defined as:
\begin{equation}
    f_{\phi, \mathcal{D}}(\vq) = \sum_{(\vk, \vv) \in \mathcal{D}} \phi(\vk, \vq) \vv \,.
    \label{eq:kernel_layer_func}
\end{equation}
Unlike standard parametric layers with fixed dimensions, this layer ``grows'' by appending new key-value pairs to $\mathcal{D}$. Deep growing kernel networks are constructed by composing such layers, interleaving them with element-wise non-linearities $\sigma$. For example, a two-layer network is represented as $F(\vx) = \sigma(f_{\phi_2, \mathcal{D}_2}(\sigma(f_{\phi_1, \mathcal{D}_1}(\vx))))$.

\subsection{The Implicit Weight Matrix View}

By the Moore-Aronszajn theorem, the positive definite kernel $\phi$ is associated with a Reproducing Kernel Hilbert Space (RKHS) $\mathcal{H}_\phi$ and a feature map $\psi_\phi: \mathbb{R}^{d_{\text{in}}} \to \mathcal{H}_\phi$ such that $\phi(\vx, \vx') = \langle \psi_\phi(\vx), \psi_\phi(\vx') \rangle_{\mathcal{H}_\phi}$.
Since $\mathcal{H}_\phi$ is typically separable (e.g., for the RBF kernel), it is isomorphic to $\ell_2(\mathbb{N})$. We can thus view $\psi_\phi(\vx)$ as an infinite-dimensional vector.

Substituting this inner product structure into Eq.~\ref{eq:kernel_layer_func}, we obtain:
\begin{align}
    f_{\phi, \mathcal{D}}(\vq) &= \sum_{(\vk, \vv) \in \mathcal{D}} \vv \langle \psi_\phi(\vk), \psi_\phi(\vq) \rangle_{\mathcal{H}_\phi} \nonumber \\
    &= \left( \sum_{(\vk, \vv) \in \mathcal{D}} \vv \psi_\phi(\vk)^\top \right) \psi_\phi(\vq) \nonumber \\
    &= \mW_{\phi, \mathcal{D}} \psi_\phi(\vq) \,,
\end{align}
where $\mW_{\phi, \mathcal{D}} = \sum_{(\vk, \vv) \in \mathcal{D}} \vv \psi_\phi(\vk)^\top$ is an implicit linear operator (or infinite matrix) mapping from $\mathcal{H}_\phi$ to $\mathbb{R}^{d_{\text{out}}}$.
This formulation reveals that a growing kernel layer is theoretically equivalent to a linear layer with a weight matrix $\mW \in \mathbb{R}^{d_{\text{out}} \times \infty}$ acting on the feature space, where the weight matrix has a finite rank determined by the size of $\mathcal{D}$.

\subsection{Derivation of the Learning Algorithm}

We now derive the update rule for the database $\mathcal{D}$ by minimizing a global loss function $\mathcal{L}$ via Stochastic Gradient Descent (SGD) on the implicit weights $\mW$.

Consider a deep kernel network $F$ and a specific layer $\ell$ with implicit weight matrix $\mW^{(\ell)}$, kernel $\phi^{(\ell)}$, and feature map $\psi^{(\ell)}$. Let $h_{\text{in}}^{(\ell)}(\vx)$ denote the input to this layer and $h_{\text{out}}^{(\ell)}(\vx) = \mW^{(\ell)} \psi^{(\ell)}(h_{\text{in}}^{(\ell)}(\vx))$ denote the output.
Given a mini-batch of training data $B = \{(\vx_b, \vy_b)\}_{b=1}^{|B|}$, the loss is $\mathcal{J} = \sum_{b=1}^{|B|} \mathcal{L}(\vx_b, \vy_b; \mW^{(\ell)})$.

The gradient of the loss with respect to the implicit matrix $\mW^{(\ell)}$ is computed via the chain rule. Let $\mathbf{e}_b^{(\ell)}$ be the error signal (gradient of the loss with respect to the layer output) for sample $b$:
\begin{equation}
    \mathbf{e}_b^{(\ell)} := \nabla_{h_{\text{out}}^{(\ell)}} \mathcal{L}(\vx_b, \vy_b) \in \mathbb{R}^{d_{\text{out}}^{(\ell)}} \,.
\end{equation}
The gradient with respect to $\mW^{(\ell)}$ is the outer product of the error signal and the layer input features:
\begin{equation}
    \nabla_{\mW^{(\ell)}} \mathcal{J} = \sum_{b=1}^{|B|} \mathbf{e}_b^{(\ell)} \psi^{(\ell)}(h_{\text{in}}^{(\ell)}(\vx_b))^\top \,.
\end{equation}
Applying the SGD update rule $\mW_{t+1}^{(\ell)} = \mW_t^{(\ell)} - \eta \nabla_{\mW^{(\ell)}} \mathcal{J}$ yields:
\begin{equation}
    \mW_{t+1}^{(\ell)} = \mW_t^{(\ell)} - \eta \sum_{b=1}^{|B|} \mathbf{e}_b^{(\ell)} \psi^{(\ell)}(h_{\text{in}}^{(\ell)}(\vx_b))^\top \,.
\end{equation}
We now examine the function computed by the layer at step $t+1$ for an arbitrary query $\vq$:
\begin{align}
    f_{t+1}^{(\ell)}(\vq) &= \mW_{t+1}^{(\ell)} \psi^{(\ell)}(\vq) \nonumber \\
    &= \mW_t^{(\ell)} \psi^{(\ell)}(\vq) + \sum_{b=1}^{|B|} \left( -\eta \mathbf{e}_b^{(\ell)} \right) \psi^{(\ell)}(h_{\text{in}}^{(\ell)}(\vx_b))^\top \psi^{(\ell)}(\vq) \nonumber \\
    &= f_t^{(\ell)}(\vq) + \sum_{b=1}^{|B|} \vv_{\text{new}, b} \phi^{(\ell)}(\vk_{\text{new}, b}, \vq) \,,
\end{align}
where $\vk_{\text{new}, b} = h_{\text{in}}^{(\ell)}(\vx_b)$ and $\vv_{\text{new}, b} = -\eta \mathbf{e}_b^{(\ell)}$.
This confirms that performing SGD on the implicit infinite-dimensional weight matrix is exactly equivalent to appending the current batch inputs (as keys) and their scaled negative gradients (as values) to the database $\mathcal{D}$.

\subsection{Optimization Dynamics: Weight Decay and Momentum}

We extend the framework to include weight decay and momentum, deriving their dual forms in the growing database context.

\subsubsection{Weight Decay and Database Pruning}
Consider the $L_2$-regularized objective $\mathcal{L}_{reg}(\mW) = \mathcal{L}(\mW) + \frac{\gamma}{2\eta} \|\mW\|_{\text{HS}}^2$, where $\|\cdot\|_{\text{HS}}$ is the Hilbert-Schmidt (Frobenius) norm and $\gamma \in (0,1)$ is the decay factor. The SGD update becomes:
\begin{equation}
    \mW_{t+1} = (1-\gamma)\mW_t - \eta \nabla_{\mW} \mathcal{L} \,.
\end{equation}
In the functional form $f(\vq) = \mW \psi(\vq)$, the scaling factor $(1-\gamma)$ distributes linearly over the sum of key-value pairs. Thus, weight decay is implemented by scaling down the value vectors of \textit{all} existing pairs in the database at each step:
\begin{equation}
    \mathcal{D}_{t+1} = \{(\vk, (1-\gamma)\vv) \mid (\vk, \vv) \in \mathcal{D}_t\} \cup \{(\vk_{\text{new}}, \vv_{\text{new}})\} \,.
\end{equation}
\textbf{Pruning:} This update implies that the magnitude of a value vector stored at time $t_0$ decays as $\|\vv(t)\| = (1-\gamma)^{t-t_0} \|\vv(t_0)\|$. This justifies a principled pruning strategy: once $\|\vv(t)\|$ drops below a threshold $\epsilon$, the pair's contribution becomes negligible, and it can be removed from memory, bounding the database size.

\subsubsection{Momentum}
We adopt the standard momentum formulation where a velocity matrix $\mV_t$ accumulates gradients. Let $\mG_t = -\eta \nabla_{\mW} \mathcal{L}_t$ be the update term at step $t$. The velocity update is $\mV_{t+1} = \mu \mV_t + \mG_t$, and the weight update is $\mW_{t+1} = \mW_t + \mV_{t+1}$.
Unrolling the recurrence, the weight matrix at time $T$ is a linear combination of all past update terms $\mG_\tau$ (created at time $\tau < T$):
\begin{equation}
    \mW_T = \sum_{\tau=1}^T c(T, \tau) \mG_\tau \,, \quad \text{with } c(T, \tau) = \sum_{k=0}^{T-\tau} \mu^k = \frac{1-\mu^{T-\tau+1}}{1-\mu} \,.
\end{equation}
Since each $\mG_\tau$ corresponds to the batch of key-value pairs created at time $\tau$, we do not need to store explicit velocity vectors. Instead, we store the creation timestamp $t_{\text{create}}$ with each pair. During the forward pass at time $T$, we dynamically compute the coefficient $c(T, t_{\text{create}})$:
\begin{equation}
    f_T(\vq) = \sum_{(\vk, \vv_{\text{init}}, t_{\text{create}}) \in \mathcal{D}} \left( \frac{1-\mu^{T-t_{\text{create}}+1}}{1-\mu} \right) \phi(\vk, \vq) \vv_{\text{init}} \,.
\end{equation}
This allows exact momentum computation with no additional memory overhead per pair. Combining both techniques results in decaying the stored values for regularization while scaling the retrieval output for momentum.

\section{Further Details and Remarks for the Softmax Case}
\label{app:softmax_from_main}

\subsection{Remark: the functional space of weighted softmax functions}
\label{app:remark_functional_space_of_weighted_softmax_functions}

In order to make the softmax function in Eq.~\ref{eq:softmax_with_weight_function} well-defined, we require that
\begin{align*}
    \int_{\mathbb{R}^{d_{\text{in}}} \times \mathbb{R}^{d_{\text{out}}}} w(\vk,\vv) e^{\langle \vk, \vq \rangle + \langle \vv,\vv'\rangle}& d\vk \, d\vv<\infty\,,\\
    &\forall \vq\in\mathbb{R}^{d_{\text{in}}}\,,\forall \vv'\in\mathbb{R}^{d_{\text{out}}}\,.
\end{align*}
In other words, if we interpret $w$ as the density of a measure $\mu_w$ on $\mathbb{R}^{d_{\text{in}}} \times \mathbb{R}^{d_{\text{out}}}$, then we require that the moment-generating function of $\mu_w$ be well-defined and finite everywhere.

\begin{remark}
    \label{rem:marginalizing_the_weight_function}
    In the case of a finite database of weighted key-value pairs, if we have several tuples having the same key, then we can combine them into one tuple. For example, we can combine $(w_1,\vk,\vv_1)$ and $(w_2,\vk,\vv_2)$ and replace them with $\left(w_1+w_2,\vk,\frac{w_1\vv_1+w_2\vv_2}{w_1+w_2}\right)$, and this will induce the same softmax function. Analogously, if we define the ``marginalized'' key-weight function $\tilde{w}:\mathbb{R}^{d_{\mathrm{in}}}\to\mathbb{R}^+$ as
    $$\tilde{w}(\vk)=\int_{\mathbb{R}^{d_{\mathrm{out}}}}w(\vk,\vv)d\vv\,,$$
    and the marginalized value function $\tilde{v}:\mathbb{R}^{d_{\mathrm{in}}}\to\mathbb{R}^{d_{\mathrm{out}}}$ which associates a unique value for every key as
    $$\tilde{v}(\vk):=\frac{1}{\tilde{w}(\vk)}\int_{\mathbb{R}^{d_{\mathrm{out}}}}w(\vk,\vv)\vv d\vv\,,$$
    then we can write $s_w$ in terms of $\tilde{w}$ and $\tilde{v}$ as follows:
    \begin{equation}
        \label{eq:marginalized_softmax_function}
        s_{w}(\vq)= s_{\tilde{w},\tilde{v}}(\vq):= \frac{\int_{\mathbb{R}^{d_{\text{in}}} } \tilde{w}(\vk) e^{\langle \vk, \vq \rangle}  \tilde{v}(\vk) \, d\vk}{\int_{\mathbb{R}^{d_{\text{in}}}} \tilde{w}(\vk) e^{\langle \vk, \vq \rangle} \, d\vk} \,.
    \end{equation}
\end{remark}

\subsection{Continuous softmax function induced by the Gaussian mixture density \& Discrete weighted softmax function}
\label{app:skipped_derivation}

Here we provide a full derivation demonstrating that the continuous softmax function induced by the Gaussian mixture density $w_{\mathcal{D}}$ is exactly equivalent to the discrete weighted softmax function $s_{\mathcal{D}}$ defined by the database $\mathcal{D}$.

\begin{align*}
    s_{w_{\mathcal{D}}}(\vq) &= \frac{\int_{\mathbb{R}^{d_{\text{in}}} \times \mathbb{R}^{d_{\text{out}}}} w_{\mathcal{D}}(\vk,\vv) e^{\langle \vk, \vq \rangle} \vv \, d\vk \, d\vv}{\int_{\mathbb{R}^{d_{\text{in}}} \times \mathbb{R}^{d_{\text{out}}}} w_{\mathcal{D}}(\vk,\vv) e^{\langle \vk, \vq \rangle} \, d\vk \, d\vv} \\
    &= \frac{\sum_{i=1}^n w_i \mathbb{E}_{(\vk,\vv) \sim \mathcal{N}((\vk_i,\vv_i), \Sigma)} [e^{\langle \vk, \vq \rangle} \vv]}{\sum_{i=1}^n w_i \mathbb{E}_{(\vk,\vv) \sim \mathcal{N}((\vk_i,\vv_i), \Sigma)} [e^{\langle \vk, \vq \rangle}]} \\
    &= \frac{\sum_{i=1}^n w_i \mathbb{E}_{\vk \sim \mathcal{N}(\vk_i, \sigma_k^2 \mI)} [e^{\langle \vk, \vq \rangle}] \mathbb{E}_{\vv \sim \mathcal{N}(\vv_i, \sigma_v^2 \mI)} [\vv]}{\sum_{i=1}^n w_i \mathbb{E}_{\vk \sim \mathcal{N}(\vk_i, \sigma_k^2 \mI)} [e^{\langle \vk, \vq \rangle}]} \\
    &= \frac{\sum_{i=1}^n w_i e^{\langle \vk_i, \vq \rangle + \frac{1}{2}\sigma_k^2 \|\vq\|^2} \vv_i}{\sum_{i=1}^n w_i e^{\langle \vk_i, \vq \rangle + \frac{1}{2}\sigma_k^2 \|\vq\|^2}} \\
    &= \frac{e^{\frac{1}{2}\sigma_k^2 \|\vq\|^2} \sum_{i=1}^n w_i e^{\langle \vk_i, \vq \rangle} \vv_i}{e^{\frac{1}{2}\sigma_k^2 \|\vq\|^2} \sum_{i=1}^n w_i e^{\langle \vk_i, \vq \rangle}} \\
    &= \frac{\sum_{i=1}^n w_i e^{\langle \vk_i, \vq \rangle} \vv_i}{\sum_{i=1}^n w_i e^{\langle \vk_i, \vq \rangle}} = s_{\mathcal{D}}(\vq) \,.
\end{align*}

\subsection{Further discussion on the standard (unweighted) softmax case}
\label{app:remark_standard_softmax}

In Sec.~\ref{sec:softmax} of the main text, we introduced the weighted softmax by noting that deriving a principled learning rule for the standard/unweighted softmax is non-trivial due to the discontinuous nature.
Here we illustrate this difficulty by presenting an example of a naive derivation of a learning rule for the standard softmax case and showing where it breaks down.

\textbf{Model description.} Consider a growing softmax layer that, at every training step $t$, maintains a key-value memory state consisting of two matrices $\mK_{t-1} = [\vk_1, \ldots, \vk_{t-1}] \in \mathbb{R}^{d_\text{in} \times (t-1)}$ and $\mV_{t-1} = [\vv_1, \ldots, \vv_{t-1}] \in \mathbb{R}^{d_\text{out} \times (t-1)}$ (both to be defined later), receives a layer input $\vx_t \in \mathbb{R}^{d_\text{in}}$ and produces a layer output $\vy_t \in \mathbb{R}^{d_\text{out}}$ through the softmax attention function:
\begin{align}
\vq_t & = \vx_t \\
\vy_t &= \mV_{t-1}\softmax(\mK_{t-1}^{\top}\vq_t) = \sum_{\tau=1}^{t-1} \alpha_{t-1, \tau} \vv_\tau \label{eq:softmax_old}
\end{align}
where $\alpha_{t-1, \tau} \in \mathbb{R}$ denotes the attention score for the key $\vk_\tau$ using the query $\vq_t$ and the system's key-value memory state from step $t-1$ (i.e., before updating the system at step $t$).
That is, $\vx_t$ is used as the query $\vq_t$ in the attention operation to retrieve from the key-value memory consisting of matrices $\mK_{t-1}$ and $\mV_{t-1}$.

Then, the system's memory state is updated as:
\begin{align}
\mK_t &= [\mK_{t-1}, \vk_t] \in \mathbb{R}^{d_\text{in} \times t} \\
\mV_t &= [\mV_{t-1}, \vv_t] \in \mathbb{R}^{d_\text{out} \times t}
\end{align}
where $[]$ denotes concatenation along the time dimension.

By following the linear case (Sec.~\ref{sec:background}), let us set the key to be the layer input, i.e., $\vk_t = \vx_t$. Now the question is how should we set $\vv_t \in \mathbb{R}^{d_\text{out}}$ so that this memory update corresponds to a principled learning algorithm for this system in a way that it optimizes a given objective function $\mathcal{L}$? 

\textbf{Direct Derivation.}
For that, we first set the new value vector $\vv_t \in \mathbb{R}^{d_\text{out}}$ as an arbitrary free parameter vector, and consider the virtually updated system with the newly added key-value pair $(\vx_t, \vv)$. After this update, the output of the system on the input $\vq_t = \vx_t$ becomes:
\begin{align}
\vy_t^{\text{new}} &= \mV_t\softmax(\mK_{t}^{\top}\vq_t) = \sum_{\tau=1}^{t-1} \alpha_{t, \tau}^{\text{new}} \vv_\tau + \alpha_{t, t}^{\text{new}} \vv \label{eq:softmax_new}
\end{align}
that is, compared to the old system of Eq.~\ref{eq:softmax_old}, one extra term (the last term in Eq.~\ref{eq:softmax_new}) is introduced in the sum as we now have one more value vector, with newly computed attention scores $\alpha_{t, \tau}^{\text{new}}$ normalized over the new set of keys, including the attention score $\alpha_{t, t}^{\text{new}}$ for the newly added key $\vx_t$.

Our goal is to improve the loss on input $\vq_t$ by adding the new value token $\vv_t$ to the system.
For that, we can apply a step of gradient descent on loss $\mathcal{L}$ w.r.t.~$\vv_t$ with a learning rate $\eta_t \in \mathbb{R}_{>0}$, which yields:
\begin{align}
\vv_t = \vv^{\text{old}} - \eta_t \dfrac{\partial\mathcal{L}_t(\vy_t^{\text{new}})}{\partial \vv} \label{eq:gd}
\end{align}
where two more variables, $\vv^{\text{old}}$ and the gradient term, are still to be determined; as we see below, this raises difficulties we cannot trivially overcome.

First, what should $\vv^{\text{old}}$ be? Essentially, $\vv^{\text{old}}$ is the value for newly added $\vv_t$ that preserves the old system (representing the initial system before applying a step of gradient descent).
For this, we can express the output of the new system $\vy_t^{\text{new}}$ (Eq.~\ref{eq:softmax_new}) as a function of that of the old one $\vy_t$ (Eq.~\ref{eq:softmax_old}) and search for the value of $\vv_t$ that makes them equal. The apparently tricky difference between them is the modified attention scores: $\alpha_{t, \tau}^{\text{new}}$ (Eq.~\ref{eq:softmax_new}) vs.~$\alpha_{t-1, \tau}$ (Eq.~\ref{eq:softmax_old}) for $1 \leq \tau \leq t-1$.
However, thankfully, we can notice a simple relation that connects the old attention scores before the update ($\alpha_{t-1, \tau}$) and the new ones ($\alpha_{t, \tau}^{\text{new}}$) after the update. By denoting $Z^{\text{old}}$ and $Z^{\text{new}}$ the denominators (normalizers) of the attention scores when the input is $\vq_t$, before and after the update, respectively, i.e.,
\begin{align}
Z^{\text{old}} = \sum_{\tau=1}^{t-1} e^{\vx_\tau^\top \vq_t} \,;\, Z^{\text{new}} = \sum_{\tau=1}^{t} e^{\vx_\tau^\top \vq_t} = Z^{\text{old}} + e^{\vq_t^\top \vq_t}\,,
\end{align}
we can express the new attention scores as a function of the old one:
\begin{align}
\alpha_{t, \tau}^{\text{new}} = \dfrac{e^{\vx_\tau^\top \vq_t}}{Z^{\text{new}}} = \dfrac{Z^{\text{old}}}{Z^{\text{new}}} \times \dfrac{e^{\vx_\tau^\top \vq_t}}{Z^{\text{old}}} =  \dfrac{Z^{\text{old}}}{Z^{\text{new}}} \alpha_{t-1, \tau}\,, \label{eq:new_score}
\end{align}
and additionally, the normalizer ratio can also be simplified, since $Z^{\text{new}} = Z^{\text{old}} + e^{\vq_t^\top \vq_t}$, as follows:
\begin{align}
\dfrac{Z^{\text{old}}}{Z^{\text{new}}} &= \dfrac{Z^{\text{old}} + e^{\vq_t^\top \vq_t} - e^{\vq_t^\top \vq_t}}{Z^{\text{new}}} = \dfrac{Z^{\text{new}} - e^{\vq_t^\top \vq_t}}{Z^{\text{new}}} = 1 - \dfrac{e^{\vq_t^\top \vq_t}}{Z^{\text{new}}} = 1 - \alpha_{t, t}^{\text{new}}\,.
\end{align}
By inserting this ratio to Eq.~\ref{eq:new_score}, it yields:
\begin{align}
\alpha_{t, \tau}^{\text{new}} = (1 - \alpha_{t, t}^{\text{new}}) \alpha_{t-1, \tau}\,,
\end{align}
that is, the new attention scores for the old keys shrink by factor $(1 - \alpha_{t, t}^{\text{new}})$. This allows us to express the new output $\vy_t^{\text{new}}$ as a function of the old one $\vy_t$:
\begin{align}
\vy_t^{\text{new}} &= \sum_{\tau=1}^{t-1} \alpha_{t, \tau}^{\text{new}} \vv_\tau + \alpha_{t, t}^{\text{new}} \vv \tag{\ref{eq:softmax_new}}\\
&=  \sum_{\tau=1}^{t-1} (1 - \alpha_{t, t}^{\text{new}}) \alpha_{t-1, \tau} \vv_\tau + \alpha_{t, t}^{\text{new}} \vv 
=  (1 - \alpha_{t, t}^{\text{new}}) \big(\sum_{\tau=1}^{t-1} \alpha_{t-1, \tau} \vv_\tau \big) + \alpha_{t, t}^{\text{new}} \vv \\ 
&= (1 - \alpha_{t, t}^{\text{new}}) \vy_t + \alpha_{t, t}^{\text{new}} \vv \\
&= \vy_t + \alpha_{t, t}^{\text{new}} (\vv  - \vy_t)\,. \label{eq:softmax_output_relation}
\end{align}
Here we can at least find $\vv_t = \vv^{\text{old}}$ such that the updated system reproduces the old system's output on input $\vq_t$, that is, $\vy_t^{\text{new}} = \vy_t$, which, from Eq.~\ref{eq:softmax_output_relation} holds when $\vv_t = \vv^{\text{old}} = \vy_t$, canceling the last term.
However, this only preserves the mapping for the current query $\vq_t$; as the new key $\vk_t$ alters the softmax denominator globally, the output for any other query will be modified, failing to achieve the global function preservation.

Second, while the computation of the gradient term seems straightforward:
\begin{align}
\dfrac{\partial\mathcal{L}_t(\vy_t^{\text{new}})}{\partial \vv} = \dfrac{\partial \mathcal{L}_t(\vy_t^{\text{new}})}{\partial \vy_t^{\text{new}}} \times \dfrac{\partial \vy_t^{\text{new}}}{\partial \vv} = (\nabla_{\vy_t^{\text{new}}} \mathcal{L}_t) \alpha_{t, t}^{\text{new}} \,,
\end{align}
by differentiating Eq.~\ref{eq:softmax_new} w.r.t.~$\vv$, this requires access to the \textit{post-update} gradient of the loss w.r.t.~the \textit{new} output $(\nabla_{\vy_t^{\text{new}}} \mathcal{L}_t)$.
While this creates a circular dependency in the general case,
by setting $\vv_t = \vy_t$ as prescribed by the derivation above, $\nabla_{\vy_t^{\text{new}}}\mathcal{L}_t$ can be essentially replaced by $\nabla_{\vy_t}\mathcal{L}_t$.

Putting these together, we obtain the following update rule:
\begin{align}
\vv_t = \vv^{\text{old}} - \eta_t \dfrac{\partial\mathcal{L}_t(\vy_t^{\text{new}})}{\partial \vv} 
 = \vy_t - \eta_t \alpha_{t, t}^{\text{new}} \nabla_{\vy_t^{\text{new}}}\mathcal{L}_t\,,
\end{align}
where $\alpha_{t, t}^{\text{new}} = \dfrac{e^{\vx_t^\top \vx_t}}{Z^{\text{new}}}\in \mathbb{R}$ can be computed simply as we have $Z^{\text{new}} = Z^{\text{old}} + e^{\vx_t^\top \vx_t}$ for the normalizer, and $Z^{\text{old}}$ is computed to compute $\vy_t$.

However, this is not a satisfactory derivation as the choice of $\vv^{\text{old}}$ is limited to the local function preservation on one data point.
This derivation illustrates that a naive approach that consists in parameterizing the value vector as a free parameter and applying gradient descent on it does not work out of the box.

\section{Functional Gradient Descent in the Weighted Softmax Space}
\label{app:functional_gradient}

In this appendix, we rigorously derive the functional gradient of the loss with respect to the weight function $w \in \mathcal{H}_{\sigma_k, \sigma_v}$ and show that the gradient descent update corresponds to adding a new key-value tuple to the database.

\subsection{Setup and Definitions}

Let $\mathcal{H} = \mathcal{H}_{\sigma_k, \sigma_v}$ be the Hilbert space of weight functions $w: \mathbb{R}^{d_{\text{in}}} \times \mathbb{R}^{d_{\text{out}}} \to \mathbb{R}$ equipped with the inner product $\langle f, g \rangle_\mathcal{H} = \int f(\vk, \vv) g(\vk, \vv) d\mu(\vk, \vv)$ and norm $\|w\|_\mathcal{H}$, where $\mu = \mathcal{N}(\mathbf{0}, \Sigma)$ is the centered Gaussian measure with block covariance $\Sigma = \text{diag}(\Sigma_k, \Sigma_v) = \text{diag}(\sigma_k^2 \mI, \sigma_v^2 \mI)$. Recall the definition of the softmax functional
$$s_w(\vq)=\frac{\int_{\mathbb{R}^{d_{\text{in}}} \times \mathbb{R}^{d_{\text{out}}}} w(\vk, \vv) e^{\langle \vk, \vq \rangle} \vv \, d\mu(\vk, \vv)}{\int_{\mathbb{R}^{d_{\text{in}}} \times \mathbb{R}^{d_{\text{out}}}} w(\vk, \vv) e^{\langle \vk, \vq \rangle} \, d\mu(\vk, \vv)}=\frac{N(w)}{D(w)}\,,$$
where we define the denominator functional $D: \mathcal{H} \to \mathbb{R}$ and the numerator operator $N: \mathcal{H} \to \mathbb{R}^{d_{\text{out}}}$ as follows:
\begin{align}
    D(w) &= \int_{\mathbb{R}^{d_{\text{in}}} \times \mathbb{R}^{d_{\text{out}}}} w(\vk, \vv) e^{\langle \vk, \vq \rangle} \, d\mu(\vk, \vv)\,, \\
    N(w) &= \int_{\mathbb{R}^{d_{\text{in}}} \times \mathbb{R}^{d_{\text{out}}}} w(\vk, \vv) e^{\langle \vk, \vq \rangle} \vv \, d\mu(\vk, \vv)\,.
\end{align}

Let $N_j(w)$ denote the $j$-th component of the vector $N(w)$. It is a scalar linear functional:
$$ N_j(w) = \int w(\vk, \vv) e^{\langle \vk, \vq \rangle} v_j \, d\mu(\vk, \vv)\,.$$
By defining the kernels
$$ K_D(\vk, \vv) = e^{\langle \vk, \vq \rangle}, \quad K_{N,j}(\vk, \vv) = e^{\langle \vk, \vq \rangle} v_j\,, $$
we can rewrite $D$ and $N_j$ as
\begin{align}
    D(w) &= \int_{\mathbb{R}^{d_{\text{in}}} \times \mathbb{R}^{d_{\text{out}}}} w(\vk, \vv) K_D(\vk,\vv)\, d\mu(\vk, \vv)\,, \\
    N_j(w) &= \int_{\mathbb{R}^{d_{\text{in}}} \times \mathbb{R}^{d_{\text{out}}}} w(\vk, \vv) K_{N,j}(\vk,\vv)\, d\mu(\vk, \vv)\,.
\end{align}

By the Riesz Representation Theorem, if the functions representing these kernels belong to $\mathcal{H}$, the functionals are bounded.

We verify membership in $\mathcal{H}$ by checking the condition $\mathbb{E}_{\mu} [ f(\vk, \vv)^2 e^{\langle \vk, \mathbf{u} \rangle + \langle \vv, \mathbf{z} \rangle} ] < \infty$ for all $\mathbf{u}, \mathbf{z}$.

For $K_D$, we have
$$ \|K_D\|_\mathcal{H}^2 = \mathbb{E}_{(\vk, \vv) \sim \mu} \left[ (e^{\langle \vk, \vq \rangle})^2 e^{\langle \vk, \mathbf{u} \rangle + \langle \vv, \mathbf{z} \rangle} \right] = \mathbb{E}_{\vk} [ e^{\langle \vk, 2\vq + \mathbf{u} \rangle} ] \mathbb{E}_{\vv} [ e^{\langle \vv, \mathbf{z} \rangle} ] \,.$$
These are moment generating functions of Gaussian variables, which are finite everywhere.

For $K_{N,j}$, we have
$$ \|K_{N,j}\|_\mathcal{H}^2 = \mathbb{E}_{(\vk, \vv) \sim \mu} \left[ (e^{\langle \vk, \vq \rangle} v_j)^2 e^{\langle \vk, \mathbf{u} \rangle + \langle \vv, \mathbf{z} \rangle} \right] = \mathbb{E}_{\vk} [ e^{\langle \vk, 2\vq + \mathbf{u} \rangle} ] \mathbb{E}_{\vv} [ v_j^2 e^{\langle \vv, \mathbf{z} \rangle} ]\,. $$
The second term involves the second moment of a Gaussian under an exponential tilt, which is also finite.
Thus, $D(w)$ and $N(w)$ are bounded linear operators. Specifically, there exist constants $C_N, C_D < \infty$ such that for any $\delta \in \mathcal{H}$:
\begin{equation}
    \|N(\delta)\|_2 \le C_N \|\delta\|_\mathcal{H}, \quad |D(\delta)| \le C_D \|\delta\|_\mathcal{H}
\end{equation}

\subsection{Fréchet Differentiability of $s_w$}

We claim that the Fréchet derivative of $s_w$ at $w$, denoted $ds_w$, is the linear operator acting on a perturbation $\delta \in \mathcal{H}$ as:
\begin{equation}
    ds_w(\delta) = \frac{N(\delta)D(w) - N(w)D(\delta)}{D(w)^2}
\end{equation}
To prove this, we must show that the remainder term $R(\delta) = s_{w+\delta}(\vq) - s_w(\vq) - ds_w(\delta)$ satisfies $\|R(\delta)\| = o(\|\delta\|_\mathcal{H})$ as $\|\delta\|_\mathcal{H} \to 0$.

\begin{proof}
Let us expand the difference $s_{w+\delta} - s_w$:
\begin{align}
    s_{w+\delta} - s_w &= \frac{N(w) + N(\delta)}{D(w) + D(\delta)} - \frac{N(w)}{D(w)} \\
    &= \frac{D(w)(N(w) + N(\delta)) - N(w)(D(w) + D(\delta))}{D(w)(D(w) + D(\delta))} \\
    &= \frac{D(w)N(\delta) - N(w)D(\delta)}{D(w)^2 + D(w)D(\delta)}
\end{align}
Subtracting the candidate derivative $ds_w(\delta)$:
\begin{align}
    R(\delta) &= \frac{D(w)N(\delta) - N(w)D(\delta)}{D(w)(D(w) + D(\delta))} - \frac{D(w)N(\delta) - N(w)D(\delta)}{D(w)^2} \\
    &= [D(w)N(\delta) - N(w)D(\delta)] \left( \frac{1}{D(w)(D(w) + D(\delta))} - \frac{1}{D(w)^2} \right) \\
    &= [D(w)N(\delta) - N(w)D(\delta)] \left( \frac{D(w) - (D(w) + D(\delta))}{D(w)^2(D(w) + D(\delta))} \right) \\
    &= -[D(w)N(\delta) - N(w)D(\delta)] \frac{D(\delta)}{D(w)^2(D(w) + D(\delta))}
\end{align}
Now we bound the norm of $R(\delta)$. Let $w$ be fixed such that $D(w) \neq 0$. For sufficiently small $\delta$ (specifically $\|\delta\|_\mathcal{H} < |D(w)|/(2C_D)$), we have $$|D(w) + D(\delta)|\ge |D(w)| - |D(\delta)|\ge |D(w)| - C_D\|\delta\|_\mathcal{H} \ge |D(w)|/2\,.$$ Furthermore,
\begin{align}
    \|D(w)N(\delta) - N(w)D(\delta)\| &\le |D(w)|\|N(\delta)\| + \|N(w)\||D(\delta)| \\
    &\le (|D(w)|C_N + \|N(w)\|C_D) \|\delta\|_\mathcal{H} \\
    &= C_1 \|\delta\|_\mathcal{H}\,.
\end{align}
Substituting these back into the expression for $R(\delta)$:
\begin{align}
    \|R(\delta)\| &\le \frac{C_1 \|\delta\|_\mathcal{H} \cdot |D(\delta)|}{|D(w)|^2 |D(w) + D(\delta)|} \\
    &\le \frac{C_1 \|\delta\|_\mathcal{H} \cdot C_D \|\delta\|_\mathcal{H}}{|D(w)|^2 (|D(w)|/2)} \\
    &= \frac{2 C_1 C_D}{|D(w)|^3} \|\delta\|_\mathcal{H}^2 \\
    &= O(\|\delta\|_\mathcal{H}^2)\,.
\end{align}
Since $\|R(\delta)\| \le C \|\delta\|_\mathcal{H}^2$, it implies $\lim_{\|\delta\| \to 0} \frac{\|R(\delta)\|}{\|\delta\|} = 0$. Thus, $s_w$ is Fréchet differentiable.
\end{proof}

\subsection{Derivation of the Functional Gradient}

Let $\mathcal{L}(\hat{\vy}, \vy)$ be a loss function that is differentiable with respect to the prediction $\hat{\vy}$. By the chain rule for Fréchet derivatives, the composition $J(w) = \mathcal{L}(s_w(\vq), \vy)$ is differentiable with respect to $w$, and its differential is:
\begin{equation}
    dJ(\delta) = \langle \nabla_{\hat{\vy}}\mathcal{L}, ds_w(\delta) \rangle
\end{equation}
Substituting $ds_w(\delta)$:
\begin{align}
    dJ(\delta) &= \left\langle \nabla_{\hat{\vy}}\mathcal{L}, \frac{N(\delta)}{D(w)} - s_w(\vq)\frac{D(\delta)}{D(w)} \right\rangle \\
    &= \frac{1}{D(w)} \langle \nabla_{\hat{\vy}}\mathcal{L}, N(\delta) \rangle - \frac{\langle \nabla_{\hat{\vy}}\mathcal{L}, s_w(\vq) \rangle}{D(w)} D(\delta)\,.
\end{align}
Using the integral definitions of $N(\delta)$ and $D(\delta)$ and the linearity of the inner product, we can rewrite this as an inner product in $\mathcal{H}$:
\begin{align}
    dJ(\delta) &= \int \delta(\vk, \vv) \left[ \frac{1}{D(w)} \langle \nabla_{\hat{\vy}}\mathcal{L}, \vv \rangle e^{\langle \vk, \vq \rangle} - \frac{\langle \nabla_{\hat{\vy}}\mathcal{L}, s_w(\vq) \rangle}{D(w)} e^{\langle \vk, \vq \rangle} \right] d\mu \\
    &= \int \delta(\vk, \vv) \frac{e^{\langle \vk, \vq \rangle}}{D(w)} \langle \nabla_{\hat{\vy}}\mathcal{L}, \vv - s_w(\vq) \rangle d\mu \\
    &= \langle \delta, \nabla_w J \rangle_\mathcal{H}\,.
\end{align}
By the Riesz representation theorem, the function $\nabla_w J$ inside the inner product is the unique functional gradient:
\begin{equation}
    \nabla_w J(\vk, \vv) = \frac{e^{\langle \vk, \vq \rangle}}{D(w)} \langle \nabla_{\hat{\vy}}\mathcal{L}, \vv - s_w(\vq) \rangle\,.
\end{equation}

\subsection{Constructive Update Analysis}
\label{app:constructive_update_analysis}

We now analyze the impact of the functional update $w_{t+1} \leftarrow w_t - \eta \nabla_w J$ on the softmax functional $s_{w_{t+1}}(\vq)$ and demonstrate that it is equivalent to adding a specific key-value tuple to the database. The impact of this update is determined by the changes to the numerator and denominator functionals.
For the denominator, the update adds a term:
\begin{align*}
    \Delta D(\vq) &= -\eta \int \nabla_w J(\vk, \vv) e^{\langle \vk, \vq \rangle} d\mu(\vk, \vv) \\
    &= -\frac{\eta}{D(w_t)} \int e^{\langle \vk, \vq_t + \vq \rangle} \langle \vg_t, \vv - \vs_t \rangle d\mu(\vk, \vv)\,,
\end{align*}
where $\vg_t = \nabla_{\hat{\vy}} \mathcal{L}$ and $\vs_t = s_{w_t}(\vq_t)$.
Since the measure $\mu$ is a centered Gaussian, $\mathbb{E}_\mu[\vv] = \mathbf{0}$. Consequently, the term involving $\vv$ integrates to zero, leaving only the scalar part:
$$ \Delta D(\vq) = \frac{\eta \langle \vg_t, \vs_t \rangle}{D(w_t)} \mathbb{E}_{\mu_k}[e^{\langle \vk, \vq_t + \vq \rangle}] \,. $$
The expectation is the moment generating function of $\mathcal{N}(\mathbf{0}, \Sigma_k)$ evaluated at $\vq_t + \vq$, which equals $e^{\frac{1}{2}\|\Sigma_k^{1/2}(\vq_t+\vq)\|^2}$. Factoring this, we get a term proportional to $e^{\langle \Sigma_k \vq_t, \vq \rangle}$. This corresponds precisely to the contribution of a new key $\vk_{\text{new}} = \Sigma_k \vq_t$ (or $\sigma_k^2 \vq_t$ in the isotropic case).
Similarly, for the numerator, the update adds a term involving $\vv$:
\begin{align*}
    \Delta N(\vq) &= -\frac{\eta}{D(w_t)} \int e^{\langle \vk, \vq_t + \vq \rangle} \langle \vg_t, \vv - \vs_t \rangle \vv d\mu(\vk, \vv) \,.
\end{align*}
The term $\langle \vg_t, -\vs_t \rangle \vv$ integrates to zero (zero mean). The remaining term involves the second moment $\mathbb{E}_{\mu_v}[\vv \vv^\top] = \Sigma_v$. Thus:
$$ \Delta N(\vq) = -\frac{\eta}{D(w_t)} \Sigma_v \vg_t \mathbb{E}_{\mu_k}[e^{\langle \vk, \vq_t + \vq \rangle}] \,. $$
Comparing $\Delta N$ and $\Delta D$, we see they share the same key kernel $e^{\langle \vk_{\text{new}}, \vq \rangle}$. The equivalent added tuple $(w_{\text{new}}, \vk_{\text{new}}, \vv_{\text{new}})$ is thus:
$$ \vk_{\text{new}} = \Sigma_k \vq_t, \quad w_{\text{new}} \propto \eta \langle \vs_t, \vg_t \rangle, \quad \vv_{\text{new}} \propto -\frac{\Sigma_v \vg_t}{\langle \vs_t, \vg_t \rangle} \,. $$
This confirms that functional gradient descent in $\mathcal{H}_{\sigma_k, \sigma_v}$ is constructively equivalent to growing the database with specific, tractable update rules.

In order to avoid dividing by $\langle \vs_t, \vg_t\rangle$ which could be small and might introduce numerical instabilities, we can adopt a representation where we store the weighted value vector $\tilde{\vv}=w\vv$ instead of the value vector $\vv$, as in Eq.~\ref{eq:weighted_softmax_with_weighted_value}. In this case, the added tuple becomes
$$ \vk_{\text{new}} = \Sigma_k \vq_t, \quad w_{\text{new}} \propto \eta \langle \vs_t, \vg_t \rangle, \quad \tilde{\vv}_{\text{new}} \propto -\eta\Sigma_v \vg_t \,. $$

If $\Sigma_v=\sigma_v^2 \mI$ and $\beta=\eta\sigma_v^2$, the weighted value vector becomes
$$\tilde{\vv}_{\text{new}} \propto -\beta \vg_t\,.$$

\subsection{Negative Weights and Numerical Stability}
\label{app:negative_weights}
A critical inspection of the update rule for $w_{\text{new}}$ reveals that its sign is determined by the inner product $\delta_t = \langle \vs_t, \vg_t \rangle$. Since the learning rate $\eta$ and normalization $D(w_t)$ are positive (assuming initialization with positive mass), $w_{\text{new}}$ can be negative if the gradient of the loss forms an obtuse angle with the current prediction vector.

This introduces the concept of ``negative attention'' or signed measures into the softmax mechanism. Negative weights act to reduce the denominator (partition function), effectively increasing the magnitude of the output vector $s(\vq)$ in the direction of the value vectors. This occurs when the current prediction $\vs_t$ is negatively correlated with the gradient $\vg_t$, indicating that gradient descent is attempting to scale up the prediction to minimize the loss.

While considering negative weights might increase the expressive power of our model, they might introduce numerical instabilities as the denominator might become close to 0 if the negative and positive terms cancel.

To ensure robustness, we propose enforcing non-negativity by clipping the new weight:
\begin{equation}
    w_{\text{new}} \leftarrow \max(0, w_{\text{new}}) \,.
\end{equation}
This clipping strategy can be theoretically justified as performing a descent step in a larger, decoupled function space $\mathcal{H}^2_{\sigma_k, \sigma_v}$, where the numerator and denominator weight functions are optimized independently. Refer to \cref{app:decoupled} for further details.

In \cref{app:gauge-invariance}, we propose another method for solving the negative weights issue.

We leave the investigation of deriving an update rule that benefits from the potential expressive power of the negative weights while ensuring numerical stability for future research.

\subsection{Functional Gradient Descent in the Decoupled Weight Space}
\label{app:decoupled}

In this appendix, we consider the functional gradient descent on the decoupled softmax functional defined in Eq.~\ref{app:eq_softmax_with_decoupled_weight_functions} below. We define the product Hilbert space $\mathcal{H}^2 = \mathcal{H}_{\sigma_k, \sigma_v} \oplus \mathcal{H}_{\sigma_k, \sigma_v}$ containing pairs of weight functions $(w_N, w_D)$.
The decoupled softmax functional is:
\begin{equation}
    \label{app:eq_softmax_with_decoupled_weight_functions}
    s_{w_N, w_D}(\vq) = \frac{N(w_N)}{D(w_D)} = \frac{\int w_N(\vk,\vv) e^{\langle \vk, \vq \rangle} \vv \, d\mu}{\int w_D(\vk,\vv) e^{\langle \vk, \vq \rangle} \, d\mu}\,.
\end{equation}

We compute the Fréchet derivatives with respect to $w_N$ and $w_D$ independently.
For $w_N$, the functional is linear in the numerator:
$$ \nabla_{w_N} s = \frac{1}{D(w_D)} \nabla_{w_N} N \,.$$
The gradient of the loss $J(w_N, w_D)$ with respect to $w_N$ is:
\begin{align*}
    \nabla_{w_N} J &= \frac{1}{D(w_D)} e^{\langle \vk, \vq_t \rangle} \langle \nabla_{\hat{\vy}} \mathcal{L}, \vv \rangle\,.
\end{align*}
For $w_D$, the functional involves the inverse of the denominator:
$$ \nabla_{w_D} s = -\frac{N(w_N)}{D(w_D)^2} \nabla_{w_D} D = -\frac{s_{w_N, w_D}}{D(w_D)} \nabla_{w_D} D \,.$$
The gradient of the loss with respect to $w_D$ is:
\begin{align*}
    \nabla_{w_D} J &= -\frac{1}{D(w_D)} e^{\langle \vk, \vq_t \rangle} \langle \nabla_{\hat{\vy}} \mathcal{L}, s_{w_N, w_D}(\vq_t) \rangle\,.
\end{align*}

We now map these updates to the growing database formulation.
Updating $w_N$ adds a term $\Delta w_N \propto -\frac{1}{D} \langle \vg_t, \vv \rangle e^{\langle \vk, \vq_t \rangle}$.
In the ``weighted value vector'' representation (where we store $\tilde{\vv}$ instead of $\vv$), adding a tuple corresponds to adding a basis function to $w_N$. Specifically, if we add $(\tilde{\vv}_{\text{new}}, \vk_{\text{new}})$, we modify the numerator density by a term proportional to $e^{\langle \vk, \Sigma_k^{-1} \vk_{\text{new}} \rangle} \langle \vv, \Sigma_v^{-1} \tilde{\vv}_{\text{new}} \rangle$.
Matching terms, we find $\vk_{\text{new}} = \Sigma_k \vq_t$ and $\tilde{\vv}_{\text{new}} \propto -\eta \Sigma_v \vg_t$. This matches the numerator update derived in the coupled case.

Updating $w_D$ adds a term $\Delta w_D \propto \frac{1}{D} \langle \vg_t, \vs_t \rangle e^{\langle \vk, \vq_t \rangle}$.
This corresponds to adding a basis function to the denominator weight density $w_D$. In the database, this means adding a tuple $(w_{\text{new}}, \vk_{\text{new}})$.
Matching terms, we find $\vk_{\text{new}} = \Sigma_k \vq_t$ and $w_{\text{new}} \propto \eta \langle \vs_t, \vg_t \rangle$. This matches the denominator update derived in the coupled case.

Therefore, the coupled update rule derived in Appendix A is a special case of the decoupled update where we enforce $w_N = w_D = w$ initially but allow them to evolve. Crucially, since the optimization of $w_D$ is independent of $w_N$ in this formulation, applying a constraint $w_D \ge 0$ (via clipping $w_{\text{new}} = \max(0, w_{\text{new}})$) corresponds to a valid Descent step in $\mathcal{H}_{\sigma_k, \sigma_v} \oplus \mathcal{H}_{\sigma_k, \sigma_v}$ which would be positively correlated with the gradient direction. This ensures numerical stability while maintaining the numerator's ability to reduce error.

\subsection{Gauge-Invariance: Another Method for Solving the Negative Weights Issue}
\label{app:gauge-invariance}
As established in \cref{app:negative_weights}, the constructive functional gradient update for the scalar weight of a newly added weighted key-value tuple is given by:
\begin{equation}
    w_{\text{new}} \propto \eta \langle \vs_t, \vg_t \rangle \,,
\end{equation}
where $\vs_t = s_{w_t}(\vq_t)$ is the current prediction and $\vg_t = \nabla_{\hat{\vy}} \mathcal{L}$ is the error signal. When the gradient forms an obtuse angle with the current prediction vector (i.e., $\langle \vs_t, \vg_t \rangle < 0$), this rule mandates appending a \emph{negative weight} ($w_{\text{new}} < 0$). While the non-negativity constraint can be enforced heuristically via clipping $w_{\text{new}} \leftarrow \max(0, w_{\text{new}})$ (theoretically grounded in the decoupled weight space of \cref{app:decoupled}), clipping discards the denominator update entirely whenever $\langle \vs_t, \vg_t \rangle < 0$.

In this appendix, we develop an alternative, geometrically principled resolution to the negative weight obstruction: an \textbf{augmented representation with gauge invariance}. By introducing a global bias degree of freedom, we show that the underlying function space possesses an exact translation symmetry in value space. By dynamically selecting an appropriate coordinate frame (a ``virtual gauge'') at each training step, we guarantee that the required weight update is \emph{strictly positive} ($w_{\text{new}} > 0$) without clipping, while preserving the locality of the database update.

\subsubsection{Augmented Functional and Gauge Symmetry}

We introduce an auxiliary global bias degree of freedom $\vb \in \mathbb{R}^{d_{\text{out}}}$ into both the continuous and discrete weighted softmax functionals.

\begin{definition}[Augmented Softmax Functionals]
Let $\vb \in \mathbb{R}^{d_{\text{out}}}$ be a bias vector.
\begin{enumerate}
    \item For a functional weight density $w \in \mathcal{H}_{\sigma_k, \sigma_v}$, the continuous augmented softmax output $s_{w, \vb}: \mathbb{R}^{d_{\text{in}}} \to \mathbb{R}^{d_{\text{out}}}$ is defined as:
    \begin{equation}
        s_{w, \vb}(\vq) := s_w(\vq) - \vb = \frac{\int_{\mathbb{R}^{d_{\text{in}}} \times \mathbb{R}^{d_{\text{out}}}} w(\vk, \vv) e^{\langle \vk, \vq \rangle} \vv \, d\mu(\vk, \vv)}{\int_{\mathbb{R}^{d_{\text{in}}} \times \mathbb{R}^{d_{\text{out}}}} w(\vk, \vv) e^{\langle \vk, \vq \rangle} \, d\mu(\vk, \vv)} - \vb \,.
        \label{eq:augmented_continuous_functional}
    \end{equation}
    \item For a discrete database $\mathcal{D} = \{(w_i, \vk_i, \tilde{\vv}_i)\}_{i=1}^N$ of weighted key-value tuples with $\tilde{\vv}_i = w_i \vv_i$, the discrete augmented retrieval functional is defined as:
    \begin{equation}
        s_{\mathcal{D}, \vb}(\vq) := s_{\mathcal{D}}(\vq) - \vb = \frac{\sum_{i=1}^N e^{\langle \vk_i, \vq \rangle} \tilde{\vv}_i}{\sum_{i=1}^N w_i e^{\langle \vk_i, \vq \rangle}} - \vb \,.
        \label{eq:augmented_discrete_functional}
    \end{equation}
\end{enumerate}
\end{definition}

The standard model corresponds to a bias that is set to zero: $s_{\mathcal{D}}(\vq) = s_{\mathcal{D}, \mathbf{0}}(\vq)$.

We now define a transformation of the representation under a spatial translation in value space.

\begin{definition}[Shifted Density and Database]
Let $\boldsymbol{\xi} \in \mathbb{R}^{d_{\text{out}}}$ be an arbitrary shift vector.
\begin{itemize}
    \item We define the shifted weight density $w_{\boldsymbol{\xi}}$ by translating the value coordinate:
    \begin{equation}
        w_{\boldsymbol{\xi}}(\vk, \vv) := w(\vk, \vv - \boldsymbol{\xi})\cdot\exp\left(\frac{2\langle \boldsymbol{\xi},\vv\rangle-\|\boldsymbol{\xi}\|^2}{2\sigma_v^2}\right) \,.
    \end{equation}
    \item In the discrete representation, we shift the value vectors as $\vv_i \mapsto \vv_i + \boldsymbol{\xi}$ which implies that the stored weighted values shift as $\tilde{\vv}_i \mapsto w_i(\vv_i + \boldsymbol{\xi}) = \tilde{\vv}_i + w_i \boldsymbol{\xi}$. Hence, we define the shifted database $\mathcal{D}_{\boldsymbol{\xi}}$ as:
    \begin{equation}
        \mathcal{D}_{\boldsymbol{\xi}} := \left\{ (w_i, \vk_i, \tilde{\vv}_i + w_i \boldsymbol{\xi}) \right\}_{i=1}^N \,.
    \end{equation}
\end{itemize}
\end{definition}

\begin{proposition}[Gauge Invariance]
\label{prop:gauge_invariance}
For any density $w \in \mathcal{H}_{\sigma_k, \sigma_v}$, discrete database $\mathcal{D}$, bias vector $\vb \in \mathbb{R}^{d_{\text{out}}}$, and shift vector $\boldsymbol{\xi} \in \mathbb{R}^{d_{\text{out}}}$, the augmented functional is invariant under simultaneous shifts of the values and the bias:
\begin{equation}
    s_{w_{\boldsymbol{\xi}}, \vb + \boldsymbol{\xi}}(\vq) = s_{w, \vb}(\vq) \quad \text{and} \quad s_{\mathcal{D}_{\boldsymbol{\xi}}, \vb + \boldsymbol{\xi}}(\vq) = s_{\mathcal{D}, \vb}(\vq) \,.
\end{equation}
In particular, any augmented state $(\mathcal{D}, \vb)$ is gauge-equivalent to a canonical zero-bias state:
\begin{equation}
    s_{\mathcal{D}, \vb}(\vq) = s_{\mathcal{D}_{-\vb}, \mathbf{0}}(\vq) \,.
\end{equation}
\end{proposition}

\begin{proof}
For the continuous case, note that
\begin{align*}
    s_{w_{\boldsymbol{\xi}}}(\vq) &= \frac{\int_{\mathbb{R}^{d_{\text{in}}} \times \mathbb{R}^{d_{\text{out}}}} w_{\boldsymbol{\xi}}(\vk,\vv) e^{\langle \vk, \vq \rangle} \vv \, d\mu(\vk, \vv)}{\int_{\mathbb{R}^{d_{\text{in}}} \times \mathbb{R}^{d_{\text{out}}}} w_{\boldsymbol{\xi}}(\vk,\vv) e^{\langle \vk, \vq \rangle} \, d\mu(\vk, \vv)}\\
    &= \frac{\int_{\mathbb{R}^{d_{\text{in}}} \times \mathbb{R}^{d_{\text{out}}}} w(\vk, \vv - \boldsymbol{\xi})\cdot\exp\left(\frac{2\langle \boldsymbol{\xi},\vv\rangle-\|\boldsymbol{\xi}\|^2}{2\sigma_v^2}\right) e^{\langle \vk, \vq \rangle} \vv \, \exp\left(-\frac{\|\vk\|^2}{2\sigma_k^2}-\frac{\|\vv\|^2}{2\sigma_v^2}\right)d\vk d\vv}{\int_{\mathbb{R}^{d_{\text{in}}} \times \mathbb{R}^{d_{\text{out}}}} w(\vk, \vv - \boldsymbol{\xi})\cdot\exp\left(\frac{2\langle \boldsymbol{\xi},\vv\rangle-\|\boldsymbol{\xi}\|^2}{2\sigma_v^2}\right) e^{\langle \vk, \vq \rangle} \, \exp\left(-\frac{\|\vk\|^2}{2\sigma_k^2}-\frac{\|\vv\|^2}{2\sigma_v^2}\right)d\vk d\vv}\\
    &= \frac{\int_{\mathbb{R}^{d_{\text{in}}} \times \mathbb{R}^{d_{\text{out}}}} w(\vk, \vv - \boldsymbol{\xi})e^{\langle \vk, \vq \rangle} \vv \, \exp\left(-\frac{\|\vk\|^2}{2\sigma_k^2}-\frac{\|\vv-\boldsymbol{\xi}\|^2}{2\sigma_v^2}\right)d\vk d\vv}{\int_{\mathbb{R}^{d_{\text{in}}} \times \mathbb{R}^{d_{\text{out}}}} w(\vk, \vv - \boldsymbol{\xi}) e^{\langle \vk, \vq \rangle} \, \exp\left(-\frac{\|\vk\|^2}{2\sigma_k^2}-\frac{\|\vv-\boldsymbol{\xi}\|^2}{2\sigma_v^2}\right)d\vk d\vv}\\
    &= \frac{\int_{\mathbb{R}^{d_{\text{in}}} \times \mathbb{R}^{d_{\text{out}}}} w(\vk, \vv)e^{\langle \vk, \vq \rangle} (\vv + \boldsymbol{\xi}) \, \exp\left(-\frac{\|\vk\|^2}{2\sigma_k^2}-\frac{\|\vv\|^2}{2\sigma_v^2}\right)d\vk d\vv}{\int_{\mathbb{R}^{d_{\text{in}}} \times \mathbb{R}^{d_{\text{out}}}} w(\vk, \vv) e^{\langle \vk, \vq \rangle} \, \exp\left(-\frac{\|\vk\|^2}{2\sigma_k^2}-\frac{\|\vv\|^2}{2\sigma_v^2}\right)d\vk d\vv}\\
    &= \frac{\int_{\mathbb{R}^{d_{\text{in}}} \times \mathbb{R}^{d_{\text{out}}}} w(\vk,\vv) e^{\langle \vk, \vq \rangle} (\vv + \boldsymbol{\xi}) \, d\mu(\vk, \vv)}{\int_{\mathbb{R}^{d_{\text{in}}} \times \mathbb{R}^{d_{\text{out}}}} w(\vk,\vv) e^{\langle \vk, \vq \rangle} \, d\mu(\vk, \vv)}\\
    &= s_w(\vq) +  \boldsymbol{\xi}\,.
\end{align*}
Evaluating the augmented functional with shifted bias $\vb + \boldsymbol{\xi}$, we get
\begin{equation}
    s_{w_{\boldsymbol{\xi}}, \vb + \boldsymbol{\xi}}(\vq) = s_{w_{\boldsymbol{\xi}}}(\vq) - (\vb + \boldsymbol{\xi}) = (s_{w}(\vq) + \boldsymbol{\xi}) - (\vb + \boldsymbol{\xi}) = s_{w}(\vq) - \vb = s_{w, \vb}(\vq) \,.
\end{equation}

Similarly, for the discrete case, we have:
\begin{align}
    s_{\mathcal{D}_{\boldsymbol{\xi}}}(\vq) &= \frac{\sum_{i=1}^N e^{\langle \vk_i, \vq \rangle} (\tilde{\vv}_i + w_i \boldsymbol{\xi})}{\sum_{i=1}^N w_i e^{\langle \vk_i, \vq \rangle}} \nonumber \\
    &= \frac{\sum_{i=1}^N e^{\langle \vk_i, \vq \rangle} \tilde{\vv}_i}{\sum_{i=1}^N w_i e^{\langle \vk_i, \vq \rangle}} + \boldsymbol{\xi} \frac{\sum_{i=1}^N w_i e^{\langle \vk_i, \vq \rangle}}{\sum_{i=1}^N w_i e^{\langle \vk_i, \vq \rangle}} \nonumber \\
    &= s_{\mathcal{D}}(\vq) + \boldsymbol{\xi} \,.
\end{align}
Evaluating the augmented functional with shifted bias $\vb + \boldsymbol{\xi}$, we get
\begin{equation}
    s_{\mathcal{D}_{\boldsymbol{\xi}}, \vb + \boldsymbol{\xi}}(\vq) = s_{\mathcal{D}_{\boldsymbol{\xi}}}(\vq) - (\vb + \boldsymbol{\xi}) = (s_{\mathcal{D}}(\vq) + \boldsymbol{\xi}) - (\vb + \boldsymbol{\xi}) = s_{\mathcal{D}}(\vq) - \vb = s_{\mathcal{D}, \vb}(\vq) \,.
\end{equation}
Choosing $\boldsymbol{\xi} = -\vb$ yields $s_{\mathcal{D}, \vb}(\vq) = s_{\mathcal{D}_{-\vb}, \mathbf{0}}(\vq)$, completing the proof.
\end{proof}

\subsubsection{Gauge Fixing for Strictly Positive Weight Updates}

Suppose at step $t$ the model is stored in the canonical zero-bias frame $\mathcal{D}^{(t)}$ with $\vb = \mathbf{0}$, producing prediction $\vs_t = s_{\mathcal{D}^{(t)}}(\vq_t)$ on query $\vq_t$ and loss gradient $\vg_t = \nabla_{\hat{\vy}} \mathcal{L}(\vs_t, \vy_t)$. 

Instead of computing the functional gradient in the canonical frame (which risks $w_{\text{new}} < 0$), we perform a gauge transformation to view the database in a \textbf{virtual gauge} with bias $\vb_t \in \mathbb{R}^{d_{\text{out}}}$. By \cref{prop:gauge_invariance}, mapping the canonical zero-bias state to a frame with bias $\vb_t$ requires shifting the values by $\boldsymbol{\xi} = \vb_t$, resulting in the shifted virtual database:
\begin{equation}
    \mathcal{D}_{\vb_t}^{(t)} = \left\{ (w_i, \vk_i, \tilde{\vv}_i + w_i \vb_t) \right\}_{i=1}^N \,.
\end{equation}
Recall that
$s_{\mathcal{D}_{\vb_t}^{(t)},\vb_t}(\vq_t)=s_{\mathcal{D}_{\vb_t}^{(t)}}(\vq_t)-\vb_t=s_{\mathcal{D}^{(t)}}(\vq_t)+\vb_t-\vb_t=s_{\mathcal{D}^{(t)}}(\vq_t)$.

Now consider adding a tuple $(w_{\text{new}}, \vk_{\text{new}}, \tilde{\vv}'_{\text{new}})$ to $\mathcal{D}_{\vb_t}^{(t)}$ in this frame. As derived in \cref{app:constructive_update_analysis}, the functional gradient descent update on the weight density $w$ yields a scalar mass update proportional to the projection of the gradient onto the effective prediction relative to the frame bias:
\begin{equation}
    w_{\text{new}} = \eta \left\langle \vg_t, s_{\mathcal{D}_{\vb_t}^{(t)}}(\vq_t) \right\rangle= \eta \left\langle \vg_t, s_{\mathcal{D}^{(t)}}(\vq_t)+\vb_t \right\rangle= \eta \langle \vg_t, \vs_t + \vb_t \rangle \,.
    \label{eq:virtual_mass_update}
\end{equation}
Because the gauge choice $\vb_t$ is a free degree of freedom that leaves the function invariant, we can \textbf{fix the gauge dynamically} at each step to ensure that the inner product in \eqref{eq:virtual_mass_update} is strictly positive. Specifically, we choose:
\begin{equation}
    \vb_t = \left(-\langle \vs_t,\vg_t\rangle /\|\vg_t\|^2+ \alpha_t\right) \vg_t \,,
    \label{eq:gauge_choice}
\end{equation}
where $\alpha_t > 0$ is a scalar gauge parameter. Substituting \eqref{eq:gauge_choice} into \eqref{eq:virtual_mass_update} yields:
\begin{equation}
    w_{\text{new}} = \eta \left\langle \vg_t, \vs_t + \left(-\langle \vs_t,\vg_t\rangle /\|\vg_t\|^2+ \alpha_t\right) \vg_t \right\rangle = \eta \alpha_t \|\vg_t\|^2 > 0 \,.
    \label{eq:strictly_positive_weight}
\end{equation}
Thus, for any non-zero error gradient $\vg_t$, the new weight is \textbf{guaranteed to be strictly positive}.

The target weighted value vector $\tilde{\vv}'_{\text{new}}$ in the virtual frame is directed along the negative error gradient, following standard functional descent on the numerator functional:
\begin{equation}
    \tilde{\vv}'_{\text{new}} = -\beta \vg_t \,,
\end{equation}
where $\beta > 0$ is the value learning rate. The key is set as usual to $\vk_{\text{new}} = \Sigma_k \vq_t = \sigma_k^2 \vq_t$. Appending this tuple to $\mathcal{D}_{\vb_t}^{(t)}$ yields the preliminary database $\mathcal{D}^{\text{prel}, (t+1)}$ in the virtual frame:
$$\mathcal{D}^{\text{prel}, (t+1)}:=\mathcal{D}_{\vb_t}^{(t)}\cup\{(w_{\text{new}},\vk_{\text{new}},\tilde{\vv}'_{\text{new}})\}\,.$$

\subsubsection{Canonical Return and Locality of the Database Update}

To maintain a standard, zero-bias database across training steps, we must transform $\mathcal{D}^{\text{prel}, (t+1)}$ back to the canonical gauge ($\vb = \mathbf{0}$). Shifting the bias from $\vb_t$ back to $\mathbf{0}$ corresponds to applying the shift $\boldsymbol{\xi} = -\vb_t$ to all values.

A naive implementation might suggest that shifting values by $-\vb_t$ requires modifying every existing entry in the database. Remarkably, this is not the case: the gauge transformation yields a local update.

\begin{theorem}[Locality of the Gauge-Invariant Update]
\label{thm:gauge_locality}
The transition from the canonical database $\mathcal{D}^{(t)}$ to $\mathcal{D}^{(t+1)}$ consists solely of appending a single new tuple:
\begin{equation}
    \mathcal{D}^{(t+1)} = \mathcal{D}^{(t)} \cup \left\{ (w_{\text{new}}, \vk_{\text{new}}, \tilde{\vv}_{\text{final}}) \right\} \,,
\end{equation}
where all previously stored tuples remain completely unaltered, and the canonical (weighted) value vector of the new tuple is given by:
\begin{equation}
    \tilde{\vv}_{\text{final}} = \left(-w_{\text{new}} \alpha_t - \beta + \eta\alpha_t \langle \vs_t,\vg_t\rangle\right) \vg_t \,.
    \label{eq:canonical_value_formula}
\end{equation}
\end{theorem}

\begin{proof}
Let $(w_i, \vk_i, \tilde{\vv}_i) \in \mathcal{D}^{(t)}$ be an arbitrary pre-existing tuple.
\begin{enumerate}
    \item \textbf{Forward transformation to virtual gauge:} Under shift $\boldsymbol{\xi} = +\vb_t$, the weighted value becomes $\tilde{\vv}_i' = \tilde{\vv}_i + w_i \vb_t$.
    \item \textbf{Atom insertion:} The new atom $(w_{\text{new}}, \vk_{\text{new}}, \tilde{\vv}'_{\text{new}})$ is added. Existing tuples are unaffected.
    \item \textbf{Return transformation to canonical gauge:} Shifting the bias by $-\vb_t$ transforms all weighted values via $\tilde{\vv} \mapsto \tilde{\vv} - w \vb_t$. For any existing tuple $i$:
    \begin{equation}
        \tilde{\vv}_i'' = \tilde{\vv}_i' - w_i \vb_t = (\tilde{\vv}_i + w_i \vb_t) - w_i \vb_t = \tilde{\vv}_i \,.
    \end{equation}
    All previously stored tuples return identically to their initial values.
\end{enumerate}
For the newly added atom, its weighted value in the virtual frame was $\tilde{\vv}'_{\text{new}} = -\beta \vg_t$. Shifting it to the canonical frame yields:
\begin{align}
    \tilde{\vv}_{\text{final}} &= \tilde{\vv}'_{\text{new}} - w_{\text{new}} \vb_t \nonumber \\
    &= -\beta \vg_t - w_{\text{new}} \left(-\langle \vs_t,\vg_t\rangle /\|\vg_t\|^2+ \alpha_t\right) \vg_t \nonumber \\
    &= \left(-w_{\text{new}} \alpha_t - \beta + w_{\text{new}} \frac{\langle \vs_t,\vg_t\rangle}{\|\vg_t\|^2}\right) \vg_t  \\
    &= \left(-w_{\text{new}} \alpha_t - \beta + \eta\alpha_t \|\vg_t\|^2 \frac{\langle \vs_t,\vg_t\rangle}{\|\vg_t\|^2}\right) \vg_t  \\
    &= \left(-w_{\text{new}} \alpha_t - \beta + \eta\alpha_t \langle \vs_t,\vg_t\rangle\right) \vg_t  \,.
\end{align}
Thus, no historical data points are modified, and the update is strictly local.
\end{proof}

\subsubsection{Algorithm Summary}

The complete procedure is formalized in \cref{alg:gauge_invariant_softmax}. 
As mentioned in Remark~\ref{rem:weighted_softmax_remark}, we store $\vk_{\text{new}} \leftarrow \vq_t$ in the database instead of $\vk_{\text{new}} \leftarrow \sigma_k^2\vq_t$.

\begin{algorithm}[H]
\caption{Gauge-Invariant Weighted Softmax Growing Step}
\label{alg:gauge_invariant_softmax}
\begin{algorithmic}[1]
\State \textbf{Input:} Current database $\mathcal{D}^{(t)}$, input query $\vq_t$, loss gradient $\vg_t = \nabla_{\hat{\vy}} \mathcal{L}(\vs_t, \vy_t)$
\State \textbf{Hyperparameters:} Weight learning rate $\eta > 0$, value learning rate $\beta > 0$, gauge parameter $\alpha_t > 0$, temperature $T$
\State \textbf{1. Compute prediction:}
\State \quad $\vs_t \leftarrow \frac{\sum_{(w, \vk, \tilde{\vv}) \in \mathcal{D}^{(t)}} e^{\frac{1}{T\sqrt{d_{\text{in}}}} \langle \vk, \vq_t \rangle} \tilde{\vv}}{\sum_{(w, \vk, \tilde{\vv}) \in \mathcal{D}^{(t)}} e^{\frac{1}{T\sqrt{d_{\text{in}}}} \langle \vk, \vq_t \rangle} w}$
\State \textbf{2. Compute strictly positive weight update:}
\State \quad $w_{\text{new}} \leftarrow \eta \alpha_t \|\vg_t\|^2$
\State \textbf{3. Compute canonical weighted value:}
\State \quad $\tilde{\vv}_{\text{final}} \leftarrow \left(-w_{\text{new}} \alpha_t - \beta + \eta\alpha_t \langle \vs_t,\vg_t\rangle\right) \vg_t$
\State \textbf{4. Set new key:}
\State \quad $\vk_{\text{new}} \leftarrow \vq_t$
\State \textbf{5. Append new tuple to database:}
\State \quad $\mathcal{D}^{(t+1)} \leftarrow \mathcal{D}^{(t)} \cup \left\{ (w_{\text{new}}, \vk_{\text{new}}, \tilde{\vv}_{\text{final}}) \right\}$
\end{algorithmic}
\end{algorithm}

\subsubsection{Hyperparameter Choice that Yields Changes in the Order of the Gradient Length}

The weight of the added tuple is $w_{\text{new}}= \eta \alpha_t \|\vg_t\|^2$ which is quadratic in the length of the gradient. In order to ensure that we add updates that are in the order of the gradient length, we can set $\alpha_t:=\frac{\alpha}{\|\vg_t\|}$. In this case, we get
$$w_{\text{new}}=\eta \alpha \|\vg_t\|\,,$$
and
\begin{align*}
\tilde{\vv}_{\text{final}} &= \left(-w_{\text{new}} \alpha_t - \beta + \eta\alpha_t \langle \vs_t,\vg_t\rangle\right) \vg_t\\
&= \left(-\eta\alpha_t \|\vg_t\|^2 \alpha_t - \beta + \eta\alpha_t \langle \vs_t,\vg_t\rangle\right) \vg_t\\
&= \left(-\eta\frac{\alpha}{\|\vg_t\|}\|\vg_t\|^2 \frac{\alpha}{\|\vg_t\|} - \beta + \eta\frac{\alpha}{\|\vg_t\|} \langle \vs_t,\vg_t\rangle\right) \vg_t\\
&= \left(-\eta\alpha^2 - \beta + \eta\frac{\alpha}{\|\vg_t\|} \langle \vs_t,\vg_t\rangle\right) \vg_t
\,.
\end{align*}

This yields the following algorithm.

\begin{algorithm}[H]
\caption{Gauge-Invariant Weighted Softmax Learning with Updates in the Order of the Gradient Length}
\label{alg:gauge_invariant_softmax_gradient_order}
\begin{algorithmic}[1]
\State \textbf{Input:} Current database $\mathcal{D}^{(t)}$, input query $\vq_t$, loss gradient $\vg_t = \nabla_{\hat{\vy}} \mathcal{L}(\vs_t, \vy_t)$
\State \textbf{Hyperparameters:} Weight learning rate $\eta > 0$, value learning rate $\beta > 0$, gauge parameter $\alpha > 0$, temperature $T$
\State \textbf{1. Compute prediction:}
\State \quad $\vs_t \leftarrow \frac{\sum_{(w, \vk, \tilde{\vv}) \in \mathcal{D}^{(t)}} e^{\frac{1}{T\sqrt{d_{\text{in}}}} \langle \vk, \vq_t \rangle} \tilde{\vv}}{\sum_{(w, \vk, \tilde{\vv}) \in \mathcal{D}^{(t)}} e^{\frac{1}{T\sqrt{d_{\text{in}}}} \langle \vk, \vq_t \rangle} w}$
\State \textbf{2. Compute strictly positive weight update:}
\State \quad $w_{\text{new}} \leftarrow \eta \alpha \|\vg_t\|$
\State \textbf{3. Compute canonical weighted value:}
\State \quad $\displaystyle\tilde{\vv}_{\text{final}} \leftarrow \left(-\eta\alpha^2 - \beta + \eta\frac{\alpha}{\|\vg_t\|} \langle \vs_t,\vg_t\rangle\right) \vg_t$
\State \textbf{4. Set new key:}
\State \quad $\vk_{\text{new}} \leftarrow \vq_t$
\State \textbf{5. Append new tuple to database:}
\State \quad $\mathcal{D}^{(t+1)} \leftarrow \mathcal{D}^{(t)} \cup \left\{ (w_{\text{new}}, \vk_{\text{new}}, \tilde{\vv}_{\text{final}}) \right\}$
\end{algorithmic}
\end{algorithm}

\section{Optimization Dynamics in the Growing Database}
\label{app:optimization_dynamics}

Standard deep learning optimization techniques such as weight decay and momentum can be naturally incorporated into the growing database framework. 
Weight decay, which corresponds to $L_2$ regularization in the functional space, translates to a simple uniform decay of the scalar weights $w$ and weighted values $\tilde{\vv}$ of all stored key-value pairs at each step. This mechanism not only regularizes the model but also introduces a ``forgetting'' factor where older memories gradually fade in influence. As we show in \cref{app:softmax_weight_decay}, weight-decay corresponds to the update rule
\begin{align}
     \mathcal{D}_{t} =& \{ ((1-\gamma)w, \vk, (1-\gamma)\tilde{\vv}) \mid (w,\vk,\tilde{\vv}) \in \mathcal{D}_{t-1} \} \nonumber \\
     &\quad\cup\, \{ (w_{\text{new}}, \vk_{\text{new}}, \tilde{\vv}_{\text{new}}) \} \,.
\end{align}
This means that the weights of old tuples decay exponentially with training steps, thereby providing a principled path towards ``compression'' of the database: Old tuples with negligible weights can be pruned with minimal impact on the function, potentially allowing the database to be maintained at a constant size.

Momentum, typically implemented by maintaining a velocity vector, would naively require doubling the storage (storing a velocity database alongside the parameter database). However, we derive a more efficient formulation where momentum is realized analytically. By storing the creation timestamp $t_i$ of each key, we can compute the effective momentum contribution of a key during the forward pass via a time-dependent scaling coefficient:
\begin{equation}
    c(t, t_i) = \frac{1-\mu^{t-t_i+1}}{1-\mu} \,,
\end{equation}
where $\mu$ is the momentum factor. As we show in \cref{app:softmax_momentum}, this allows the growing network to benefit from accelerated convergence without the storage overhead of explicit velocity buffers.

One can also combine both approaches. The resulting algorithm applies the geometric decay to the stored weights at every step to handle regularization, while applying the time-dependent momentum coefficient during the retrieval operation to simulate the accumulation of gradients. We provide detailed derivations of these dynamics in \cref{app:softmax_momentum_and_weight_decay}.

In this appendix, we derive the update rules for weight decay and momentum within the growing softmax framework. We assume the decoupled representation where the database stores $(w, \vk, \tilde{\vv})$, and the functions are $w_D(\vk) = \sum w_i e^{\langle \vk, \vk_i \rangle}$ and $w_N(\vk) = \sum \tilde{\vv}_i e^{\langle \vk, \vk_i \rangle}$. Note that here we use $\tilde{\vv}$ to denote the weighted value vector, consistent with the main text notation.

\subsection{Weight Decay}
\label{app:softmax_weight_decay}
Weight decay corresponds to minimizing a regularized objective:
$$ J_{reg}(w_N, w_D) = J(w_N, w_D) + \frac{\gamma}{2\eta} (\|w_N\|_{\mathcal{H}}^2 + \|w_D\|_{\mathcal{H}}^2) \,.$$
Here, we adopt the ``decoupled weight'' formulation of \cref{app:decoupled}. Nevertheless, following a ``coupled weight'' formulation as in \eqref{eq:softmax_with_weight_function_Gaussian} yields similar results.

The functional gradient descent update for $w_D$ becomes:
$$ w_{D,t} \leftarrow w_{D,t-1} - \eta (\nabla_{w_D} J + \frac{\gamma}{\eta} w_{D,t-1}) = (1-\gamma)w_{D,t-1} - \eta \nabla_{w_D} J \,.$$
Similarly for the numerator function $w_N$:
$$ w_{N,t} \leftarrow (1-\gamma)w_{N,t-1} - \eta \nabla_{w_N} J \,.$$
In the database representation, the function $w_{D,t-1}$ is represented by the sum of existing kernels with weights $w_i$. Multiplying the function by $(1-\gamma)$ is equivalent to multiplying every coefficient $w_i$ in the database by $(1-\gamma)$. The same logic applies to the numerator function and its coefficients $\tilde{\vv}_i$.
Thus, the update rule for the database $\mathcal{D}_{t-1}$ is:
$$ \mathcal{D}_{t} = \{ ((1-\gamma)w, \vk, (1-\gamma)\tilde{\vv}) \mid (w,\vk,\tilde{\vv}) \in \mathcal{D}_{t-1} \} \cup \{ (w_{\text{new}}, \vk_{\text{new}}, \tilde{\vv}_{\text{new}}) \} \,,$$
where the new tuple is derived from the gradients $-\eta \nabla_{w_D} J$ and $-\eta \nabla_{w_N} J$ as in \cref{app:constructive_update_analysis}.

This operation implies that the weight of a tuple added at time $t_0$ decays geometrically as $(1-\gamma)^{t-t_0}$ at step $t$. Consequently, older tuples will eventually have weights strictly bounded by some small $\epsilon$. This suggests a natural pruning strategy: removing tuples where $|w| < \epsilon$ and $\|\tilde{\vv}\| < \epsilon$ results in a negligible perturbation to the softmax function output. By setting a threshold, one can effectively bound the maximum size of the database, converting the growing nonparametric model into a fixed-capacity model with a dynamic memory.

\subsection{Momentum}
\label{app:softmax_momentum}
Momentum accelerates optimization by accumulating a velocity function $\nu \in \mathcal{H}$. We apply momentum to the gradient updates for both $w_N$ and $w_D$. Let's focus on $w_D$ (the derivation for $w_N$ is identical). The update rules are:
$$ \nu_t = \mu \nu_{t-1} + \nabla_{w_D} J_t, \quad w_{D,t} = w_{D,t-1} - \eta \nu_t \,. $$
Unrolling the recursion, the weight function at step $t$ can be expressed as a linear combination of all past gradients:
$$ w_{D,t} = w_{D,0} - \eta \sum_{j=1}^t \left( \sum_{k=0}^{t-j} \mu^k \right) \nabla_{w_D} J_j = w_{D,0} - \eta \sum_{j=1}^t \frac{1-\mu^{t-j+1}}{1-\mu} \nabla_{w_D} J_j \,. $$
Recall that the gradient $\nabla_{w_D} J_j$ at step $j$ corresponds to adding a new kernel centered at $\vk_j = \sigma^2 \vq_j$ with a specific weight $\Delta w_j$. The unrolled equation implies that the accumulated weight for the kernel at $\vk_j$ at current time $t$ is simply the initial update weight $\Delta w_j$ scaled by the coefficient:
$$ c(t, t_j) = \frac{1-\mu^{t-t_j+1}}{1-\mu} \,. $$
Therefore, we can implement momentum without storing the velocity function $\nu_t$ explicitly. Instead, we store the creation time $t_j$ for each tuple. During the retrieval (forward pass) at current time $t$, the contribution of the $j$-th tuple to the denominator is scaled by $c(t, t_j)$. The same scaling applies to the weighted value vector $\tilde{\vv}_j$ in the numerator.

\subsection{Combined Dynamics}
\label{app:softmax_momentum_and_weight_decay}
In this section, we provide a rigorous derivation of the update rule when both weight decay and momentum are active. We adopt the \textbf{decoupled weight decay} formulation \citep{LoshchilovH19}, which is the standard for modern deep learning optimizers like AdamW and SGDW.

\subsubsection{Setup and Update Equations}
Let $w_t \in \mathcal{H}$ denote the function state at time $t$. Let $\nabla J_t$ denote the functional gradient of the loss at step $t$.
The update dynamics are defined by the following system of recurrence relations:
\begin{align}
    \nu_t &= \mu \nu_{t-1} + \nabla J_t \label{eq:momentum_update} \,, \\
    w_t &= (1-\gamma) w_{t-1} - \eta \nu_t \label{eq:weight_update} \,.
\end{align}
where:
\begin{itemize}
    \item $\mu \in [0, 1)$ is the momentum coefficient.
    \item $\gamma \in [0, 1)$ is the weight decay factor per step.
    \item $\eta > 0$ is the learning rate.
    \item $\nu_t$ is the velocity function (initialized at $\nu_0 = 0$).
    \item $w_0$ is the initial function state.
\end{itemize}

\subsubsection{Derivation of the Closed-Form State}
We aim to express $w_t$ as a direct summation of the history of gradients $\{\nabla J_j\}_{j=1}^t$.
For convenience, we define the \textbf{effective momentum} $\tilde{\mu}$ and \textbf{effective learning rate} $\tilde{\eta}$ as:
$$ \tilde{\mu} = \frac{\mu}{1-\gamma}, \quad \tilde{\eta} = \frac{\eta}{1-\gamma} \,. $$

First, we unroll the velocity term $\nu_t$ using \eqref{eq:momentum_update}:
\begin{equation}
    \nu_k = \sum_{j=1}^k \mu^{k-j} \nabla J_j \,.
\end{equation}

Next, we substitute this into the weight update \eqref{eq:weight_update}. By unrolling the recurrence for $w_t$, we obtain:
\begin{align}
    w_t &= (1-\gamma) w_{t-1} - \eta \nu_t \nonumber \\
        &= (1-\gamma)^t w_0 - \eta \sum_{k=1}^t (1-\gamma)^{t-k} \nu_k \nonumber \\
        &= (1-\gamma)^t w_0 - \tilde{\eta} \sum_{k=1}^t (1-\gamma)^{t-k+1} \nu_k \,.
\end{align}
Substituting the expression for $\nu_k$ and rearranging the summation order ($1 \le j \le k \le t$):
\begin{equation}
    w_t = (1-\gamma)^t w_0 - \tilde{\eta} \sum_{j=1}^t \nabla J_j \left( \sum_{k=j}^t (1-\gamma)^{t-k+1} \mu^{k-j} \right) \,.
\end{equation}
We simplify the inner summation $\mathcal{C}(t, j)$ as follows:
\begin{align}
    \mathcal{C}(t, j) &=\sum_{k=j}^t (1-\gamma)^{t-k+1} \mu^{k-j}\\
    &= \sum_{m=0}^{t-j} (1-\gamma)^{(t-j+1)-m} \mu^m \nonumber \\
    &= (1-\gamma)^{t-j+1} \sum_{m=0}^{t-j} \left( \frac{\mu}{1-\gamma} \right)^m \nonumber \\
    &= (1-\gamma)^{t-j+1} \sum_{m=0}^{t-j} \tilde{\mu}^m \\
    &= (1-\gamma)^{t-j+1} \frac{1 - \tilde{\mu}^{t-j+1}}{1 - \tilde{\mu}} \,.
\end{align}
Substituting this back into the expression for $w_t$, we arrive directly at the compact closed-form solution:
\begin{equation}
    \label{eq:combined_exact_simplified}
    w_t = (1-\gamma)^t w_0 - \tilde{\eta} \sum_{j=1}^t \left[(1-\gamma)^{t-j+1}\frac{1 - \tilde{\mu}^{t-j+1}}{1- \tilde{\mu}} \right] \nabla J_j\,.
\end{equation}

Similarly to the implementation of momentum, we can store the creation time $t_j$ for each tuple and then take it into account in the forward pass.

\section{Additional technical details on advanced linear attention layers}
\label{app:mesa_foof}

\subsection{DeltaNet}

 In (unnormalized) linear attention, each step writes a new key-value association into a fixed size state via an outer-product update. The delta rule instead updates the value associated with the current key by (partly) overwriting the previously stored value with write strength $\lambda \in [0,1],
    \vv^{\text{new}}_t = \lambda \vv_t + (1-\lambda)\vv_t^{\text{old}}$ and $
    \mW_t = \mW_{t-1} - \vv_t^{\text{old}} \vk_t^\top + \vv_t^{\text{new}}\vk_t^\top$.
In our setting, we set the key to be the activation, i.e., $\vk_t :=\vx_t$ and define the previously stored value for that key as the current prediction $\mW_{t-1}\vx_t$. To compress a mapping from activations to (negative) gradients into a fixed size state as in vanilla gradient descent we set the target value to be the negative output-gradient scaled by a step size $\eta$, i.e., $\vv_t := -\eta \nabla_{\vy_t} \mathcal{L}$, which gives $\vv^{\text{new}}_t = -\lambda \eta \nabla_{\vy_t} \mathcal{L} + (1-\lambda) \mW_{t-1}\vx_t.$

This yields the recurrent update
\begin{equation}
\label{Delta:rec}
    \mW_t = \mW_{t-1} (\mI - \lambda \vx_t \vx_t^\top)  +  (-\lambda \eta \nabla_{\vy_t} \mathcal{L} )\vx_t^\top,
\end{equation}
which shows that a DeltaNet-like update can be interpreted as gradient descent preceded by a targeted forgetting mechanism along the current activation direction. Notably, the update itself corresponds to performing a gradient descent step w.r.t.~$\mW$ on the layer-local, time-dependent loss \begin{equation}
    \ell_t(\mW) := \frac{1}{2} \| \mW\vk_t - \vv_t \|_2^2 = \frac{1}{2} \|\mW\vx_t + \eta \nabla_{\vy_t} \mathcal{L} \|_2^2
    \label{eq:delta-loss-app}
\end{equation}
with learning rate $\lambda$ \citep{yang2024parallelizing, sun2024learning}, meaning the Delta-rule-based update implicitly performs gradient descent within the gradient descent weight update.
In practice, the update can be implemented with little computational overhead by rewriting Eq.~\ref{Delta:rec} to 
\vspace*{-1mm}
\begin{align}
    \mW_t = \mW_{t-1} + \lambda (- \eta \nabla_{\vy_t} \mathcal{L} - \mW_{t-1}\vx_t)\vx_t^\top,
\end{align}
caching the activation $\mW_{t-1} \vx_t$ from the forward pass and using it to form the corrected term $- \eta \nabla_{\vy_t} \mathcal{L} - \mW_{t-1}\vx_t$ before computing the outer-product update. 

\subsection{MesaNet: Problem formulation and connection to FOOF}

Next, we show the connection between FOOF and Mesa layer when applied to the database. We present a slightly modified formulation of the Mesa layer, with decay factor and an additional regularization center. We present the case of batch size 1, but the result can easily be extended to the general case.

Consider the online estimation of a parameter matrix $\mW \in \mathbb{R}^{k \times d}$ from a stream of observations $(\vx_t, \vy_t)$ (key and value of the database) for $t=1, \dots, T$. The mesa layer seeks to minimize the following time-discounted objective function $J_T(\mW)$ at each step $T$:

\begin{equation}\label{app:ea_loss_mesa}
J_T(\mW) = \sum_{t=1}^T \gamma^{T-t} \|\vy_t - \mW \vx_t\|^2 + \Lambda_T \|\mW - \mC_T\|_F^2
\end{equation}

where:
\begin{itemize}
    \item $\gamma \in [0, 1]$ is the forgetting factor.
    \item $\Lambda_T > 0$ is the regularization strength at time $T$.
    \item $\mC_T \in \mathbb{R}^{k \times d}$ is the regularization center (prior mean) at time $T$, taken as the exponential moving average of past weights, defined as $ \mC_T = \gamma \frac{\Lambda_{T-1}}{\Lambda_T} \mC_{T-1} + (1 - \gamma \frac{\Lambda_{T-1}}{\Lambda_T}) \mW_{T-1}$, and $\mC_0 = \mathbf{0}$.
\end{itemize}

The loss in the main text is recovered by choosing $\gamma=1$ and $\Lambda_t = \lambda \forall t$.

It can be shown that the optimal solution $\mW_T$ satisfies the normal equation:
\begin{equation}
\mW_T \Phi_T = \mP_T
\end{equation}
where $\Phi_T$ and $\mP_T$ are the regularized correlation matrix and cross-correlation vector, defined respectively as:
\begin{align}
\Phi_T &= \sum_{t=1}^T \gamma^{T-t} \vx_t \vx_t^\top + \Lambda_T \mI \\
\mP_T &= \sum_{t=1}^T \gamma^{T-t} \vy_t \vx_t^\top + \Lambda_T \mC_T
\end{align}

\subsection{Derivation of the Recursive Update}

We seek to express $\mW_T$ as a function of $\mW_{T-1}$ and the current observation $(\vx_T, \vy_T)$. First, we expand the recursive definitions of the sufficient statistics.

\subsubsection{Recursive Matrix Update}
Separating the term at $t=T$ and substituting $\Phi_{T-1}$:
\begin{equation}
\begin{aligned}
\Phi_T &= \gamma \left( \sum_{t=1}^{T-1} \gamma^{T-1-t} \vx_t \vx_t^\top \right) + \vx_T \vx_T^\top + \Lambda_T \mI \\
&= \gamma (\Phi_{T-1} - \Lambda_{T-1} \mI) + \vx_T \vx_T^\top + \Lambda_T \mI \\
&= \gamma \Phi_{T-1} + \vx_T \vx_T^\top + (\Lambda_T - \gamma \Lambda_{T-1}) \mI
\end{aligned}
\end{equation}

\subsubsection{Recursive Vector Update}
Similarly for $P_T$:
\begin{equation}
\begin{aligned}
\mP_T &= \gamma (\mP_{T-1} - \Lambda_{T-1} \mC_{T-1}) + \vy_T \vx_T^\top + \Lambda_T \mC_T \\
&= \gamma \mP_{T-1} + \vy_T \vx_T^\top + (\Lambda_T \mC_T - \gamma \Lambda_{T-1} \mC_{T-1})
\end{aligned}
\end{equation}

\subsubsection{Sufficient Condition for the Delta Rule}
Assuming $\mW_{T-1} = \mP_{T-1} \Phi_{T-1}^{-1}$, we substitute $\mP_{T-1} = \mW_{T-1} \Phi_{T-1}$ into the expression for $\mP_T$:
\begin{equation}
\mP_T = \mW_{T-1} (\gamma \Phi_{T-1}) + \vy_T \vx_T^\top + (\Lambda_T \mC_T - \gamma \Lambda_{T-1} \mC_{T-1})
\end{equation}
From the matrix update step, we substitute $\gamma \Phi_{T-1} = \Phi_T - \vx_T \vx_T^\top - (\Lambda_T - \gamma \Lambda_{T-1}) \mI$:
\begin{equation}
\mP_T = \mW_{T-1} \left( \Phi_T - \vx_T \vx_T^\top - (\Lambda_T - \gamma \Lambda_{T-1}) \mI \right) + \vy_T \vx_T^\top + (\Lambda_T \mC_T - \gamma \Lambda_{T-1} \mC_{T-1})
\end{equation}
Rearranging terms to isolate the prediction error:
\begin{equation}
\mP_T = \mW_{T-1} \Phi_T + (\vy_T - \mW_{T-1} \vx_T) \vx_T^\top + \mathcal{R}
\end{equation}
where the residual term $\mathcal{R}$ is defined as:
\begin{equation}
\mathcal{R} = (\Lambda_T \mC_T - \gamma \Lambda_{T-1} \mC_{T-1}) - \mW_{T-1} (\Lambda_T - \gamma \Lambda_{T-1})
\end{equation}
Multiplying by $\Phi_T^{-1}$ from the right yields the update for $\mW_T$:
\begin{equation}
\mW_T = \mW_{T-1} + (\vy_T - \mW_{T-1} \vx_T) \vx_T^\top \Phi_T^{-1} + \mathcal{R} \Phi_T^{-1}
\end{equation}

For the residual term $\mathcal{R}$ to vanish, it is necessary and sufficient to impose the following consistency condition on the update of the regularization center $\mC_T$:
\begin{equation}
\label{eq:consistency}
\Lambda_T \mC_T = \gamma \Lambda_{T-1} \mC_{T-1} + (\Lambda_T - \gamma \Lambda_{T-1}) \mW_{T-1},
\end{equation}
which is satisfied by definition.

Under the above condition, and using as target at every timestep $\vy_T = \mW_{T-1}\vx_T - \eta \nabla_{\vy_{T}} \mathcal{L}$, applying the loss in Equation~\ref{app:ea_loss_mesa} to the database would correspond to having a state, $\mW, \Phi$, which is updated following:

\begin{align}
\Phi_T &= \gamma \Phi_{T-1} + \vx_T \vx_T^\top + (\Lambda_T - \gamma \Lambda_{T-1}) \mI \\
\mW_T &= \mW_{T-1} - \eta \nabla_{\mW_{T-1}} \mathcal{L} \cdot \Phi_T^{-1}
\end{align}

It now only remains to define a meaningful $\Lambda_T$.

\subsection{Proposed Algorithms}

We identify two distinct regimes satisfying the consistency condition \eqref{eq:consistency}.

\subsubsection{Algorithm A: Static Regularization}
We assume a constant regularization strength $\Lambda_T = \lambda$ for all $T$.

\textbf{Matrix Update:}
    \[ \Phi_T = \gamma \Phi_{T-1} + \vx_T \vx_T^\top + \lambda(1-\gamma)\mI \]

Which corresponds exactly to the FOOF optimizer. 

\subsubsection{Algorithm B: Cumulative Regularization}

The previous algorithm causes growing norm of $\Phi$ and a decrease in the importance of the prior, which may incur numerical and hyperparameter instabilities. We therefore propose the next algorithm, which we use in practice.

We assume the regularization strength accumulates over time according to $\Lambda_T = \lambda + \gamma \Lambda_{T-1}$. This is equivalent to optimizing the \textit{weighted averaged} loss

\begin{equation}\label{app:ea_loss_mesa2}
J_T(\mW) = \frac{1}{\sum_{t=1}^T \gamma^{T-t}}\sum_{t=1}^T \gamma^{T-t} \|\vy_t - \mW \vx_t\|^2 + \lambda \|\mW - \mC_T\|_F^2
\end{equation}

\textbf{Matrix Update:} \[ \Phi_T = \gamma \Phi_{T-1} + \vx_T \vx_T^\top + \lambda \mI \]

While keeping the above objective, we introduce additional tricks to stabilize numeric and hyperparameter across different layer widths, resulting in the following recursive update for a batch of $B$ inputs $\mX_T \in \mathbb{R}^{d \times B}$:

\begin{align}
\Gamma_T &= \gamma\Gamma_{T-1} + 1 \\
\Psi_T &= \frac{\gamma \Gamma_{T-1}}{\Gamma_T} \Psi_{T-1} + \frac{1}{Bd\Gamma_T}\mX_T \mX_T^\top \\
\mW_T &= \mW_{T-1} - \eta \nabla_{\mW_{T-1}} \mathcal{L} \cdot (\Psi_T + \lambda \mI)^{-1},
\end{align}

while initializing $\Psi_0=\mathbf{0}$, $\mW_0$ to a standard weight initialization, and $\Gamma_0=0$.

\section{Experimental Setup}
\label{app:experimental}

\subsection{Tasks and Model Selection}
\label{app:data}
We run experiments on image classification and teacher-student regression tasks. For all experiments, we report the mean and standard deviation over 5 seeds. To produce the results in Figure \ref{fig:cifar-test-acc}, we run each model/optimizer with a comparable set of hyperparameters where they are shared and for each method add the needed, additional hyperparameters. For Figures \ref{fig:teacher-student}, \ref{fig:deep-representations-and-temperature}, and \ref{fig:bbb}, we select a subset of the hyperparameters relevant to the produced result and again detail the selection in the respective model subsection. The considered ranges of hyperparameters are reported beneath in the respective model subsections. 

Generally, for each plot and for each method displayed in the respective plot, we select those hyperparameters that maximize/minimize the metric displayed in the plot. The only place where we share the selection across multiple plots is Fig.~\ref{fig:cifar-test-acc}. There, the runs visualized in Fig.~\ref{fig:cifar-test-acc}/Left and Fig.~\ref{fig:cifar-test-acc}/Middle are tied in that the two metrics (Test Accuracy and Training Loss) are displayed for the same selection of runs that maximize the final test accuracy. In particular, looking for the hyperparameter that minimizes the training loss would yield a different selection of runs achieving even smaller training loss (but potentially worse generalization). 

In the following two subsections, we first briefly describe the data processing/generation pipeline used for classification and regression experiments.

\subsubsection{Image Classification}
\label{app:data_image}

CIFAR-10 and CIFAR-100 \cite{krizhevsky} are image classification datasets on $32 \times 32 \times 3$ dimensional RGB images from 10 and 100 classes, respectively. Before handing these images to the network, we first apply standard normalization by subtracting the mean and dividing by the standard deviation per channel. Then, to increase the effective size of the training dataset, we apply standard image transformations, randomly sampled during training. These include reflect-padding ($\pm 4$ pixels) followed by a random $32 \times 32$ crop applied to all 3 channels and an epoch-wise horizontal flip. Test data is only normalized without applying augmentation. Before handing the images to our MLP network, we flatten them into a $3072$-dimensional vector. The notion of epoch in the respective experiments refers to the following: take the 50000 training samples from the respective dataset, apply a randomly chosen augmentation (following the above rules) to each sample, and ultimately shuffle the samples. Only for Fig.~\ref{fig:cifar-test-acc}/Right, where only a single pass on the training data is performed, we use the non-augmented standard version of CIFAR-10.

\subsubsection{Teacher-Student Regression}
\label{app:data_teacher_student}
We generate synthetic regression tasks using a random teacher network. The inputs $\vx \in \mathbb{R}^{d_{\text{in}}}$ are sampled from a standard normal distribution and projected onto the unit sphere. The targets for the randomly sampled inputs are generated by a randomly initialized MLP teacher with $L$ residual hidden blocks adding to a residual stream of width $d_e$. Crucially, this teacher has nonlinearities thereby, upon initialization, implementing some nonlinear input-output(target) mapping. Particularly, the teacher first projects the input data to the hidden dimension by applying a linear layer and followed by a ReLU activation. Subsequently, $L-1$ residual blocks, where each block consists of a linear layer and a ReLU activation, are applied to the initial residual stream activation. A final linear projection maps the last residual stream activation mapping to the output feature (the target for the regression task) of dimension $d_{\text{out}}$. All weights are randomly initialized (standard variance scaling init). In our experiments, we set $d_{\text{in}}=100$, $d_e=100$, $L=2$, and $d_{\text{out}}=1$.

\subsection{Model Architecture}
\label{app:model}
Throughout the entire paper, and importantly shared across all methods, all experiments (unless explicitly stated otherwise) are run with a ResNet-MLP architecture. The inputs act as the initial residual stream activation, thereby governing the residual stream dimension $d_{e} = d_{\text{in}}$. The network backbone consists of $L$ residual blocks arranged in a pre-normalization configuration. Specifically, each block applies layer normalization, followed by a trainable module (a standard dense layer or one of the proposed alternatives: Delta, Mesa, RBF, or Weighted Softmax). In the case of a standard dense layer, Delta, or Mesa, a non-linear activation function is applied to the output of the block and the norm rescaling parameters are trained. Note that for RBF and Weighted Softmax, nonlinearities are by design included in the block. Also, for those we do not train the scale parameters of the norm but fix them to be unit. In all cases, the block's output is added back to the residual stream. The final residual stream activation is processed by a final normalization (again scale parameters not trained for RBF and Softmax and trained for all others) before being mapped to the $C$ class logits via a final output block (again, either a dense layer or one of the proposed methods). In the following subsections, we detail the various instantiations of the residual blocks with the methods described in this work.

\subsection{Weighted Softmax}
\label{app:experimental:Softmax}
We evaluate Weighted Softmax kernel-based RCDL models on both synthetic teacher-student regression and image classification tasks, using the classic CIFAR-10 and CIFAR-100 benchmarks. For these experiments, we instantiate the blocks of the ResNet-MLP backbone architecture, described in Appendix \ref{app:model}, with Weighted Softmax RCDL modules.

\begin{algorithm}
\caption{Growing Weighted Softmax Attention}
\label{alg:growing_softmax}
\begin{algorithmic}[1]
\State \textbf{Input:} Stream of query batches $\{\vq_{t,n}\}_{n=1}^{B}$, temperature $T$, value-learning rate $\beta > 0$, weight-learning rate $\eta > 0$, $\mW_0 = \frac{1}{B}$, $\mK_0 \sim \mathcal{N}(0,\frac{1}{\sqrt{d_{\mathrm{in}}}}\mI)$, $\mV_0 \sim \mathcal{N}(0,\frac{1}{\sqrt{d_{\mathrm{out}}}}\mI)$
\State \textbf{Initialize:} $\mathcal{D}_0 \leftarrow \{(\mW_{0}, \mK_{0}, \mV_0)\}$
\For{$t = 1, 2, \ldots$}
    \For{$n = 1, \ldots, B$}
        \State Compute prediction: $\vs_{t,n} \leftarrow \frac{\sum_{(w,\vk,\tilde{\vv})\in\mathcal{D}_{t-1}} e^{\frac{1}{T\sqrt{d_{\text{in}}}}\langle \vk, \vq_{t,n} \rangle}\, \tilde{\vv}}{\sum_{(w,\vk,\tilde{\vv})\in\mathcal{D}_{t-1}} w\, e^{\frac{1}{T\sqrt{d_{\text{in}}}}\langle \vk, \vq_{t,n} \rangle}}$
        \State Compute gradient: $\vg_{t,n} \leftarrow \nabla \mathcal{L}(\vs_{t,n})$
        \State Compute inner product: $\delta_{t,n} \leftarrow \langle \vs_{t,n}, \vg_{t,n} \rangle$
    \EndFor
    \State \textbf{Update database:} $\mathcal{D}_t \leftarrow \mathcal{D}_{t-1} \cup \{(\max\{0, \eta\,\delta_{t,n}\},\; \vq_{t,n},\; -\beta\, \vg_{t,n}) : n = 1,\ldots,B\}$
\EndFor
\end{algorithmic}
\end{algorithm}

A Weighted Softmax module maintains a growing database of key, value, weight tuples $\mathcal{D}_t = (\mK_t, \mV_t, \mW_t) = (\{\vk_i, \vv_i, w_i\}_{i=0}^t)$ of size $(t+1)B$ at time step $t$, where $B$ is the batch size. For numerical stability, we initialize this database with a randomly sampled batch at $t=0$. Particularly, we set $\mathcal{D}_0 = (\mK_0, \mV_0, \mW_0)$. Here, $\mK_0 \in \mathbb{R}^{B\times d_{e}}$ and $\mV_0 \in \mathbb{R}^{B\times d_{\text{out}}}$ ($d_{\text{out}} = d_e$ for the residual blocks) are sampled from a standard normal distribution and subsequently normalized along the feature dimension by root mean squared normalization. All entries of $\mW_0 \in \mathbb{R}^{B}$ are set to a constant weight $w_0 =\frac{1}{B}$, where $B$ is the batch size. At timestep $t > 0$, the database is grown following the steps outlined in Algorithm \ref{alg:growing_softmax}. Briefly, we apply the Weighted Softmax update rule but clip negative weights to 0 (cf. Appendix \ref{app:negative_weights}). Moreover, for numerical stability, we drop the division by the normalization $D_{t-1}$ for both the weight and the value update. Ultimately, we introduce learning rates $\eta$ and $\beta$ to scale weights and values independently.

\input{Tables/nonparametric_softax_teacher_student}

We report the hyperparameters considered to produce the Weighted Softmax numbers in Figure \ref{fig:teacher-student} in Table \ref{tab:nonparametric_softmax_teacher_student}. Additionally, the dimensionality of the inputs and thus of the keys in the synthetic regression task is $d_e = 100$, $\text{Input}=100$.

\input{Tables/hyperparameters_softmax}

The hyperparameters considered to produce the Weighted Softmax numbers in Figure \ref{fig:cifar-test-acc} and \ref{fig:deep-representations-and-temperature} are reported in Table \ref{tab:softmax_cifar_hparams}. Both Figures display data from the same set of experiments. On top, for Figure \ref{fig:cifar-test-acc}/Right we also include the batch size in the sweeps and consider the values $4, 16, 32$ and $64$. For the Weighted Softmax numbers reported in Figure \ref{fig:bbb} we again follow Table \ref{tab:softmax_cifar_hparams} but only consider models with $L=0$ residual blocks. 

\subsubsection{Gauge-Invariant Weighted Softmax}
\label{app:experimental:gauge_softmax}

The Gauge-Invariant Weighted Softmax module uses the same database, initialization and prediction as the Weighted Softmax module (Appendix~\ref{app:experimental:Softmax}). Only the weight and value of each new tuple differ (Algorithm~\ref{alg:gauge_invariant_softmax_gradient_order}). As explained in Appendix~\ref{app:gauge-invariance}, we can choose the gauge offset $\vb_t$ so that the added weight $w_t = \eta\alpha\lVert \vg_t\rVert \geq 0$ is non-negative by construction, so the clipping $\max\{0,\cdot\}$ of Algorithm~\ref{alg:growing_softmax} is no longer needed. As before, we drop the division by $D_{t-1}$ in both updates. In practice we add $\epsilon = 10^{-8}$ to $\lVert \vg_t\rVert$ in the denominator for numerical stability.

\subsection{Hybrid Weighted Softmax}
\label{app:experimental:hybrid}
\begin{algorithm}
\caption{Hybrid Weighted Softmax Attention}
\label{alg:hybrid_softmax}
\begin{algorithmic}[1]
\State \textbf{Input:} Stream of query batches $\{\vq_{t,n}\}_{n=1}^{B}$, temperature $T$, growth threshold $\tau$, maximal size $N$, value-learning rate $\beta > 0$, weight-learning rate $\eta > 0$, Adam learning rate $\alpha$, $\mW_0 = \frac{1}{B}$, $\mK_0 \sim \mathcal{N}(0,\frac{1}{\sqrt{d_{\mathrm{in}}}}\mI)$, $\mV_0 \sim \mathcal{N}(0,\frac{1}{\sqrt{d_{\mathrm{out}}}}\mI)$
\State \textbf{Initialize:} $\mathcal{D}_0 \leftarrow \{(\mW_{0}, \mK_{0}, \mV_0)\}$
\For{$t = 1, 2, \ldots$}
    \For{$n = 1, \ldots, B$}
        \State Compute prediction: $\vs_{t,n} \leftarrow \frac{\sum_{(w,\vk,\tilde{\vv})\in\mathcal{D}_{t-1}} e^{\frac{1}{T\sqrt{d_{\text{in}}}}\langle \vk, \vq_{t,n} \rangle}\, \tilde{\vv}}{\sum_{(w,\vk,\tilde{\vv})\in\mathcal{D}_{t-1}} w\, e^{\frac{1}{T\sqrt{d_{\text{in}}}}\langle \vk, \vq_{t,n} \rangle}}$
        \State Compute gradient: $\vg_{t,n} \leftarrow \nabla \mathcal{L}(\vs_{t,n})$
        \State Compute inner product: $\delta_{t,n} \leftarrow \langle \vs_{t,n}, \vg_{t,n} \rangle$
    \EndFor
    \State \textbf{Select novel queries:} $\mathcal{S}_t \leftarrow \{n : \min_{(\cdot,\vk,\tilde{\vv})\in\mathcal{D}_{t-1}} \|\vq_{t,n} - \vk\|_2 > \tau\}$, keeping only the first $N - |\mathcal{D}_{t-1}|$ elements if necessary
    \State \textbf{Parametric update:} update all $(w, \vk, \tilde{\vv}) \in \mathcal{D}_{t-1}$ with one Adam step (learning rate $\alpha$) on $\frac{1}{B}\sum_{n}\mathcal{L}(\vs_{t,n})$
    \State \textbf{Update database:} $\mathcal{D}_t \leftarrow \mathcal{D}_{t-1} \cup \{(\max\{0, \eta\,\delta_{t,n}\},\; \vq_{t,n},\; -\beta\, \vg_{t,n}) : n \in \mathcal{S}_t\}$
\EndFor
\end{algorithmic}
\end{algorithm}

The Hybrid Weighted Softmax layer keeps the database $\mathcal{D}$ of the Weighted Softmax (Appendix~\ref{app:experimental:Softmax}), but differs in two ways (Algorithm~\ref{alg:hybrid_softmax}). First, all stored tuples are treated as parameters and trained by backpropagation with standard gradient descent optimization. Second, a query is only added if its Euclidean distance to the nearest stored key exceeds the growth threshold $\tau$. The new tuple follows the Weighted Softmax rule of Algorithm~\ref{alg:growing_softmax}, with $\delta_{t,n}$ clipped to $[0,\infty)$ for stability. $\mathcal{D}_0$ consists of $B$ tuples with zero keys and values and weight $\frac{1}{B}$. Because queries are RMS-normalized, $\|\vq\|_2 = \sqrt{d_{\text{in}}} \approx 55.4$, which sets the scale of $\tau$. We define capacity as the total number of stored tuples after training, summed over all layers and divided by the maximum $4N$.
\begin{table}[t]
\centering
\caption{Hyperparameters for the Hybrid Growing Weighted Softmax on CIFAR-10 (Figure~\ref{fig:sub-linear-growth}).}
\label{tab:hybrid_hparams}
\begin{tabular}{@{}ll@{}}
\toprule
\textbf{Hyperparameter} & \textbf{Value} \\
\midrule
Model                        & $L=3$ residual blocks + readout layer, no nonlinearity \\
Embedding dimension $d_e$    & $3072$ (flattened image, no input embedding) \\
Normalization                & RMSNorm (no learned scale) before each layer and the readout \\
Maximal width per layer $N$  & $50016$ \\
Growth threshold $\tau$      & $[35, 40, 45, 50, 55, 60]$ \\
Value learning rate $\beta$ & $[10, 30, 60]$ \\
Weight learning rate $\eta$  & $1.0$ \\
Softmax temperature $T$             & $[1.0, 5.0, 10.0]$ \\
Adam learning rate $\alpha$  & $[0.0003, 0.001, 0.003]$ \\
Batch size                   & $16$ \\
Epochs                       & $1$ \\
\bottomrule
\end{tabular}
\end{table}

Table~\ref{tab:hybrid_hparams} lists the hyperparameters for Figure~\ref{fig:sub-linear-growth}.

\subsection{Parametric Softmax}
\label{app:experimental:BBB}

Here we describe the Parametric Softmax baseline that is referenced in Figures \ref{fig:teacher-student} and \ref{fig:bbb}. The motivation for this baseline is to compare growing a (weighted) softmax kernel-based module with learning a fixed size key-value parameterization for a softmax kernel where all keys and values are treated as parameters i.e., are optimized jointly to minimize the loss.

Unlike the growing module, this baseline uses a static architecture with fixed size parameters (keys, values and weights), pre-allocated with enough capacity to match the size of a fully grown database from the RCDL case from the beginning. We obtain predictions $\vs_{t-1}$ applying a softmax to this huge database, but instead of growing the database step-wise, update the entire key-value parameterization by gradient descent at every timestep.  Essentially, this baseline can be seen as wide MLP with a softmax nonlinearity. 

\input{Tables/hyperparameters_bbb}

We report the hyperparameters considered to produce Parametric Softmax numbers in CIFAR-10 experiments (cf. Figure \ref{fig:bbb}) in Table \ref{tab:bbb_cifar_hparams}.

\input{Tables/parametric_softmax_teacher_student}

For the synthetic Teacher-Student regression task presented in Figure \ref{fig:teacher-student}, we adapt the "database size" of our fixed parametric representations to shapes $(10240000, d_{\text{in}})$ for $\mW_{\text{Key}}$ and $(10240000, d_{\text{out}})$ for $\mW_{\text{Value}}$. We report the remaining hyperparameters for this result in Table \ref{tab:parametric_softmax_teacher_student}.  We find that training the weights for a weighted version of softmax leads to severe training instabilities in this parametric setup and hence report the numbers for an unweighted softmax variant.

\subsection{RBF}
\label{app:experimental:RBF}

\begin{algorithm}
\caption{Growing RBF}
\label{alg:growing_rbf}
\begin{algorithmic}[1]
\State \textbf{Input:} Stream of query batches $\{\vq_{t,n}\}_{n=1}^{B}$, temperature $T$, value-learning rate $\eta > 0$
\State \textbf{Initialize:} $\mathcal{D}_0 \leftarrow \emptyset$
\For{$t = 1, 2, \ldots$}
    \For{$n = 1, \ldots, B$}
        \State Compute prediction: $\vs_{t,n} \leftarrow \sum_{(\vk,\vv)\in\mathcal{D}_{t-1}} \exp\!\left(-\frac{\|\vq_{t,n}-\vk\|_2^2}{2T\sqrt{d_{\text{in}}}}\right) \vv$
        \State Compute gradient: $\vg_{t,n} \leftarrow \nabla \mathcal{L}(\vs_{t,n})$
    \EndFor
    \State \textbf{Update database:} $\mathcal{D}_t \leftarrow \mathcal{D}_{t-1} \cup \{(\vq_{t,n},\; -\eta\, \vg_{t,n}) : n = 1,\ldots,B\}$
\EndFor
\end{algorithmic}
\end{algorithm}

We evaluate RBF kernel-based RCDL models on image classification tasks, using the classic CIFAR-10 and CIFAR-100 benchmarks. For these experiments, we instantiate the blocks of the ResNet-MLP backbone architecture, described in Appendix \ref{app:model}, with RBF RCDL modules. In all RBF kernel-based RCDL modules, we predict and grow the database according to Algorithm \ref{alg:growing_rbf}.

\input{Tables/hyperparameters_rbf}

We report the hyperparameters considered to produce the RBF numbers in Figure \ref{fig:cifar-test-acc} in Table \ref{tab:rbf_cifar_hparams}. On top, for Figure \ref{fig:cifar-test-acc}/Right we also include the batch size in the sweeps and consider the values $4, 16, 32$ and $64$.

\subsection{Mesa (FOOF)}
\label{app:experimental:Mesa}
We evaluate the performance of the Mesa update rule introduced in Section \ref{sec:mesa} on CIFAR-10 and CIFAR-100 benchmarks. For these experiments, we instantiate the blocks of the ResNet-MLP
backbone architecture, described in Appendix \ref{app:model}, with vanilla dense linear layers. We train the parameters of the resulting network by applying the Mesa weight update detailed in Section \ref{sec:mesa}.

\begin{table}[h]
    \centering
    \caption{Best Average Accuracy on CIFAR-10 and CIFAR-100 across different numbers of steps for the newton iteration algorithm for approximating the inverse. We find that generalization performance on CIFAR saturates with $10$ iterations.}
    \label{tab:mesa_newton_steps}
    \begin{tabular}{lccc}
        \toprule
        & \multicolumn{3}{c}{\textbf{Mesa newton-approximation steps}} \\
        \cmidrule(l){2-4}
        \textbf{Dataset} & \textbf{5} & \textbf{10} & \textbf{Inf.} \\
        \midrule
        CIFAR-10 & 0.6914 & 0.7130 & 0.7124  \\
        CIFAR-100 & 0.4149 & 0.4323 & 0.4319  \\
        \bottomrule
    \end{tabular}
\end{table}

\input{Tables/hyperparameters_mesa}

We report the hyperparameters used to produce the numbers for the Mesa update rule in Figure \ref{fig:cifar-test-acc} in Table \ref{tab:mesa_cifar_hparams}. We implement the computation of the matrix inverse (cf. Appendix \ref{app:mesa_foof}) using the Newton-Schulz Iteration. We report the effect of iterations for this algorithm in Table \ref{tab:mesa_newton_steps}.

\subsection{Delta}
\label{app:experimental:DeltaNet}

Lastly, we evaluate the performance of the Delta rule variant introduced in Section \ref{sec:mesa} on the CIFAR-10 and CIFAR-100 benchmark. For these experiments, we instantiate the blocks of the ResNet-MLP backbone architecture, described in Appendix \ref{app:model}, with vanilla dense linear layers. We train the parameters of the resulting network by applying the Delta weight update detailed in Section \ref{sec:mesa}.

\input{Tables/hyperparameters_delta_main}

We report the hyperparameters used to produce the numbers for the Delta update rule in Figure \ref{fig:cifar-test-acc} in Table \ref{tab:hyperparameters_delta_main}. On top, for Figure \ref{fig:cifar-test-acc}/Right we also include the batch size in the sweeps and consider the values $4, 16, 32$ and $64$.

\subsubsection{Additional Experimental Results: Extended training for deep networks with Delta regularization}

\label{app:additionalDelta}
\begin{figure}[ht]
\begin{center}
\centerline{\includegraphics[width=\columnwidth]{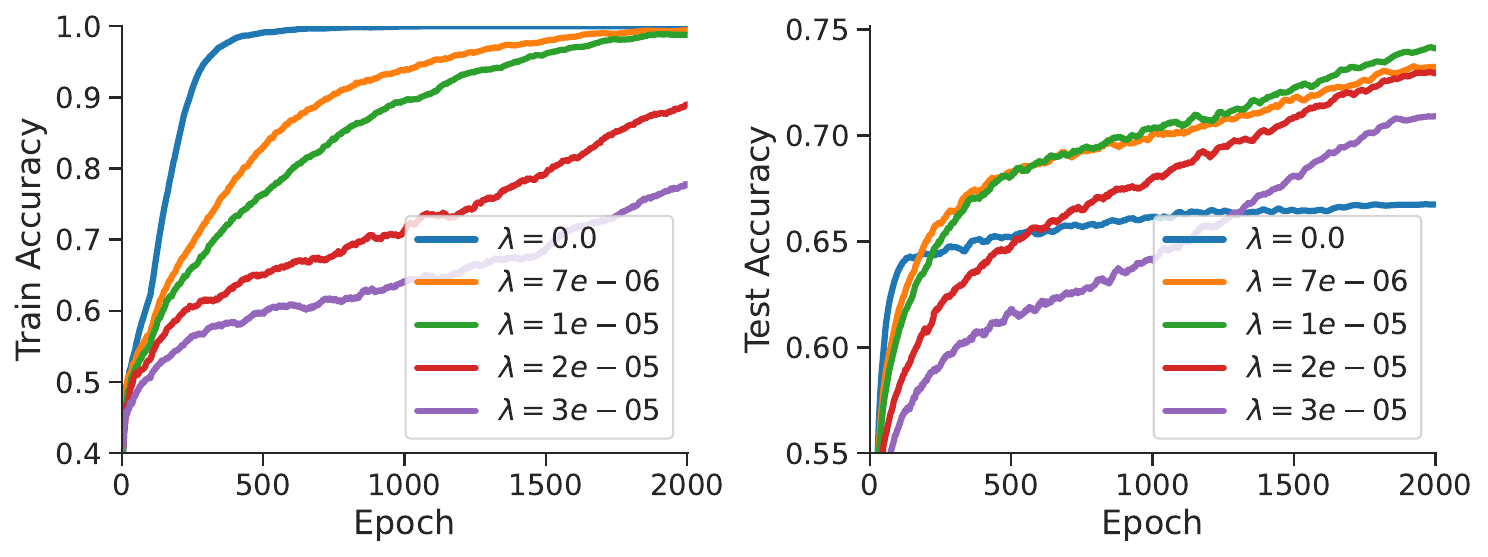}}
\caption{Training and test accuracy for a 10-layer residual MLP trained with SGD under varying DeltaNet decay rate $\lambda$. While $\lambda=0$ achieves near-perfect training accuracy quickly, test accuracy plateaus at around 67\%. Modest DeltaNet-like decay rates significantly improve generalization, as indicated by the reduced train-test gap, which demonstrates the regularizing effect of the update rule. Curves are smoothed using a running average.}
\label{fig:Delta}
\end{center}
\end{figure}

\begin{figure}[ht]
\begin{center}
\centerline{\includegraphics[width=\columnwidth]{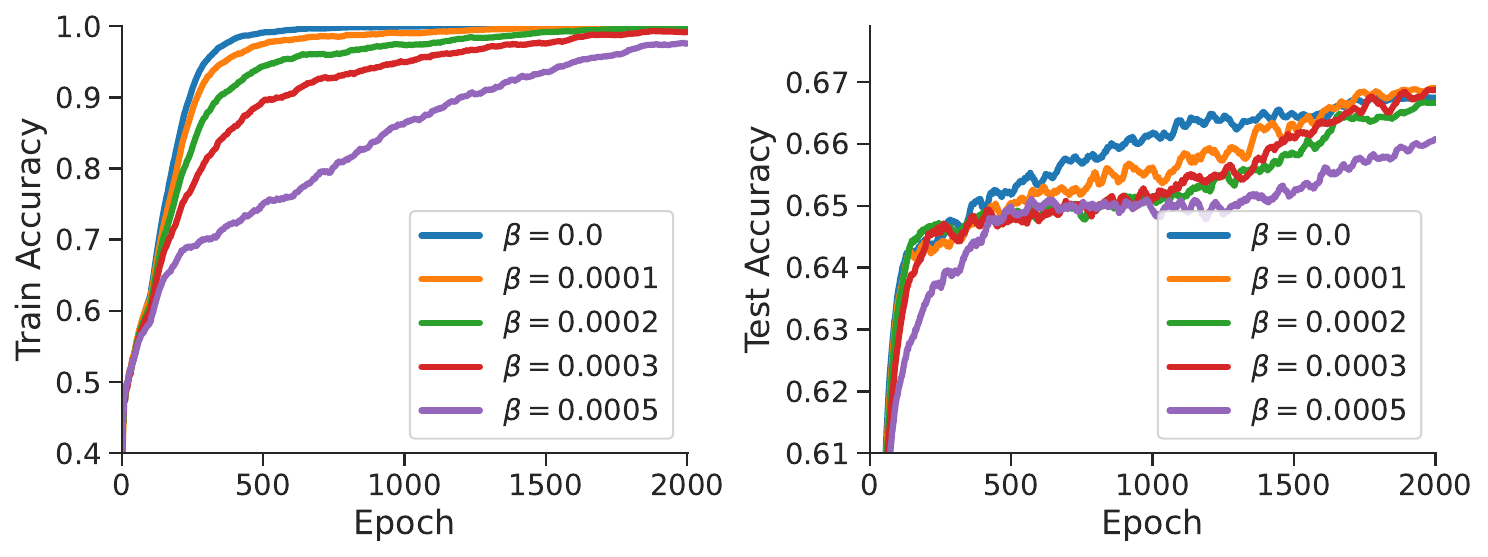}}
\caption{Training and test accuracy for a 10-layer residual MLP trained with SGD under varying multiplicative decay rate $\beta$. Similarly to DeltaNet-like decay the training accuracy increases much slower, however, generalization improves only marginally. Curves are smoothed using a running average.}
\label{fig:weight_decay}
\end{center}
\end{figure}

Here, we present additional experimental results to evaluate the DeltaNet-like gradient descent update, which were obtained in a different experimental setting compared to the results in Section~\ref{sec:exp}, to better assess the regularizing effect for deep neural networks.
Specifically, we train a 10-layer ResNet MLP on Cifar-10 for 2000 epochs. Our architecture downprojects the preprocessed and flattened input to a hidden size of 512, and uses 2-layer MLPs with 4x expansion ratio, GELU activation \citep{hendrycks_gaussian_2016} and a pre-block Layer norm \citep{ba2016layer}.
We train with a batch size of 64 and implement the DeltaNet update in the batch setting by applying all decays within a batch simultaneously, that is, multiplying the linear layer by $\mI - \lambda \mX_t \mX_t^\top$ before applying the gradient update, where $\mX_t \in \mathbb{R}^{ d_{\text{in}} \times B}$ denotes the stacked input activations from batch $t$. In contrast to the derivation in Section~\ref{sec:mesa}, we reparametrize the update rule to align the decay rate with the learning rate scheduling (analogous to standard weight decay). Specifically, we update the weights as
\begin{align}
    \mW_t =  \mW_{t-1}(\mI - \lambda \eta_t \mX_t \mX_t^\top) - \eta_t \mG_{\mY_t} \mX_t^\top,
\end{align}
where $\mG_{\mY_t}$ contains the stacked output gradients, $\lambda$ is the Delta rule decay rate and $\eta_t$ is the learning rate, for which we use cosine annealing with linear warmup. A more detailed description of the experimental setup can be found in Table~\ref{tab:delta_hyperparams}.

We compare plain SGD, SGD with weight decay, and SGD with Delta rule decay when training for $2000$ epochs. Note that this is significantly longer than in other experiments (with a significantly deeper model), as it is computationally feasible and allows us to better understand the effect of regularization. As shown in Fig.~\ref{fig:Delta}, adding a DeltaNet-like decay mechanism significantly improves the final test accuracy (74\% vs. 67\%), far exceeding that of multiplicative weight decay (Fig.~\ref{fig:weight_decay}). While both mechanisms lead to a much slower increase in train accuracy, only the DeltaNet-like decay achieves superior generalization. We also tested combining both decay mechanisms (as in Gated DeltaNet), but this yielded no additional improvement.

Additionally, we experimented with normalized DeltaNet updates as is standard for sequence modeling, i.e., using an input-dependent write strength of $\lambda \eta /\|\vx\|_2^2$. However, this slightly degraded performance.  A likely explanation is that faithfully following the DeltaNet derivation would require scaling the learning rate accordingly; doing so, however, leads to highly unstable optimization, since it makes the (effective) per-sample learning rate dependent on the activation norm.

\input{Tables/hyperparameters_delta}

\subsection{MLP Baselines}
\label{app:experimental:Baselines}
We complement the RCDL modules discussed in this work with strong MLP baselines on image classification tasks, using the classic CIFAR-10 and CIFAR-100 benchmarks. For these experiments, we instantiate the blocks of the ResNet-MLP backbone architecture, described in Appendix \ref{app:model}, with vanilla dense linear layers. We train the parameters of the resulting network with either SGD \cite{robbins1951stochastic}, RMSProp \cite{hinton2012neural}, Adam \cite{kingma2014adam}, AdamW \cite{LoshchilovH19}, and Muon \cite{jordan2024muon}.

\input{Tables/hyperparameters_baselines}

We report the hyperparameters used to produce the numbers for the respective optimizer in Figure \ref{fig:cifar-test-acc} in Table \ref{tab:baselines_cifar_hparams}. On top, for Figure \ref{fig:cifar-test-acc}/Right we also include the batch size in the sweeps and consider the values $4, 16, 32$ and $64$.

\input{Tables/hyperparameters_mlp_teacher_student}
We report the hyperparameters used to produce the numbers for the reference MLP in the Teacher-Student setup in Figure \ref{fig:teacher-student} in Table \ref{tab:mlp_teacher_student}.

\subsection{Nonparametric Baselines}\label{app:experimental:nonparam}
In Figure \ref{fig:cifar-test-acc}/Right we complement the RCDL methods with two nonparametric baselines: A kernel ridge regression (least squares (LS-)SVM) on the 10-dimensional one-hot targets and a k-Nearest Neighbour (kNN) implementation.

\subsubsection{Kernel ridge regression (least squares (LS-)SVM)}

For the LS-SVM we chose the RBF-kernel $k(\vx, \vx') = \exp(- \gamma \lVert \vx - \vx'\rVert_2^2)$ inducing the Gram matrix $\mK \in \mathbb{R}^{n \times n}$, with $K_{i,j} = k(\vx_i, \vx_j)$, where $\vx_i, \vx_j$ are the channel-wise normalized RGB images, flattened into $3072$ dimensional vectors. We parametrize the bandwidth parameter following the median heuristic $\gamma = \frac{\gamma_{\text{mult}}}{\text{median}_{i,j}\lVert \vx_i - \vx_j\rVert_2^2}$ and approximate the sample median on all pairwise distances of a subsample of $2048$ images. With the one-hot targets $\mY \in \mathbb{R}^{n \times 10}$ we invoke a standard solver to compute $\bm{\alpha}=(\mK + \lambda \mI_n)^{-1}\mY \in \mathbb{R}^{n \times 10}$ and evaluate $f(\vx) = \sum_{i=1}^n \alpha_i k(\vx, \vx_i) \in \mathbb{R}^{10}$ on the test samples $\vx$ (predicting the argmax class).

\input{Tables/hyperparameters_lssvm}

We report the considered hyperparameter ranges for producing the data-point displayed in Figure \ref{fig:cifar-test-acc}/Right in Table \ref{tab:svm_cifar_hparams}.

\subsubsection{k-Nearest Neighbors (kNN)}
We implement a kNN search using similarity score $s(\vx, \vx_i)$ on the channel-wise normalized and flattened test image and respective database entry $\vx, \vx_i \in \mathbb{R}^{3072}$. We consider two similarity scores
\begin{equation*}
    s(\vx, \vx_i) = 
\begin{cases} 
  -\lVert \vx - \vx_i \rVert_2^2 & \text{euclidean} \\ 
  \vx^\top \vx_i & \text{unnormalised dot product.} 
\end{cases}
\end{equation*}
Denoting $\mathcal{N}_{k}(\vx)$ as the indices of the top-k neighbors (according to the similarity score) we then weight the labels of these entries either uniformly or proportional to the respective similarity score leading to one of the following weights
\begin{equation*}
    w_i = 
\begin{cases} 
  1 & \text{uniform (regardless of similarity score)} \\ 
  \frac{1}{\lVert \vx - \vx_i \rVert_2 + 10^{-5}} & \text{proportional, euclidean (inverse distance)} \\ 
  \frac{\exp(\vx^\top \vx_i)}{\sum_{j \in N_k} \exp(\vx^\top \vx_j)} & \text{proportional, dot (softmax over top-k),} 
\end{cases}
\end{equation*}
for $i \in \mathcal{N}_k(\vx)$. Using these weights, the prediction for an input $\vx$ is given by taking the argmax over 
\begin{equation*}
    \frac{\sum_{i \in \mathcal{N}_k(\vx)} w_i 1[y_i = c]}{\sum_{i \in \mathcal{N}_k(\vx)} w_i}.
\end{equation*}

To obtain the number reported in Figure \ref{fig:cifar-test-acc}/Right we sweep the grid of the two similarity metrics, the two weighting mechanisms, and $k \in \{5, 10, 15, 20, 25\}$.

\subsection{Vision Transformer Experiments}\label{app:experimental:vit}
We investigate the performance of RCDL modules beyond ResNet architectures in classic vision Transformers (ViTs). We employ pre-LayerNorm transformer blocks consisting of standard softmax-attention and feed-forward networks (FFNs) and a linear readout on the CLS token. In the Weighted Softmax ViT, we replace each block's FFN with a single Growing Weighted Softmax Module (Algorithm~\ref{alg:growing_softmax}), all other components of the model are trained with standard parametric learning. Both are trained for a single epoch on CIFAR-10 (Appendix~\ref{app:data_image}), with hyperparameters in Tables~\ref{tab:vit_standard_cifar_hparams} and~\ref{tab:vit_softmax_cifar_hparams}.
\input{Tables/hyperparameters_vit_standard}
\input{Tables/hyperparameters_vit_softmax}

\subsection{Computational considerations}
\label{app:compute_resource}
Growing Weighted Softmax, RBF-kernel based RCDL, and Parametric Weighted Softmax modules have a significant memory footprint. In the reported CIFAR-\{10,100\} experiments, where each epoch consists of $5\times10^4$ training examples, we add $2.5\times10^6$ tuples $(w,\vv,\vk)$ (or $(\vv,\vk)$, respectively, for RBF-kernel based modules) to $\mathcal{D}$ in growing modules. For a hidden layer with $d_{\text{in}} = d_{\text{out}} = 3072$, implemented in $\text{float}32$, as used in CIFAR experiments, this results in a total memory footprint of $61.5$GB (sum of memory footprint for both keys, values and weights). To realize the training we rely on parallelization and shard this database across $4$ TPUs for $50$-epoch CIFAR experiments. For the teacher-student regression experiments, training for $200$ epochs adds $10^7$ tuples to $\mathcal{D}$ in growing modules. For a hidden layer with $d_{\text{in}} = d_{\text{out}} = 100$, this results in a minimal memory footprint of approximately $8$GB. As a result, this task does not require parallelization or database sharding and can be efficiently trained on a single device.

However, note that the update is relatively cheap in growing as we do not have to compute gradients for these large databases. This is in contrast to Parametric Weighted Softmax modules, where we furthermore have to consider the additional footprint of the equally sized gradients, and if using advanced optimizers such as Adam \citep{kingma2014adam}, first and second moments, making training significantly more expensive and slower than by learning the Weighted Softmax module by growing.

\section{Further Experimental Results}
\label{app:weighted_softmax_exp}

\subsection{Training loss on CIFAR-100}
Complementary to Figure \ref{fig:cifar-test-acc}/Middle where we report the train loss on CIFAR-10, we here report the respective numbers for CIFAR-100 in Figure \ref{fig:cifar-test-acc-cifar-100}.

\begin{figure}[t]

    \centering

    \setlength{\tabcolsep}{4pt} 

    \begin{tabular}{@{} c  @{}}

        \includegraphics[width=0.24\textwidth]{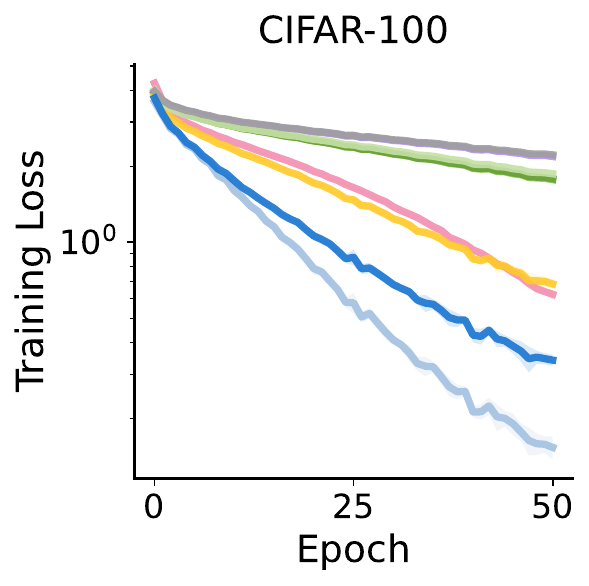}

    \end{tabular}
    \includegraphics[width=\textwidth]{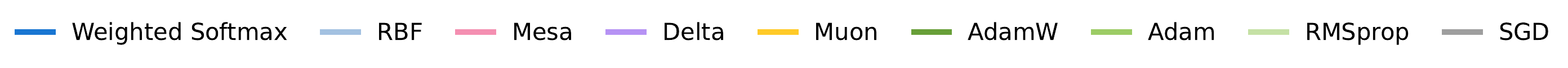}
    \caption{Training loss for Weighted Softmax \& RBF  decreases much faster than that of the other models.}
    \label{fig:cifar-test-acc-cifar-100}

    \vspace{-7mm}

\end{figure}

\subsection{Gauge-Invariant Weighted Softmax}
\label{app:gauge_comparison}
We compare the Gauge-Invariant Weighted Softmax (Algorithm~\ref{alg:gauge_invariant_softmax_gradient_order}) with the standard Weighted Softmax and RBF in the single-epoch setting of Fig.~\ref{fig:cifar-test-acc}/Right: non-augmented CIFAR-10, batch size $16$, $L=2$ residual blocks, $d_e=3072$, and RMSNorm without learned scale. For Weighted Softmax and RBF we use the best configurations from the respective sweeps: Weighted Softmax with $T=6$, $\eta=0.003$, $\beta=3$; RBF with temperature $20$ and learning rate $10$. For the Gauge-Invariant variant we use $T=5$, $\alpha=10$, $\eta=0.03$, $\beta=100$. Fig.~\ref{fig:gauge_wsm_rbf} shows mean and standard deviation over seeds. All three methods reach similar final test accuracy (Weighted Softmax $59.4\pm0.4\%$, RBF $58.8\pm0.2\%$, Gauge-Invariant $58.4\pm0.4\%$). Guaranteeing non-negative weights by the choice of gauge, rather than by clipping, therefore performs comparably and removes the need for the clipping heuristic.
Given this similar behavior, we decided for the results in the main text of the paper to only analyze standard weighted softmax RCDL layers.  
\begin{figure}[t]

    \centering

    \setlength{\tabcolsep}{4pt} 

    \includegraphics[width=0.45\textwidth]{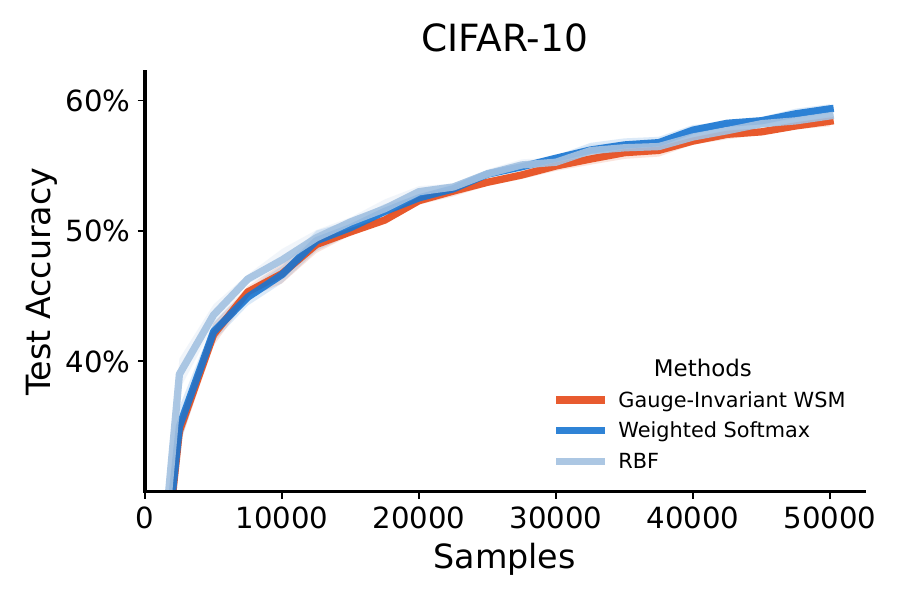}

    \caption{Test accuracy of non-parametric methods (Weighted Softmax, Gauge-Invariant WSM \& RBF) as a function of the number of samples seen during training.}
    \label{fig:gauge_wsm_rbf}

\end{figure}

\section{Further Related Work and Discussions}
\label{app:related_work}

Here we provide further related work and discussions. At a high-level, retrieval-centric deep learning (RCDL) also shares conceptual similarities with existing ``differentiable memory models'' such as Differentiable Neural Computer \citep{graves2014neural,graves2016hybrid}, Memory Networks \citep{weston2014memory,SukhbaatarSWF15} and Neural Episodic Control \citep{pritzel2017neural} (see also more recent work in the context of language models, such as \citet{berges2025memory,eyuboglu2025cartridges}); with a crucial difference that RCDL's memory retrieval operates on the training datapoints. Its growing behavior, recruiting ``resources'' (new key/value allocations) as the system processes increasingly more training data (conceptually to infinity, with practical hardware limitations) is somewhat also reminiscent of Turing machines \citep{turing1936} as is the case for the transformer at runtime \citep{trafo} (see also how \citet{behrouz2026memory} adopt a similar principle to recurrent nets).

RCDL also differs from ``deep kernel learning'' \citep{wilson2016deep} which uses a deep NN as a learnable pre-processor to produce features for a single-layer Gaussian Process (GP); as well as from a special variant of modern Hopfield-lookup layers \citep{ramsauer2021hopfield} that stores the static training data and retrieves their labels via softmax attention. This comparison also provides an interesting contrast to infinite-width deep learning theories, where NNs asymptotically reduce to a kernel machine, such as the neural tangent kernel (NTK; \citet{jacot2018neural,lee2019wide}); in that regime, a central question is whether deep NNs meaningfully learn representations or operate in a lazy training regime with fixed random features \citep{chizat2019lazy}. In contrast, RCDL learns representations in an NN that is explicitly structured as a (deep) kernel machine---an intriguing future work may be to study RCDL through the lens of NTK/GP.

The challenge of expanding model capacity has antecedents in early constructive algorithms \citep{fahlman1989cascade,platt1991resource}, which dynamically allocate new units to reduce error. In modern deep learning, this concept evolved into architectures like ``progressive NNs'' \citep{rusu2016progressive}, which instantiate new columns for novel tasks to prevent forgetting. In fact, the idea of growing networks is rather common in continual learning (see, e.g., \citet{yoonYLH18}), as well as in architecture search/optimization (see, e.g., \citet{stanley2002evolving,EvciMUPV22,butkus2026growing}).
RCDL offers a distinct path: unlike heuristics that add coarse-grained parametric modules, RCDL grows by accumulating fine-grained nonparametric key-value pairs, resulting in capacity that scales linearly with training data. Crucially, while this growth rate is fixed, the update rule itself is strictly principled: we do not add units arbitrarily, but rather derive the addition of each new key-value pair as a precise step of functional gradient descent on the global loss.

``Retrieval'' has also received recent popularity in the context of large language models (LLMs) \citep{lewis2020retrieval, min2023nonparametric} as a method to complement the deficiency of LLMs' knowledge stored in their fixed-size weight matrices.
In a sense, RCDL integrates such a process \textit{within} the network itself; store the entire internet in the NN itself, and ``compression'' occurs at inference time through attention that recombines features---a slogan here is ``\textit{store the entire internet in a growing neural net, compression can wait!}''.\looseness=-1

\end{document}

%% file: math_commands.tex
\usepackage{amsmath,amsfonts,bm}

{}
{}
{}
{}

\def\eqref#1{equation~\ref{#1}}

\def\1{\bm{1}}

\def\vb{{\bm{b}}}

\def\ve{{\bm{e}}}

\def\vg{{\bm{g}}}

\def\vk{{\bm{k}}}

\def\vq{{\bm{q}}}

\def\vs{{\bm{s}}}

\def\vv{{\bm{v}}}

\def\vx{{\bm{x}}}
\def\vy{{\bm{y}}}

\def\mC{{\bm{C}}}

\def\mG{{\bm{G}}}

\def\mI{{\bm{I}}}

\def\mK{{\bm{K}}}

\def\mP{{\bm{P}}}

\def\mV{{\bm{V}}}
\def\mW{{\bm{W}}}
\def\mX{{\bm{X}}}
\def\mY{{\bm{Y}}}

\DeclareMathAlphabet{\mathsfit}{\encodingdefault}{\sfdefault}{m}{sl}
\SetMathAlphabet{\mathsfit}{bold}{\encodingdefault}{\sfdefault}{bx}{n}

\newcommand{\softmax}{\mathrm{softmax}}

\DeclareMathOperator*{\argmin}{arg\,min}

%% file: Tables/nonparametric_softax_teacher_student.tex
\begin{table}[t]
\centering
\caption{Hyperparameters Nonparametric Growing Weighted Softmax for Teacher-Student experiments.}
\label{tab:nonparametric_softmax_teacher_student}
\begin{tabular}{@{}ll@{}}
\toprule
\textbf{Hyperparameter} & \textbf{Value} \\
\midrule
Model                           & Single Weighted Softmax layer \\
Data normalization          & RMSNorm (on input) \\
Value learning rate $\beta$                   & $[0.1, 0.3, 1.0]$ \\
Weight learning rate $\eta$                   & $[10.0, 30.0, 100.0]$ \\
Softmax temperature $T$                        & $[0.1, 1.0, 3.0, 10.0]$ \\
Momentum & $[0.6, 0.9]$\\
Batch size                       & $64$ \\
Epochs                           & $200000$ \\
\bottomrule
\end{tabular}
\end{table}

%% file: Tables/hyperparameters_softmax.tex
\begin{table}[t]
\centering
\caption{Hyperparameters Growing Weighted Softmax for experiments on CIFAR-10 and CIFAR-100.}
\label{tab:softmax_cifar_hparams}
\begin{tabular}{@{}ll@{}}
\toprule
\textbf{Hyperparameter} & \textbf{Value} \\
\midrule
Dataset                         & CIFAR-10, CIFAR-100 \\
Input                           & $32 \times 32 \times 3$ images flattened to $3072$ \\
Model                           & Residual MLP \\
Residual stream size ($d_{\text{e}}$) & $3072$ \\
\# residual blocks  ($L$)                & $[0,1,2]$ \\
Normalization in blocks         & RMSNorm (pre-norm, per block) \\
Final normalization             & RMSNorm before head \\
Classifier head                 & Weighted softmax-kernel RCDL layer \\
Weight learning rate $\eta$       & $[0.3, 1.0, 1.5, 3.0]$ \\
Value learning rate $\beta$    & $[20.0, 30.0, 40.0, 60.0]$ \\
Softmax temperature $T$          & $[1.0,3.0,10.0,15.0,30.0]$ \\
Momentum                         & $[0.0, 0.1, 0.9, 0.999]$ \\
Batch size                       & $64$ \\
Epochs                           & $50$ \\
\bottomrule
\end{tabular}
\end{table}

%% file: Tables/hyperparameters_bbb.tex
\begin{table}[t]
\centering
\caption{Hyperparameters Parametric Softmax for experiments on CIFAR-10.}
\label{tab:bbb_cifar_hparams}
\begin{tabular}{@{}ll@{}}
\toprule
\textbf{Hyperparameter} & \textbf{Value} \\
\midrule
Dataset                          & CIFAR-10 \\
Input                            & $32 \times 32 \times 3$ images flattened to $3072$ \\
Model                            & MLP \\
Normalization in blocks          & RMSNorm (pre-norm, per block) \\
$W_{\text{Key}}$ shape           & (2560000, 3072) \\
$W_{\text{Value}}$ shape         & (2560000, 10) \\
Nonlinear Activation             & Softmax \\
Learning rate $\alpha$             & $[0.003, 0.01, 0.03, 0.1]$ \\
Softmax temperature $T$          & $[0.001,0.03,0.1,0.3,1.0,3.0,10.0]$ \\
Batch size                       & $64$ \\
Epochs                           & $2000$ \\
\bottomrule
\end{tabular}
\end{table}

%% file: Tables/parametric_softmax_teacher_student.tex
\begin{table}[t]
\centering
\caption{Hyperparameters Parametric Softmax for Teacher-Student experiments.}
\label{tab:parametric_softmax_teacher_student}
\begin{tabular}{@{}ll@{}}
\toprule
\textbf{Hyperparameter} & \textbf{Value} \\
\midrule
Model                           & Single Parametric Softmax layer \\
Hidden size & $10240000$ \\
Data normalization          & RMSNorm (on input) \\
Learning rate  $\eta$                  & $[0.0003, 0.001, 0.003, 0.01, 0.03]$ \\
Temperature $T$                        & $[0.001, 0.01, 0.1, 0.3, 1.0, 3.0, 10.0, 30.0, 100.0]$ \\
Optimizer & Adam \\
Momentum & $0.85$ \\
Batch size                       & $64$ \\
Training steps                           & $200000$ \\
\bottomrule
\end{tabular}
\end{table}

%% file: Tables/hyperparameters_rbf.tex
\begin{table}[t]
\centering
\caption{Hyperparameters RBF for experiments on CIFAR-10 and CIFAR-100.}
\label{tab:rbf_cifar_hparams}
\begin{tabular}{@{}ll@{}}
\toprule
\textbf{Hyperparameter} & \textbf{Value} \\
\midrule
Dataset                         & CIFAR-10, CIFAR-100 \\
Input                           & $32 \times 32 \times 3$ images flattened to $3072$ \\
Model                           & Residual MLP \\
Residual stream size ($d_e$) & $3072$ \\
\# residual blocks  ($L$)                & $[0,1,2]$ \\
Normalization in blocks         & RMSNorm (pre-norm, per block) \\
Final normalization             & RMSNorm before head \\
Classifier head                 & Weighted RBF-kernel RCDL layer \\
Value learning rate $\eta$                   & $[0.3, 1.0, 3.0]$ \\
Temperature $T$        & $[0.5,1.0, 1.5,5.0,15.0]$ \\
Momentum                         & $[0.0, 0.1, 0.9, 0.999]$ \\
Batch size                       & $64$ \\
Epochs                           & $50$ \\
\bottomrule
\end{tabular}
\end{table}

%% file: Tables/hyperparameters_mesa.tex
\begin{table}[t]
\centering
\caption{Hyperparameters Mesa for experiments on CIFAR-10 and CIFAR-100.}
\label{tab:mesa_cifar_hparams}
\begin{tabular}{@{}ll@{}}
\toprule
\textbf{Hyperparameter} & \textbf{Value} \\
\midrule
Dataset                         & CIFAR-10, CIFAR-100 \\
Input                           & $32 \times 32 \times 3$ images flattened to $3072$ \\
Model                           & Residual MLP \\
Residual stream size ($d_e$) & $3072$ \\
\# residual blocks  ($L$)                & $[2,4]$ \\
Normalization in blocks         & RMSNorm (pre-norm, per block) \\
Final normalization             & RMSNorm before head \\
MLP activation                  & GELU \\
Classifier head                 & Linear $(3072 \rightarrow 10)$ \\
Mesa-$\gamma$                   & $[0.9, 0.99, 1.0]$ \\
Mesa-$\lambda$                   & $[0.001,0.01,0.1]$ \\
Optimizer                        & Adam \\
Learning rate   $\alpha$                 & $[0.0001,0.0003,0.001,0.003,0.01]$ \\
Momentum                         & $0.85$ \\    
Batch size                       & $64$ \\
Epochs                           & $50$ \\
\bottomrule
\end{tabular}
\end{table}

%% file: Tables/hyperparameters_delta_main.tex
\begin{table}[t]
\centering
\caption{Hyperparameters Delta for experiments on CIFAR-10 and CIFAR-100.}
\label{tab:hyperparameters_delta_main}
\begin{tabular}{@{}ll@{}}
\toprule
\textbf{Hyperparameter} & \textbf{Value} \\
\midrule
Dataset                         & CIFAR-10, CIFAR-100 \\
Input                           & $32 \times 32 \times 3$ images flattened to $3072$ \\
Model                           & Residual MLP \\
Residual stream size ($d_e$) & $3072$ \\
\# residual ($L$)               & $[2,4]$ \\
Normalization in blocks         & RMSNorm (pre-norm, per block) \\
Final normalization             & RMSNorm before head \\
MLP activation                  & GELU \\
Decay rate $\lambda$                   & $[0.0000003,0.000001,0.000003,0.00001,0.0001,0.0003]$ \\
Classifier head                 & Linear $(3072 \rightarrow 10)$ \\
Learning rate $\alpha$                   & $[0.0001,0.0003,0.001,0.003,0.01]$ \\
Batch size                       & $64$ \\
Epochs                           & $50$ \\
\bottomrule
\end{tabular}
\end{table}

%% file: Tables/hyperparameters_delta.tex
\begin{table}[t]
\centering
\caption{Hyperparameters used for DeltaNet-GD experiments on extended training on Cifar-10 with deep nets.}
\label{tab:delta_hyperparams}
\begin{tabular}{@{}ll@{}}
\toprule
\textbf{Hyperparameter} & \textbf{Value} \\
\midrule
Dataset                         & CIFAR-10 \\
Input                           & $32 \times 32 \times 3$ images flattened to $3072$ \\
Model                           & Residual MLP \\
Input projection                & Linear $(3072 \rightarrow 512)$ \\
Residual stream size ($d_e$) & $512$ \\
\# residual blocks               & $10$ \\
Normalization in blocks         & LayerNorm (pre-norm, per block) \\
Final normalization             & LayerNorm before head \\
MLP expansion                   & $4\times$ \\
MLP hidden dimension             & $2048$ \\
MLP activation                  & GELU \\
Dropout                         & $0.1$ \\
Classifier head                 & Linear $(512 \rightarrow 10)$ \\
Optimizer                        & SGD \\
Learning rate $\alpha$                    & $1\times 10^{-2}$ \\
Momentum                         & $0.0$ \\
LR schedule                      & Step-wise: 5 epoch linear warmup + cosine annealing \\
Cosine schedule                  & CosineAnnealingLR over remaining steps \\
Batch size                       & $64$ \\
Epochs                           & $2000$ \\
\bottomrule
\end{tabular}
\end{table}

%% file: Tables/hyperparameters_baselines.tex
\begin{table}[t]
\centering
\caption{Hyperparameters MLP Baselines for experiments on CIFAR-10 and CIFAR-100.}
\label{tab:baselines_cifar_hparams}
\begin{tabular}{@{}ll@{}}
\toprule
\textbf{Hyperparameter} & \textbf{Value} \\
\midrule
Dataset                         & CIFAR-10, CIFAR-100 \\
Input                           & $32 \times 32 \times 3$ images flattened to $3072$ \\
Model                           & Residual MLP \\
Residual stream size ($d_e$) & $3072$ \\
\# residual ($L$)               & $[2,4]$ \\
Normalization in blocks         & RMSNorm (pre-norm, per block) \\
Final normalization             & RMSNorm before head \\
MLP activation                  & GELU \\
Classifier head                 & Linear $(3072 \rightarrow 10)$ \\
Optimizer                        & $[\text{SGD},\text{RMSProp},\text{Adam},\text{AdamW},\text{Muon}]$ \\
Learning rate $\alpha$                   & $[0.0001,0.0003,0.001,0.003,0.01]$ \\
Weight decay (for AdamW)         & $[0.0001,0.01]$ \\
Momentum                         & $0.0$ \\    
Batch size                       & $[16,64]$ \\
Epochs                           & $50$ \\
\bottomrule
\end{tabular}
\end{table}

%% file: Tables/hyperparameters_mlp_teacher_student.tex
\begin{table}[t]
\centering
\caption{Hyperparameters MLP Baseline for Teacher-Student setup.}
\label{tab:mlp_teacher_student}
\begin{tabular}{@{}ll@{}}
\toprule
\textbf{Hyperparameter} & \textbf{Value} \\
\midrule
Model                           & Residual MLP \\
Residual stream size ($d_e$) & $[256, 512, 1024,4096,8192]$ \\
\# residual ($L$)               & $2$ \\
Normalization         & None \\
MLP activation                  & ReLU \\
Optimizer                        & SGD \\
Learning rate                    & $[0.0003,0.001,0.003,0.01, 0.03, 0.1, 0.3, 1.0, 3.0]$ \\
Momentum                         & $[0.0,0.9,0.999]$ \\    
Batch size                       & $[64, 256, 1024]$ \\
Training steps                           & $200000$ \\
\bottomrule
\end{tabular}
\end{table}

%% file: Tables/hyperparameters_lssvm.tex
\begin{table}[t]
\centering
\caption{Hyperparameters RBF kernel regression for experiments on CIFAR-10.}
\label{tab:svm_cifar_hparams}
\begin{tabular}{@{}ll@{}}
\toprule
\textbf{Hyperparameter} & \textbf{Value} \\
\midrule
Dataset                         & CIFAR-10\\
Input                           & $32 \times 32 \times 3$ images flattened to $3072$ \\
Model                           & RBF kernel \\
Classification                & $\arg \max$ on predicted output \\
Optimizer                        & Exact solver \\
Bandwidth $\gamma_{\text{mult}}$                    & $[0.25, 0.5, 1.0, 2.0, 3.0, 5.0]$ \\
Ridge regulariser $\lambda$         & $[0.001,0.01, 0.1, 0.3, 1.0]$ \\
\bottomrule
\end{tabular}
\end{table}

%% file: Tables/hyperparameters_vit_standard.tex
\begin{table}[t]
\centering
\caption{Hyperparameters Standard ViT for single-epoch experiments on CIFAR-10.}
\label{tab:vit_standard_cifar_hparams}
\begin{tabular}{@{}ll@{}}
\toprule
\textbf{Hyperparameter} & \textbf{Value} \\
\midrule
Dataset                          & CIFAR-10 \\
Input                            & $32 \times 32 \times 3$ images, $4 \times 4$ patches ($64$ patch tokens + CLS token) \\
Model                            & ViT, $2$ transformer blocks \\
Embedding dimension $d_e$        & $128$ \\
Attention heads                  & $8$ \\
Positional embedding             & learned \\
Normalization in blocks          & LayerNorm (pre-norm, per block) + final LayerNorm \\
FFN                              & Dense ($128 \to 512$) $\to$ GELU $\to$ Dense ($512 \to 128$) \\
FFN blowup factor                & $[1,2,4,8,16,32]$ \\
Readout                          & Linear layer on CLS token \\
Optimizer                        & Adam \\
Learning rate    $\alpha$                & $[0.0003, 0.001, 0.003]$ \\
Batch size                       & $64$ \\
Epochs                           & $1$ \\
Seeds                            & $5$ \\
\bottomrule
\end{tabular}
\end{table}

%% file: Tables/hyperparameters_vit_softmax.tex
\begin{table}[t]
\centering
\caption{Hyperparameters Weighted Softmax ViT for single-epoch experiments on CIFAR-10.}
\label{tab:vit_softmax_cifar_hparams}
\begin{tabular}{@{}ll@{}}
\toprule
\textbf{Hyperparameter} & \textbf{Value} \\
\midrule
Dataset                          & CIFAR-10 \\
Input                            & $32 \times 32 \times 3$ images, $4 \times 4$ patches ($64$ patch tokens + CLS token) \\
Model                            & ViT, $2$ transformer blocks, FFN replaced by growing GWS \\
Embedding dimension $d_e$        & $128$ \\
Attention heads                  & $8$ \\
Positional embedding             & learned \\
Normalization in blocks          & LayerNorm (pre-norm, per block) + final LayerNorm \\
Readout                          & Linear layer on CLS token \\
Optimizer (non-growing params.)  & Adam \\
Learning rate (Adam) $\alpha$    & $[0.0003, 0.001, 0.003]$ \\
Weight learning rate $\eta$      & $[0.1, 0.3, 1.0]$ \\
Value learning rate $\beta$     & $[10, 30]$ \\
Softmax temperature $T$          & $[0.1, 0.3, 1.0, 3.0]$ \\
Value momentum                   & $[0.0, 0.999]$ \\
Batch size                       & $64$ \\
Epochs                           & $1$ \\
Seeds                            & $5$ \\
\bottomrule
\end{tabular}
\end{table} 

%% file: references.bib
@STRING{confEMNLP={Proc. Conf. on Empirical Methods in Natural Language Processing (EMNLP)}}

@STRING{confICLR = "Int. Conf. on Learning Representations (ICLR)"}

@STRING{confNIPS = "Proc. Advances in Neural Information Processing Systems (NIPS)"}

@STRING{confNEU = "Proc. Advances in Neural Information Processing Systems (NeurIPS)"}

@STRING{confICML = "Proc. Int. Conf. on Machine Learning (ICML)"}

@article{rusu2016progressive,
  title={Progressive neural networks},
  author={Rusu, Andrei A and Rabinowitz, Neil C and Desjardins, Guillaume and Soyer, Hubert and Kirkpatrick, James and Kavukcuoglu, Koray and Pascanu, Razvan and Hadsell, Raia},
  journal={Preprint arXiv:1606.04671},
  year={2016}
}

@article{platt1991resource,
  title={A resource-allocating network for function interpolation},
  author={Platt, John},
  journal={Neural Computation},
  volume={3},
  number={2},
  pages={213--225},
  year={1991},
}

@inproceedings{fahlman1989cascade,
  title={The cascade-correlation learning architecture},
  author={Fahlman, Scott and Lebiere, Christian},
  booktitle=confNIPS,
  volume={2},
  year={1989}
}

@article{friedman2000greedy,
  author = {Friedman, Jerome H.},
  journal = {Annals of Statistics},
  pages = {1189--1232},
  title = {Greedy Function Approximation: A Gradient Boosting Machine},
  volume = 29,
  year = 2000
}

@inproceedings{mason1999boost,
  author       = {Llew Mason and
                  Jonathan Baxter and
                  Peter L. Bartlett and
                  Marcus R. Frean},
  editor       = {Sara A. Solla and
                  Todd K. Leen and
                  Klaus{-}Robert M{\"{u}}ller},
  title        = {Boosting Algorithms as Gradient Descent},
  booktitle    = confNIPS,
  year         = {1999},
}

@inproceedings{grubb2010boosted,
  author       = {Alexander Grubb and
                  J. Andrew Bagnell},
  editor       = {Johannes F{\"{u}}rnkranz and
                  Thorsten Joachims},
  title        = {Boosted Backpropagation Learning for Training Deep Modular Networks},
  booktitle    = confICML,
  year         = {2010},
}

@inproceedings{martens2015kfac,
  author       = {James Martens and
                  Roger B. Grosse},
  editor       = {Francis R. Bach and
                  David M. Blei},
  title        = {Optimizing Neural Networks with Kronecker-factored Approximate Curvature},
  booktitle    = {Proceedings of the 32nd International Conference on Machine Learning,
                  {ICML} 2015, Lille, France, 6-11 July 2015},
  series       = {{JMLR} Workshop and Conference Proceedings},
  volume       = {37},
  pages        = {2408--2417},
  publisher    = {JMLR.org},
  year         = {2015},
  url          = {http://proceedings.mlr.press/v37/martens15.html},
  bibsource    = {dblp computer science bibliography, https://dblp.org}
}

@inproceedings{benzing2022foof,
  author       = {Frederik Benzing},
  editor       = {Kamalika Chaudhuri and
                  Stefanie Jegelka and
                  Le Song and
                  Csaba Szepesv{\'{a}}ri and
                  Gang Niu and
                  Sivan Sabato},
  title        = {Gradient Descent on Neurons and its Link to Approximate Second-order
                  Optimization},
  booktitle    = confICML,
  year         = {2022},
}

@inproceedings{bahdanau2014neural,
  address =       {San Diego, {CA}, {USA}},
  author =        {Bahdanau, Dzmitry and Cho, Kyunghyun and
                   Bengio, Yoshua},
  booktitle =     {Int. Conf. on Learning Representations (ICLR)},
  month =         may,
  title =         {Neural machine translation by jointly learning to
                   align and translate},
  year =          {2015},
}

@inproceedings{katharopoulos2020transformers,
  title={Transformers are {RNN}s: Fast autoregressive transformers with linear attention},
  author={Katharopoulos, Angelos and Vyas, Apoorv and Pappas, Nikolaos and Fleuret, Fran{\c{c}}ois},
  booktitle=confICML,
  address = {Virtual only},
  month=jul,
  year={2020}
}

@inproceedings{trafo,
  address =       {Long Beach, {CA}, {USA}},
  author =        {Vaswani, Ashish and Shazeer, Noam and Parmar, Niki and
                   Uszkoreit, Jakob and Jones, Llion and Gomez, Aidan N and
                   Kaiser, {\L}ukasz and Polosukhin, Illia},
  booktitle =     confNIPS,
  month =         dec,
  pages =         {5998--6008},
  title =         {Attention is All you Need},
  year =          {2017},
}

@article{schmidhuber1992learning,
  title = {Learning to control fast-weight memories: An alternative to dynamic recurrent networks},
  author = {Schmidhuber, J{\"u}rgen},
  journal = {Neural Computation},
  volume = {4},
  number = {1},
  pages = {131--139},
  year = {1992},
}

@inproceedings{yang2024gated,
  title={Gated Delta Networks: Improving {M}amba2 with Delta Rule},
  author={Yang, Songlin and Kautz, Jan and Hatamizadeh, Ali},
  booktitle=confICLR,
    address = {Vancouver, Canada}, 
    month = apr,
  year={2025}
}

@inproceedings{
    dosovitskiy2021an,
    title={An Image is Worth 16x16 Words: Transformers for Image Recognition at Scale},
    author={Alexey Dosovitskiy and Lucas Beyer and Alexander Kolesnikov and Dirk Weissenborn and Xiaohua Zhai and Thomas Unterthiner and Mostafa Dehghani and Matthias Minderer and Georg Heigold and Sylvain Gelly and Jakob Uszkoreit and Neil Houlsby},
    booktitle=confICLR,
    address = {Virtual only},
    month=may,
    year={2021},
}

@article{shen2018efficient,
  title={Efficient Attention: Attention with Linear Complexities},
  author={Shen, Zhuoran and Zhang, Mingyuan and Zhao, Haiyu and Yi, Shuai and Li, Hongsheng},
  journal={Preprint arXiv:1812.01243},
  year={2018}
}

@inproceedings{ramsauer2021hopfield,
title={Hopfield Networks is All You Need},
author={Hubert Ramsauer and Bernhard Sch{\"a}fl and Johannes Lehner and Philipp Seidl and Michael Widrich and Lukas Gruber and Markus Holzleitner and Thomas Adler and David Kreil and Michael K Kopp and G{\"u}nter Klambauer and Johannes Brandstetter and Sepp Hochreiter},
booktitle=confICLR,
year={2021},
month=may,
address={Virtual only}
}

@inproceedings{kingma2014adam,
  title={Adam: A method for stochastic optimization},
  author={Kingma, Diederik P and Ba, Jimmy},
  booktitle=confICLR,
  month = may,
  address={San Diego, CA, USA},
  year={2015}
}

@inproceedings{schlag2021linear,
      title={Linear {T}ransformers Are Secretly Fast Weight Programmers}, 
      author={Imanol Schlag and Kazuki Irie and J\"urgen Schmidhuber},
      booktitle=confICML,
      address = {Virtual only},
      month = jul,
      year={2021}
}

@inproceedings{peng2021random,
  title={Random Feature Attention},
  author={Peng, Hao and Pappas, Nikolaos and Yogatama, Dani and Schwartz, Roy and Smith, Noah A and Kong, Lingpeng},
  booktitle=confICLR,
  address = {Virtual only},
  year={2021}
}

@inproceedings{widrow1960adaptive,
  title={Adaptive switching circuits},
  author={Widrow, Bernard and Hoff, Marcian E},
  year={1960},
  booktitle={Proc. {IRE} WESCON Convention Record},
  month = aug,
  pages = {96--104},
  address = {Los Angeles, {CA}, {USA}}
}

@article{graves2014neural,
  title={Neural turing machines},
  author={Graves, Alex and Wayne, Greg and Danihelka, Ivo},
  journal={Preprint arXiv:1410.5401},
  year={2014}
}

@article{graves2016hybrid,
  title={Hybrid computing using a neural network with dynamic external memory},
  author={Graves, Alex and Wayne, Greg and Reynolds, Malcolm and Harley, Tim and Danihelka, Ivo and Grabska-Barwi{\'n}ska, Agnieszka and Colmenarejo, Sergio G{\'o}mez and Grefenstette, Edward and Ramalho, Tiago and Agapiou, John and others},
  journal={Nature},
  volume={538},
  number={7626},
  pages={471--476},
  year={2016},
}

@article{ba2016layer,
  author =        {Ba, Jimmy  and Kiros, Jamie Ryan and
                   Hinton, Geoffrey E},
  journal =       {Preprint arXiv:1607.06450},
  title =         {Layer Normalization},
  year =          {2016},
}

@article{rmsprop,
  author =        {Tieleman, Tijmen and Hinton, Geoffrey},
  title =         {Lecture 6.5-
{RMSP}rop: Divide the gradient by a running average of
its recent magnitude},
  journal = {{COURSERA}: Neural Networks for
Machine Learning, 4},
  year =          {2012},
}

@article{hendrycks_gaussian_2016,
	title = {Gaussian {Error} {Linear} {Units} ({GELUs})},
	journal = {arXiv preprint arXiv:1606.08415},
	author = {Hendrycks, Dan and Gimpel, Kevin},
	year = {2016},
}

@inproceedings{irie2021going,
      title={Going Beyond Linear Transformers with Recurrent Fast Weight Programmers}, 
      author={Kazuki Irie and Imanol Schlag and R\'obert Csord\'as and J\"urgen Schmidhuber},
      booktitle=confNEU,
      address = {Virtual only},
      month=dec,
      year={2021}
}

@inproceedings{irie2022dual,
  title={The Dual Form of Neural Networks Revisited: Connecting Test Time Predictions to Training Patterns via Spotlights of Attention},
  author={Irie, Kazuki and Csord{\'a}s, R{\'o}bert and Schmidhuber, J{\"u}rgen},
  booktitle=confICML,
  address = {Baltimore, {MD}, {USA}},
  month = jul,
  year={2022}
}

@article{aizerman1964theoretical,
  title={Theoretical foundations of potential function method in pattern recognition},
  author={Aizerman, Mark A. and Braverman, Emmanuil M. and Rozonoer, Lev I.},
  journal={Automation and Remote Control},
  volume={25},
  number={6},
  pages={917--936},
  year={1964}
}

@inproceedings{IrieSCS22,
  author    = {Kazuki Irie and
               Imanol Schlag and
               R{\'{o}}bert Csord{\'{a}}s and
               J{\"{u}}rgen Schmidhuber},
  title     = {A Modern Self-Referential Weight Matrix That Learns to Modify Itself},
  booktitle = confICML,
  address = {Baltimore, {MA}, {USA}},
  month = jul,
  pages     = {9660--9677},
  year      = {2022},
}

@inproceedings{irie2022neural,
  title={Neural Differential Equations for Learning to Program Neural Nets Through Continuous Learning Rules},
  author={Irie, Kazuki and Faccio, Francesco and Schmidhuber, J{\"u}rgen},
  booktitle=confNEU,
  address = {New Orleans, {LA}, {USA}},
  month = dec,
  year={2022}
}

@MASTERSTHESIS{krizhevsky,
  author = {Alex Krizhevsky},
  title = {Learning multiple layers of features from tiny images},
  school = {Computer Science Department, University of Toronto},
  year = {2009}
}

@article{robbins1951stochastic,
  title={A Stochastic Approximation Method},
  author={Robbins, Herbert and Monro, Sutton},
  journal={The Annals of Mathematical Statistics},
  volume={22},
  number={3},
  pages={400--407},
  year={1951}
}

@misc{hinton2012neural,
  title={Neural Networks for Machine Learning (Lecture 6e)},
  author={Hinton, Geoffrey and Srivastava, Nitish and Swersky, Kevin},
  howpublished={Coursera, video lectures},
  year={2012}
}

@inproceedings{von2022transformers,
  author       = {Johannes von Oswald and
                  Eyvind Niklasson and
                  Ettore Randazzo and
                  Jo{\~{a}}o Sacramento and
                  Alexander Mordvintsev and
                  Andrey Zhmoginov and
                  Max Vladymyrov},
  title        = {Transformers Learn In-Context by Gradient Descent},
  booktitle    = confICML,
  month = jul,
  address = {Honolulu, HI, {USA}},
  year         = {2023},
}

@article{von2023uncovering,
  title={Uncovering mesa-optimization algorithms in {T}ransformers},
  author={von Oswald, Johannes and Niklasson, Eyvind and Schlegel, Maximilian and Kobayashi, Seijin and Zucchet, Nicolas and Scherrer, Nino and Miller, Nolan and Sandler, Mark and Vladymyrov, Max and Pascanu, Razvan and others},
  journal={Preprint arXiv:2309.05858},
  year={2023}
}

@inproceedings{siems2025deltaproduct,
  title={DeltaProduct: Improving State-Tracking in Linear {RNN}s via Householder Products},
  author={Siems, Julien and Carstensen, Timur and Zela, Arber and Hutter, Frank and Pontil, Massimiliano and Grazzi, Riccardo},
  booktitle=confNEU,
  month = dec,
  address = {San Diego, CA, USA},
  year={2025}
}

@inproceedings{yang2024parallelizing,
  title={Parallelizing Linear Transformers with the Delta Rule over Sequence Length},
  author={Yang, Songlin and Wang, Bailin and Zhang, Yu and Shen, Yikang and Kim, Yoon},
  booktitle=confNEU,
address = {Vancouver, Canada},
month = dec,
  year={2024}
}

@inproceedings{YangWSPK24,
  author       = {Songlin Yang and
                  Bailin Wang and
                  Yikang Shen and
                  Rameswar Panda and
                  Yoon Kim},
  title        = {Gated Linear Attention Transformers with Hardware-Efficient Training},
  booktitle    = confICML,
  address = {Vienna, Austria},
  month = jul,
  year         = {2024},
}

@article{gershman2025key,
  title={Key-value memory in the brain},
  author={Gershman, Samuel J and Fiete, Ila and Irie, Kazuki},
  journal={Neuron},
volume={113},
  year={2025}
}

@inproceedings{von2025mesanet,
  title={Mesa{N}et: Sequence Modeling by Locally Optimal Test-Time Training},
  author={von Oswald, Johannes and Scherrer, Nino and Kobayashi, Seijin and Versari, Luca and Yang, Songlin and Schlegel, Maximilian and Maile, Kaitlin and Schimpf, Yanick and Sieberling, Oliver and Meulemans, Alexander and others},
  booktitle=confICLR,
  address = {Rio de Janeiro, Brazil},
  year={2026}
}

@article{behrouz2025atlas,
  title={Atlas: Learning to optimally memorize the context at test time},
  author={Behrouz, Ali and Li, Zeman and Kacham, Praneeth and Daliri, Majid and Deng, Yuan and Zhong, Peilin and Razaviyayn, Meisam and Mirrokni, Vahab},
  journal={Preprint arXiv:2505.23735},
  year={2025}
}

@inproceedings{sun2024learning,
  title={Learning to (learn at test time): {RNN}s with expressive hidden states},
  author={Sun, Yu and Li, Xinhao and Dalal, Karan and Xu, Jiarui and Vikram, Arjun and Zhang, Genghan and Dubois, Yann and Chen, Xinlei and Wang, Xiaolong and Koyejo, Sanmi and others},
  booktitle=confICML,
  address = {Vancouver, Canada},
  month = jul,
  year={2025}
}

@article{wang2025test,
  title={Test-time regression: a unifying framework for designing sequence models with associative memory},
  author={Wang, Ke Alexander and Shi, Jiaxin and Fox, Emily B},
  journal={Preprint arXiv:2501.12352},
  year={2025}
}

@inproceedings{behrouz2024titans,
  title={Titans: Learning to memorize at test time},
  author={Behrouz, Ali and Zhong, Peilin and Mirrokni, Vahab},
  booktitle=confNEU,
  address = {San Diego, CA, USA},
  year={2025}
}

@inproceedings{behrouz2025nested,
  title={Nested learning: The illusion of deep learning architectures},
  author={Behrouz, Ali and Razaviyayn, Meisam and Zhong, Peilin and Mirrokni, Vahab},
  booktitle=confNEU,
  address = {San Diego, CA, USA},
  month = dec,
  year={2025}
}

@article{aronszajn1950theory,
  title={Theory of reproducing kernels},
  author={Aronszajn, Nachman},
  journal={Transactions of the American mathematical society},
  volume={68},
  number={3},
  pages={337--404},
  year={1950}
}

@book{scholkopf2002learning,
  title={Learning with kernels: support vector machines, regularization, optimization, and beyond},
  author={Sch{\"o}lkopf, Bernhard and Smola, Alexander J},
  year={2002},
  publisher={MIT press}
}

@inproceedings{min2023nonparametric,
  title={Nonparametric masked language modeling},
  author={Min, Sewon and Shi, Weijia and Lewis, Mike and Chen, Xilun and Yih, Wen-tau and Hajishirzi, Hannaneh and Zettlemoyer, Luke},
  booktitle={Findings of the Association for Computational Linguistics},
  address = {Toronto, Canada},
  month = jul,
  pages={2097--2118},
  year={2023}
}

@inproceedings{wilson2016deep,
  title={Deep kernel learning},
  author={Wilson, Andrew Gordon and Hu, Zhiting and Salakhutdinov, Ruslan and Xing, Eric P},
  booktitle={Proc. Int. Conf. on Artificial Intelligence and Statistics (AISTATS)},
  address = {Cadiz, Spain},
  month = may,
  pages={370--378},
  year={2016},
}

@inproceedings{BoserGV92,
  author    = {Bernhard E. Boser and
               Isabelle Guyon and
               Vladimir Vapnik},
  title     = {A Training Algorithm for Optimal Margin Classifiers},
  booktitle = {Proc. Annual {ACM} Conference on Computational
               Learning Theory ({COLT})},
  address = {Pittsburgh, PA, USA},
  month = jul,
  pages     = {144--152},
  publisher = {{ACM}},
  year      = {1992},
}

@article{burges1998tutorial,
  title={A tutorial on support vector machines for pattern recognition},
  author={Burges, Christopher JC},
  journal={Data mining and knowledge discovery},
  volume={2},
  number={2},
  pages={121--167},
  year={1998},
}

@inproceedings{pritzel2017neural,
  title={Neural episodic control},
  author={Pritzel, Alexander and Uria, Benigno and Srinivasan, Sriram and Badia, Adria Puigdomenech and Vinyals, Oriol and Hassabis, Demis and Wierstra, Daan and Blundell, Charles},
  booktitle=confICML,
  pages={2827--2836},
  year={2017},
  address = {Lille, France},
  month = aug,
}

@article{weston2014memory,
  title={Memory networks},
  author={Weston, Jason and Chopra, Sumit and Bordes, Antoine},
  journal={Preprint arXiv:1410.3916},
  year={2014}
}

@inproceedings{SukhbaatarSWF15,
  author    = {Sainbayar Sukhbaatar and
               Arthur Szlam and
               Jason Weston and
               Rob Fergus},
  title     = {End-To-End Memory Networks},
  booktitle = confNIPS,
  address = {Montr\'eal, Canada},
  month = dec,
  pages     = {2440--2448},
  year      = {2015},
}

@inproceedings{jacot2018neural,
  title={Neural tangent kernel: Convergence and generalization in neural networks},
  author={Jacot, Arthur and Gabriel, Franck and Hongler, Cl{\'e}ment},
  booktitle=confNEU,
  address = {Montréal, Canada},
  month = dec,
  year={2018}
}

@inproceedings{chizat2019lazy,
  title={On lazy training in differentiable programming},
  author={Chizat, Lenaic and Oyallon, Edouard and Bach, Francis},
  booktitle=confNEU,
  address = {Vancouver, Canada},
  month = dec,
  year={2019}
}

@inproceedings{lee2019wide,
  title={Wide neural networks of any depth evolve as linear models under gradient descent},
  author={Lee, Jaehoon and Xiao, Lechao and Schoenholz, Samuel and Bahri, Yasaman and Novak, Roman and Sohl-Dickstein, Jascha and Pennington, Jeffrey},
  booktitle=confNEU,
  address = {Vancouver, Canada},
  month = dec,
  year={2019}
}

@article{turing1936,
  author  = {Turing, Alan M.},
  title   = {On Computable Numbers, with an Application to the Entscheidungsproblem},
  journal = {Proc. of London Mathematical Society},
  year    = {1936},
}

@article{ryan2015engram,
  title={Engram cells retain memory under retrograde amnesia},
  author={Ryan, Tom{\'a}s J and Roy, Dheeraj S and Pignatelli, Michele and Arons, Autumn and Tonegawa, Susumu},
  journal={Science},
  volume={348},
  number={6238},
  pages={1007--1013},
  year={2015},
}

@article{roy2017silent,
  title={Silent memory engrams as the basis for retrograde amnesia},
  author={Roy, Dheeraj S and Muralidhar, Shruti and Smith, Lillian M and Tonegawa, Susumu},
  journal={Proc. of National Academy of Sciences (PNAS)},
  volume={114},
  number={46},
  pages={E9972--E9979},
  year={2017},
}

@inproceedings{lewis2020retrieval,
  title={Retrieval-augmented generation for knowledge-intensive {NLP} tasks},
  author={Lewis, Patrick and Perez, Ethan and Piktus, Aleksandra and Petroni, Fabio and Karpukhin, Vladimir and Goyal, Naman and K{\"u}ttler, Heinrich and Lewis, Mike and Yih, Wen-tau and Rockt{\"a}schel, Tim and Riedel, Sebastian and Kiela, Douwe},
  booktitle=confNEU,
  address = {Virtual only},
  month = dec,
  year={2020}
}

@inproceedings{BolyaFDZFH23,
  author       = {Daniel Bolya and
                  Cheng{-}Yang Fu and
                  Xiaoliang Dai and
                  Peizhao Zhang and
                  Christoph Feichtenhofer and
                  Judy Hoffman},
  title        = {Token Merging: Your ViT But Faster},
  booktitle    = confICLR,
  address = {Kigali, Rwanda},
  month = may,
  year         = {2023},
}

@misc{jordan2024muon,
  author       = {Keller Jordan and Yuchen Jin and Vlado Boza and Jiacheng You and
                  Franz Cesista and Laker Newhouse and Jeremy Bernstein},
  title        = {Muon: An optimizer for hidden layers in neural networks},
  year         = {2024},
  url          = {https://kellerjordan.github.io/posts/muon/}
}

@inproceedings{LoshchilovH19,
  author       = {Ilya Loshchilov and
                  Frank Hutter},
  title        = {Decoupled Weight Decay Regularization},
  booktitle    = confICLR,
  address = {New Orleans, LA, USA},
  month = may,
  year         = {2019},
}

@article{fukushimaRelu,
  title={Visual feature extraction by a multilayered network of analog threshold elements},
  author={Fukushima, Kunihiko},
  journal={IEEE Transactions on Systems Science and Cybernetics},
  volume={5},
  number={4},
  pages={322--333},
  year={1969},
}

@inproceedings{NairH10,
  author    = {Vinod Nair and
               Geoffrey E. Hinton},
  title     = {Rectified Linear Units Improve Restricted Boltzmann Machines},
  booktitle = confICML,
  month = jun,
  address = {Haifa, Israel},
  pages     = {807--814},
  year      = {2010},
}

@article{hendrycks2016gaussian,
  title={Gaussian Error Linear Units ({GELU}s)},
  author={Hendrycks, Dan},
  journal={Preprint arXiv:1606.08415},
  year={2016}
}

@article{bridle1990softmax,
  title={Probabilistic interpretation of feedforward classification network outputs, with relationships to statistical pattern recognition},
  author={Bridle, John S},
  journal={Neurocomputing: Algorithms, architectures and applications},
  pages={227--236},
  year={1990},
  publisher={Springer}
}

@inproceedings{ZhangS19a,
  author       = {Biao Zhang and
                  Rico Sennrich},
  title        = {Root Mean Square Layer Normalization},
  booktitle    = confNEU,
  address = {Vancouver, Canada},
  year         = {2019},
}

@article{gardner1989three,
  title={Three unfinished works on the optimal storage capacity of networks},
  author={Gardner, Elizabeth and Derrida, Bernard},
  journal={Journal of Physics A: Mathematical and General},
  volume={22},
  number={12},
  year={1989},
}

@inproceedings{simsek2021geometry,
  title={Geometry of the loss landscape in overparameterized neural networks: Symmetries and invariances},
  author={Simsek, Berfin and Ged, Fran{\c{c}}ois and Jacot, Arthur and Spadaro, Francesco and Hongler, Cl{\'e}ment and Gerstner, Wulfram and Brea, Johanni},
  booktitle=confICML,
  address = {Virtual only},
  month = jul,
  year={2021},
}

@inproceedings{berges2025memory,
  title={Memory layers at scale},
  author={Berges, Vincent-Pierre and O{\u{g}}uz, Barlas and Haziza, Daniel and Yih, Wen-tau and Zettlemoyer, Luke and Ghosh, Gargi},
  booktitle=confICML,
  address = {Vancouver, Canada},
  month = jul,
  year={2025}
}

@article{eyuboglu2025cartridges,
  title={Cartridges: Lightweight and general-purpose long context representations via self-study},
  author={Eyuboglu, Sabri and Ehrlich, Ryan and Arora, Simran and Guha, Neel and Zinsley, Dylan and Liu, Emily and Tennien, Will and Rudra, Atri and Zou, James and Mirhoseini, Azalia and others},
  journal={Preprint arXiv:2506.06266},
  year={2025}
}

@inproceedings{gevaetal2021transformer,
    title = "Transformer Feed-Forward Layers Are Key-Value Memories",
    author = "Geva, Mor  and
      Schuster, Roei  and
      Berant, Jonathan  and
      Levy, Omer",
    booktitle = confEMNLP,
    month = nov,
    year = "2021",
    address = "Punta Cana, Dominican Republic",
    pages = "5484--5495",
}

@article{domingos2020every,
  title={Every model learned by gradient descent is approximately a kernel machine},
  author={Domingos, Pedro},
  journal={Preprint arXiv:2012.00152},
  year={2020}
}

@inproceedings{behrouz2026memory,
title={Memory Caching: {RNN}s with Growing Memory},
author={Ali Behrouz and Zeman Li and Yuan Deng and Peilin Zhong and Meisam Razaviyayn and Vahab Mirrokni},
booktitle=confICML,
year={2026},
address={Seoul, South Korea},
month=jul
}

@inproceedings{yoonYLH18,
  author       = {Jaehong Yoon and Eunho Yang and
                  Jeongtae Lee and
                  Sung Ju Hwang},
  title        = {Lifelong Learning with Dynamically Expandable Networks},
  booktitle    = confICLR,
  address = {Vancouver, Canada},
  month = apr,
  year         = {2018},
}

@article{irie2026dual,
  author =        {Irie, Kazuki},
  journal =     {Preprint},
  title =         {The Dual Form of the Perceptron (1964) Meets Deep Learning and What That Enables},
  year =          {2026},
}

@inproceedings{EvciMUPV22,
  author       = {Utku Evci and
                  Bart van Merrienboer and
                  Thomas Unterthiner and
                  Fabian Pedregosa and
                  Max Vladymyrov},
  title        = {Grad{M}ax: Growing Neural Networks using Gradient Information},
  booktitle    = confICLR,
  month = apr,
  address = {Virtual only},
  year         = {2022},
}

@article{butkus2026growing,
  title={Growing a Neural Network in Breadth, Depth, and Time},
  author={Butkus, Eivinas and Gupta, Kedar Garz{\~A}{\`g}n and Kriegeskorte, Nikolaus},
  journal={Preprint arXiv:2605.25174},
  year={2026}
}

@article{stanley2002evolving,
  title={Evolving neural networks through augmenting topologies},
  author={Stanley, Kenneth O and Miikkulainen, Risto},
  journal={Evolutionary computation},
  volume={10},
  number={2},
  pages={99--127},
  year={2002},
}

@inproceedings{luong2015effective,
  title={Effective approaches to attention-based neural machine translation},
  author={Luong, Minh-Thang and Pham, Hieu and Manning, Christopher D},
  booktitle=confEMNLP,
  month = sep,
  address = {Lisbon, Portugal},
  year={2015}
}
